\documentclass[11pt,reqno,a4paper]{article}

\usepackage{fullpage}
\usepackage[utf8]{inputenc}
\usepackage[T1]{fontenc}
\usepackage[backref=section]{hyperref}
\usepackage{url}
\usepackage{booktabs}
\usepackage{amsfonts}
\usepackage{nicefrac}
\usepackage[expansion=false]{microtype}
\usepackage{xcolor}
\usepackage{graphicx}
\usepackage{subfigure}
\usepackage{amsmath}
\usepackage{amssymb}
\usepackage{mathtools}
\usepackage{amsthm}
\usepackage{algorithm}
\usepackage{algorithmic}
\usepackage{enumitem}

\usepackage{array}
\usepackage{natbib}

\usepackage{placeins}
\usepackage{needspace}

\hypersetup{breaklinks=true,colorlinks=true,citecolor=violet}

\newtheorem{theorem}{Theorem}
\newtheorem{proposition}[theorem]{Proposition}
\newtheorem{lemma}[theorem]{Lemma}
\newtheorem{corollary}[theorem]{Corollary}
\theoremstyle{definition}
\newtheorem{definition}{Definition}

\theoremstyle{remark}

\usepackage{amsmath,amsfonts,bm}

\def\1{\bm{1}}

\def\vzero{{\bm{0}}}

\def\ve{{\bm{e}}}

\def\vp{{\bm{p}}}
\def\vq{{\bm{q}}}
\def\vr{{\bm{r}}}

\def\vu{{\bm{u}}}
\def\vv{{\bm{v}}}
\def\vw{{\bm{w}}}
\def\vx{{\bm{x}}}
\def\vy{{\bm{y}}}
\def\vz{{\bm{z}}}

\def\mA{{\bm{A}}}

\def\mD{{\bm{D}}}

\def\mF{{\bm{F}}}

\def\mI{{\bm{I}}}

\def\mK{{\bm{K}}}
\def\mL{{\bm{L}}}
\def\mM{{\bm{M}}}

\def\mS{{\bm{S}}}

\def\mW{{\bm{W}}}

\def\mZ{{\bm{Z}}}

\def\mPhi{{\bm{\Phi}}}

\DeclareMathAlphabet{\mathsfit}{\encodingdefault}{\sfdefault}{m}{sl}
\SetMathAlphabet{\mathsfit}{bold}{\encodingdefault}{\sfdefault}{bx}{n}

\newcommand{\iprod}[1]{\left\langle #1 \right\rangle}

\newcommand{\abs}[1]{\left| #1 \right|}
\newcommand{\norm}[1]{\left\| #1 \right\|}
\newcommand{\parens}[1]{\left( #1 \right)}

\newcommand{\normTwo}[1]{\left\| #1 \right\|_2}
\newcommand{\normMax}[1]{\left\| #1 \right\|_{\infty}}
\newcommand{\normTwoToTwo}[1]{\left\| #1 \right\|_{2\to2}}

\newcommand{\highlight}{\boxed}

\newcommand{\nhead}{H}

\newcommand{\ReLU}{\mathrm{ReLU}}
\newcommand{\mlp}{\mathrm{MLP}}
\newcommand{\attn}{\mathrm{Attn}}
\newcommand{\TF}{\mathrm{TF}}
\newcommand{\TFB}{\mathrm{TFB}}

\newcommand{\attnQ}{\mW_Q}
\newcommand{\attnK}{\mW_K}
\newcommand{\attnV}{\mW_V}
\newcommand{\mlpin}{\mW_{\mathrm{in}}}
\newcommand{\mlpout}{\mW_{\mathrm{out}}}
\newcommand{\KerMat}{\mK}
\newcommand{\eig}{\mathrm{eig}}
\newcommand{\diag}{\mathrm{diag}}

\newcommand{\nsq}{n_{\mathrm{sq}}}
\newcommand{\hatnsq}{\widehat{n}_{\mathrm{sq}}}
\newcommand{\tildensq}{\widetilde{n}_{\mathrm{sq}}}
\newcommand{\nflip}{n_{\mathrm{flip}}}

\newcommand{\ninv}{n_{\mathrm{inv}}}

\newcommand{\epssq}{\epsilon_{\mathrm{sq}}}
\newcommand{\hatepssq}{\widehat{\epsilon}_{\mathrm{sq}}}
\newcommand{\tildeepssq}{\widetilde{\epsilon}_{\mathrm{sq}}}
\newcommand{\epsflip}{\epsilon_{\mathrm{flip}}}

\newcommand{\epsinv}{\epsilon_{\mathrm{inv}}}

\newcommand{\scalingerror}{\vr}

\NewDocumentCommand{\myl}{ O{\ell} }{%
    ^{\left(#1\right)}%
}

\newcommand{\veps}{\bm{\epsilon}}
\newcommand{\phisq}{\phi_{\mathrm{sq}}}
\newcommand{\hatphisq}{\widehat{\phi}_{\mathrm{sq}}}
\newcommand{\tildephisq}{\widetilde{\phi}_{\mathrm{sq}}}

\newcommand{\phiinv}{\phi_{\mathrm{inv}}}
\title{
Transformers Can Implement 
Preconditioned Richardson Iteration 
for In-Context Gaussian Kernel Regression
}

\author{%
Mingsong Yan\thanks{Corresponding authors.} , \quad
Dongyang Li, \quad
Charles Kulick, \quad 
Sui Tang\footnotemark[1]
\\[0.4em]
\small Department of Mathematics, University of California, Santa Barbara, CA
\\[0.25em]
\small
\texttt{\{mingsongyan,dongyang\_li,charleskulick,suitang\}@ucsb.edu} 
}
\begin{document}

\date{}
\maketitle

\begin{abstract}

Mechanistic studies of in-context learning (ICL) have connected transformer computations to iterative optimization algorithms for linear regression and related tasks, often using linear or ReLU attention variants. For nonlinear ICL, prior work has related softmax and kernelized attention to functional gradient descent, but it remains unclear whether a standard transformer with softmax attention can implement a convergent solver with an end-to-end prediction-error guarantee. In this paper, we study in-context kernel ridge regression (KRR) with Gaussian kernels and show, both theoretically and empirically, that a standard softmax-attention transformer can approximate the KRR predictor during its forward pass. Under bounded-data assumptions, we construct a single-head transformer whose forward pass approximately implements \textit{preconditioned Richardson iteration} on the associated kernel system. The construction uses $\mathcal{O}(\log(1/\epsilon))$ blocks and MLP width $\mathcal{O}(\sqrt{N/\epsilon})$ to achieve $\epsilon$-accurate prediction for prompts of length $N$. Our construction reveals a functional decomposition within the transformer architecture: softmax attention produces a row-normalized Gaussian-kernel operator needed for \emph{cross-token} interactions, while MLP layers act locally to approximate the \emph{intra-token} scalar arithmetic required by the update. Empirically, we train GPT-2-style transformers on Gaussian-process regression tasks and observe that they progressively align with the exact Gaussian KRR estimator across depth in terms of both \emph{prediction error} and \emph{induced weights}, with ablations further supporting this trend. Comparisons with classical KRR solvers also show that deeper layers align with later solver iterates. Together, we empirically demonstrate that the pretrained transformers exhibit progressive refinement toward exact Gaussian KRR across depth, and theoretically establish inexact preconditioned Richardson iteration as a concrete mechanism for approximating this predictor within an explicitly constructed softmax-attention transformer.

\end{abstract}

\section{Introduction}
\label{sec:introduction}
In-context learning (ICL) refers to the phenomenon in which a pretrained transformer, given a sequence of input-output examples followed by a query input, produces a prediction for that query in a single forward pass, \emph{without parameter updates} \citep{brown2020language,garg2022can}. This has emerged as a central question in the theoretical study of ICL: does the forward pass itself implement a learning algorithm on the examples in the prompt, and if so, what algorithm is being run and how is it realized by the architecture?  

For linear regression and related linear tasks, prior work has made this
algorithmic viewpoint precise: transformers can implement classical
optimization algorithms, including gradient descent, preconditioned gradient
descent, and Newton-type iterations \citep{vonoswald2023transformers,
akyurek2023learning,ahn2023transformers,bai2023transformers,
mahankali2024optimal,vladymyrov2024linear,fu2024transformers,zhang2025training}.
Transformer layers can be understood as steps of an iterative optimizer.

For nonlinear ICL, the corresponding mechanism is less understood. Kernel
ridge regression (KRR) is a natural test scenario: nonlinear in the input but
mathematically explicit, and on Gaussian process regression tasks its
solution coincides with the Bayes posterior mean. Recent work analyzes
attention as a form of kernel or functional gradient descent in a
reproducing kernel Hilbert space (RKHS) \citep{cheng2023transformers,
han2025understanding,dragutinovic2025softmax,sander2024towards}. Concrete
convergence guarantees in this line, however, require attention to be
\emph{kernelized} so that each layer realizes a standard functional gradient
step \citep{cheng2023transformers}. Standard softmax attention corresponds
instead to kernel gradient descent with a context-adaptive learning rate
\citep{cheng2023transformers,sander2024towards}. However, it remains unclear, empirically, to what extent pretrained
GPT-style transformers align with exact KRR predictor across depth and,
theoretically, whether one can explicitly construct a transformer
with standard softmax attention and ReLU MLPs that realizes an
iterative KRR solver with end-to-end prediction-error control.

We study both questions in the setting of in-context Gaussian KRR. Figure~\ref{fig:convergence} presents our central empirical observation:
GPT-style transformers trained on Gaussian-process regression tasks
progressively align with exact Gaussian KRR across depth, in terms of
both prediction error and output-induced weights. These output-level observations motivate our constructive analysis.

To construct a transformer with end-to-end prediction-error control,
we move from the primal RKHS formulation to the dual kernel linear
system: given training samples $\{(\vx_i,y_i)\}_{i=1}^N$, the KRR
predictor admits dual coefficients $\vw\in\mathbb{R}^N$ that satisfy
\[
    (\KerMat+\lambda \mI)\vw=\vy,
    \qquad
    \KerMat_{ij}=\mathcal{K}(\vx_i,\vx_j).
\]
The mechanism question then shifts from
\emph{\text{``does the transformer descend a primal RKHS loss?''}}
to \emph{\text{``does the transformer solve the dual kernel linear system?''.}}
For the Gaussian kernel
$\mathcal{K}(\vx,\vx')=\exp(-\|\vx-\vx'\|^2/(2v^2))$, suppose the raw
attention score between tokens $i$ and $j$ is
$\mL_{ij}=-\|\vx_i-\vx_j\|^2/(2v^2)$. After softmax normalization,
\[
    \mathrm{softmax}(\mL)_{ij}
    =
    \frac{\mathcal{K}(\vx_i,\vx_j)}
    {\sum_{j'=1}^N \mathcal{K}(\vx_i,\vx_{j'})}
    =
    (\mD^{-1}\KerMat)_{ij},
    \qquad
    \mD=\diag(\KerMat\mathbf{1}).
\]
Softmax thus returns the kernel matrix itself with each row rescaled to sum to one, i.e., the row-sum preconditioner of the dual kernel system. Aggregating value vectors with these weights implements the cross-token operation in a row-sum--preconditioned Richardson step; the remaining tokenwise operations (rescaling by diagonal quantities, adding the data term, updating the current iterate) are within the reach of an MLP layer. Section~\ref{sec:mechanism} makes this construction precise.

\begin{figure}[!t]
\centering
\subfiguretopcaptrue

\makeatletter
\renewcommand*{\@thesubfigure}{%
  \hspace*{1.2em}\thesubfigure\hskip\subfiglabelskip
}
\makeatother

\subfigure[Prediction-error alignment]{%
  \includegraphics[width=0.49\textwidth]
    {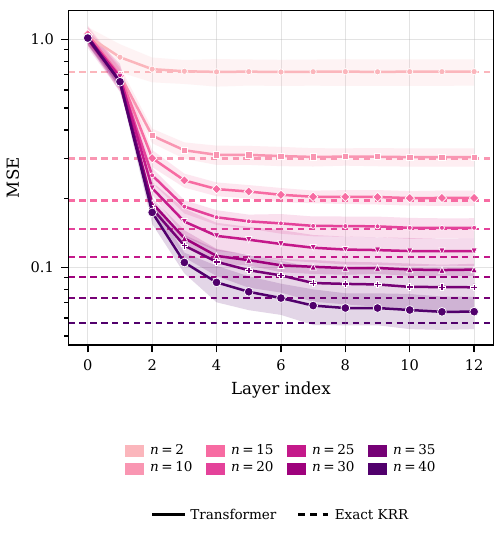}%
  \label{fig:convergence-a}%
}\hfill
\subfigure[Induced-weight alignment]{%
  \includegraphics[width=0.49\textwidth]
    {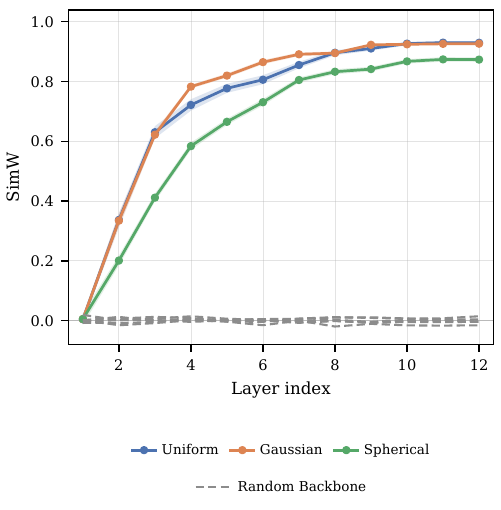}%
  \label{fig:convergence-b}%
}

\caption{The transformer's in-context \textit{prediction errors} and \textit{induced weights} progressively align with those of exact KRR as depth increases. \textbf{(a)} Layer-wise mean-squared error (MSE) on spherical inputs; dashed
horizontal lines are the corresponding exact-KRR MSEs, coloured to match their
context length. \textit{For each context length $n$, the transformer MSE progressively
approaches the exact-KRR reference as depth increases.} Similar trends hold for
Uniform and Gaussian inputs (Figure~\ref{fig-app: all MSEs}). \textbf{(b)} Cosine similarity (SimW) between the transformer's induced weights and the exact KRR weights at context length $n=40$. \textit{The induced weights become increasingly aligned with the exact KRR weights as the layer depth increases, across all three input distributions.} Grey dashed curves show random-initialized backbones, which remain near zero at every depth. In both panels, per-layer quantities are read out using linear probes trained with the backbone frozen. Results are computed over \(5\) evaluation seeds, each with \(256\) input sequences, with shaded bands indicating one standard deviation
across seeds. Full details are in Section~\ref{sec:experiments}.}
\label{fig:convergence}
\end{figure}

We summarize our contributions as follows. 

\begin{itemize}[leftmargin=1.4em, itemsep=4pt, topsep=2pt]

\item \textbf{A realization of in-context Gaussian KRR with softmax-attention and ReLU-MLP layers}\ (Theorem~\ref{theorem: informal version of KRR ICL}; Appendix~\ref{app: construction of transformers}). We construct a transformer that approximates the exact Gaussian KRR predictor without modifying the standard block architecture: softmax attention is followed by a single-hidden-layer ReLU MLP, with residual connections around both sublayers. Its forward pass approximately implements Richardson iteration with row-sum preconditioning, with softmax-attention performing the row-normalized cross-token kernel operations and ReLU-MLPs approximating the required tokenwise arithmetic. The construction neither replaces softmax with another attention operator nor directly exposes its normalizer.

\item \textbf{End-to-end convergence with explicit depth--width scaling}
(Theorem~\ref{theorem: informal version of KRR ICL}).
We establish an end-to-end prediction-error guarantee by combining a
convergence analysis of \emph{inexact} preconditioned Richardson iteration with
explicit ReLU approximation rates for the required scalar operations. Under
bounded-data assumptions, the constructed transformer achieves prediction error at most $\varepsilon$ using depth
$\mathcal{O}(\log(1/\varepsilon))$ and MLP width
$\mathcal{O}(\sqrt{N/\varepsilon})$ on prompts of length $N$. The depth scaling
follows from the convergence rate of preconditioned Richardson, while the width
scaling reflects the cost of approximating the tokenwise scalar arithmetic.

\item \textbf{Empirical evidence of KRR refinement} (Section~\ref{sec:experiments}). We pretrain GPT-2-style transformers on Gaussian-process regression tasks for
three input distributions and evaluate them across multiple context lengths.
Using layer-wise linear probes, we observe progressive alignment with exact
Gaussian KRR across depth in terms of both prediction error and induced-weight
similarity. The probing ablations indicate that the observed trend is due to the trained transformer rather than the probes; the kernel-mismatch ablation shows that the transformer aligns
more closely with KRR based on the data-generating kernel; and the noise-mismatch ablation supports the interpretation of the trained transformer as a finite-iterate KRR solver with a fixed regularization parameter. We compare the transformer with five classical KRR solvers and
find that deeper layers generally best match later solver iterates. The late-layer similarities are closely clustered across solvers, consistent with their convergence to the same exact-KRR solution.
\end{itemize}

\section{Related work}
\label{sec:related}

\textbf{Algorithmic interpretations of ICL.}
A major line of research frames the transformer's forward pass as the execution of a learning algorithm on the prompt. \citet{garg2022can} empirically showed that transformers can be trained to solve linear regression and decision-tree tasks in context. Subsequent theoretical work has focused predominantly on linear regression, identifying gradient descent \citep{vonoswald2023transformers,akyurek2023learning}, preconditioned gradient descent \citep{ahn2023transformers,mahankali2024optimal,vladymyrov2024linear}, and Newton-type iterations \citep{fu2024transformers}, with parallel analyses of training dynamics \citep{zhang2024trained,chen2024training,lu2024asymptotic,zhang2025training,huang2023context}. Related works also study algorithm selection across linear tasks \citep{bai2023transformers}, nearest-neighbor classification \citep{li2024one}, discrete and Boolean function learning \citep{bhattamishra2024understanding}, and formal-language learning \citep{akyurek2024incontext}. 
Within this line of research, several studies examine how MLPs
contribute to ICL in linear-attention models
\citep{zhang2024linearblock,sun2025feedforward,demir2025asymptotic}. For noisy linear dynamical systems, \citet{cole2026context} construct deep linear transformers that implement modified Richardson iteration to approximate least-squares prediction with quantitative error guarantees.

\textbf{Kernel and softmax views of nonlinear ICL.}
A complementary line connects ICL to kernel-based prediction. \citet{cheng2023transformers} show that kernel-matched attention corresponds to functional gradient descent in the RKHS. \citet{goel2024can} realize Gaussian Nadaraya--Watson smoothing with softmax attention for Kalman filtering, while
\citet{collins2024context} interpret pretraining single-head
softmax attention as learning the bandwidth of such an estimator.
\citet{han2025understanding} relate Bayesian inference on
in-context prompts to normalized kernel regression.
\citet{shen2025understanding} construct transformers for kernel
smoothing on manifolds. \citet{dragutinovic2025softmax} show that softmax attention implements context-adaptive kernel gradient descent for cross-entropy classification. 

We relate softmax attention to preconditioned Richardson iteration for solving the dual Gaussian-KRR system. We note that independent and concurrent work by \citet{kang2026transformers}
also connects attention normalization to this iteration for Gaussian-process posterior prediction.
Their construction uses a modified transformer architecture that assumes both direct access to context-dependent attention denominators and exact division by them within each solver update. Under these architectural assumptions, their constructed layers
realize each Richardson iteration \emph{exactly}. In contrast, our construction retains standard softmax attention followed by ReLU MLPs throughout, using a dummy token to recover normalization information and MLPs to approximate the required arithmetic. We jointly analyze these approximation errors and how they propagate through the resulting \emph{inexact} iterations. The experimental designs also fundamentally differ: their experiments evaluate posterior prediction in models pretrained with their proposed architecture. Our experiments investigate whether standard GPT-style transformers, trained from scratch without a prescribed solver, progressively approach exact KRR across depth. We assess this progression through layerwise diagnostics, ablations, and comparisons with several classical iterative solvers.

\section{Preliminaries}
\label{sec:problem}

\subsection{Transformers}
We consider transformer architectures that operate on an input sequence of $N$ tokens, represented by a feature matrix $\mZ = [\vz_1, \dots, \vz_N] \in \mathbb{R}^{D \times N}$ whose $i$-th column $\vz_i \in \mathbb{R}^D$ is the embedding of the $i$-th token. We consider the following transformer architecture \citep{vaswani2017attention}. 

\begin{definition}[Softmax attention layer]
A multi-head attention layer with \(\nhead\) heads is denoted by \(\attn(\cdot)\). For each head \(h\in[\nhead]\), let
$\attnQ^{(h)},\attnK^{(h)},\attnV^{(h)}\in\mathbb{R}^{D\times D}$ be the query, key, and value matrices, and let $\mM^{(h)}\in\{0,-\infty\}^{N\times N}$ be a mask matrix. Given $\mZ\in\mathbb{R}^{D\times N}$, the layer outputs $\widetilde{\mZ}=\attn(\mZ)$ via the residual update
$\widetilde{\mZ} := \mZ + \sum_{h=1}^{\nhead} (\attnV^{(h)} \mZ) \cdot \sigma\!\bigl((\attnK^{(h)} \mZ)^\top (\attnQ^{(h)} \mZ)+(\mM^{(h)})^\top\bigr),$
where \(\sigma(\cdot)\) is the column-wise softmax function. In vector form, for each \(i\in[N]\),
$\widetilde{\vz}_i = \vz_i + \sum_{h=1}^{\nhead} \sum_{j=1}^N \frac{\exp\parens{\iprod{ \attnQ^{(h)} \vz_i, \attnK^{(h)} \vz_j } + \mM_{ij}^{(h)}}}{\sum_{j'=1}^N\exp\parens{\iprod{ \attnQ^{(h)} \vz_i, \attnK^{(h)} \vz_{j'} } + \mM_{ij'}^{(h)}}} \attnV^{(h)} \vz_j.$
\end{definition}

\begin{definition}[MLP layer]
An MLP layer with hidden dimension $D'$, denoted $\mlp(\cdot)$, has weight matrices $\mlpin\in\mathbb{R}^{D'\times D}, \mlpout \in \mathbb{R}^{D\times D'}$ and acts pointwise on each token: for $\mZ \in \mathbb{R}^{D \times N}$,
$\mlp(\mZ) := \mZ + \mlpout\,\sigma(\mlpin\,\mZ),$
where $\sigma(t) = \max\{t,0\}$ is the entrywise ReLU activation.
\end{definition}

\begin{definition}[Transformer block and $L$-block transformer]
A transformer block ($\TFB$) is the composition $\TFB(\mZ) := \mlp(\attn(\mZ))$ of an attention layer followed by an MLP layer; we also allow an MLP-only block $\TFB_0$ in which the attention layer acts as the identity (e.g., by setting all value matrices to zero), so $\TFB_0(\mZ) = \mlp(\mZ)$. An $L$-block transformer $\TF$ is the sequential composition $\TF(\mZ) := \bigl(\TFB^{(L)} \circ \TFB^{(L-1)} \circ \cdots \circ \TFB^{(1)}\bigr)(\mZ)$ of $L$ blocks.
\end{definition}

\subsection{In-context Learning for Kernel Ridge Regression}\label{section: In-context Learning for Kernel Ridge Regression}
\paragraph{Kernel ridge regression.}
Let $\mathcal{D}=\{(\vx_i,y_i)\}_{i=1}^N$ be a given training dataset. Here, $\{\vx_i\}_{i=1}^N \subset \mathbb{R}^d$ represent the input vectors, $\{y_i\}_{i=1}^N \subset \mathbb{R}$ are the corresponding labels. We formulate the classical kernel ridge regression problem within a Reproducing Kernel Hilbert Space (RKHS) $\mathcal{H}$ associated with a reproducing kernel $\mathcal{K}$. The objective is to minimize the regularized empirical risk
\begin{equation}\label{eq: optimization problem in RKHS}
\min_{f\in\mathcal{H}}\left\{\frac{1}{N}\sum_{i=1}^N (f(\vx_i)-y_i)^2 + \lambda_0\|f\|_{\mathcal{H}}^2\right\},
\end{equation}
where $\lambda_0>0$ is the regularization parameter. By the Representer Theorem \citep{scholkopf2001generalized}, the optimal solution admits the finite-dimensional form 
$f^*(\cdot)=\sum_{i=1}^N w_i \mathcal{K}(\vx_i,\cdot).$ In the RKHS setting considered here, optimizing over the dual coefficients $\vw = [w_1, \dots, w_N]^\top \in \mathbb{R}^N$ yields the explicit dual linear system
\begin{equation}\label{eq: unnormalized kernel system}
(\KerMat+\lambda\mI)\vw=\vy,
\end{equation}
with kernel matrix $\KerMat_{ij} = \mathcal{K}(\vx_i,\vx_j)$, label vector $\vy = [y_1, \dots, y_N]^\top$, and $\lambda := \lambda_0 N$. 

This formulation has a clean statistical interpretation \citep{rasmussen2006gaussian,caponnetto2007optimal}. Suppose the training labels are generated as $y_i=f(\vx_i)+\epsilon_i$, where
$f$ is drawn from a zero-mean Gaussian process with covariance kernel
$\mathcal{K}$, denoted $f\sim\mathrm{GP}(0,\mathcal{K})$, and
$\epsilon_i\sim\mathcal{N}(0,\sigma^2)$ are i.i.d.\ observation noises. If the test label is the noiseless latent value
$y_{N+1}=f(\vx_{N+1})$, then setting $\lambda=\sigma^2$ makes the KRR
predictor $f^*(\vx_{N+1})$ coincide with the Bayes posterior mean.

\paragraph{Iterative solution via preconditioned Richardson iteration.} As discussed in the introduction, the row-sum diagonal $\mD = \diag(\KerMat\mathbf{1})$ acts as a preconditioner: left-multiplying~\eqref{eq: unnormalized kernel system} by $\mD^{-1}$ gives the equivalent system $\mD^{-1}(\KerMat+\lambda\mI)\vw = \mD^{-1}\vy$, on which the preconditioned Richardson iteration with step size $\eta>0$ reads
\begin{equation}\label{eq: precond_richardson}
    \vw^{(\ell+1)} = \vw^{(\ell)} + \eta\,\bigl( \mD^{-1}\vy - \mD^{-1}(\KerMat+\lambda \mI)\vw^{(\ell)} \bigr).
\end{equation}
Iteration~\eqref{eq: precond_richardson} is the update that the transformer
construction in Section~\ref{sec:mechanism} approximately implements step by step in its forward pass.

We compare plain and preconditioned Richardson iteration for solving the kernel system~\eqref{eq: unnormalized kernel system}. Plain Richardson corresponds to iteration~\eqref{eq: precond_richardson} with $\mD$ replaced by the identity matrix. In both cases, the iteration converges if and only if $\eta\in(0,2/\eig_{\max}(\mS))$, where $\mS=\mK+\lambda\mI$ for plain Richardson and $\mS=\mD^{-1/2}(\mK+\lambda\mI)\mD^{-1/2}$ for preconditioned Richardson. Figure~\ref{fig:lr2-all-distributions} compares the stability threshold $2/\eig_{\max}(\mS)$ across context sizes $N$. For plain Richardson, this threshold decreases substantially with $N$, whereas for the preconditioned system it remains nearly constant, allowing a single step size across context sizes. However, this stabilization of the admissible step-size range does not lead to faster convergence in the regimes considered here. For either system, the theoretically optimal constant step size is $\eta=2/(\eig_{\max}(\mS)+\eig_{\min}(\mS))$, which yields the optimal spectral convergence factor $(\kappa(\mS)-1)/(\kappa(\mS)+1)$. As shown in Figure~\ref{fig:condition-all-distributions}, the corresponding matrices $\mS$ for plain and preconditioned Richardson exhibit comparable condition numbers across $N$. Moreover, Figure~\ref{fig:convergence-all-distributions} shows that plain and preconditioned Richardson require comparable numbers of iterations to reach the same relative coefficient error across context sizes $N$, with no systematic acceleration from the row-sum preconditioning.

\paragraph{In-context learning.}
For in-context learning, the transformer receives the training dataset $\mathcal{D}$, a new test input $\vx_{N+1} \in \mathbb{R}^d$, and a set of auxiliary weights $\{w_i\}_{i=1}^N$ that act as placeholders for the intermediate iterates $\vw^{(\ell)}$ in~\eqref{eq: precond_richardson}; the construction in Section~\ref{sec:mechanism} updates these weights through depth and reads off the prediction at $\vx_{N+1}$.

To facilitate the required mathematical operations within the transformer's forward pass, we adopt two simple input augmentations: we include the squared norms $\|\vx_j\|^2$ as a feature of each token (or, in the unit-norm setting, this feature is constant and can be omitted \citep{dragutinovic2025softmax}), and we prepend the sequence with a designated dummy token $\vx_0 = \vzero$. We view these augmentations as \emph{input preprocessing} rather than part of the
core iterative-solver mechanism. This is consistent with prior constructive analyses of ICL, where input encoding, initialization, or readout operations are often separated from the main algorithmic update
\citep{vonoswald2023transformers,bai2023transformers,fu2024transformers}. Indeed, the squared-norm feature $\|\vx_j\|^2$ can be approximated by an MLP-only transformer block; the dummy token $\vx_0=\vzero$ is also a mild input convention, analogous to
special sequence tokens such as beginning-of-sequence markers in autoregressive
language models \citep{touvron2023llama}. We hard-code both in the input embedding to keep the iterative-solver mechanism clean.

Formally, we construct the input matrix $\mZ=[\vz_0,\vz_1,\ldots,\vz_{N+1}]
\in\mathbb{R}^{D\times(N+2)}$, with $D=d+11$, where
$\vz_j=[\vx_j;y_j;w_j;\|\vx_j\|^2;\mathbf{0}_5;s_j;t_j;1]\in\mathbb{R}^D$.
Here, the five zero entries provide cached memory for intermediate computations,
while the structural indicators $s_j$ and $t_j$ act as positional encodings
that identify the dummy and test tokens, respectively: $s_0=1$ and
$s_j=0$ for $j\in[N+1]$, whereas $t_{N+1}=1$ and
$t_j=0$ for $j\in\{0,\ldots,N\}$. We set the label and weight entries of the
dummy and test tokens to zero. Applying the transformer gives
$\widetilde{\mZ}=\TF(\mZ)$, from which we extract the prediction for the test token via a readout function $y_{N+1} = \mathrm{readout}(\widetilde{\mZ}) := \widetilde{\mZ}_{d+1,N+1}.$

\section{Main Result: An End-to-End Construction}
\label{sec:mechanism}

We construct a single-head softmax-attention transformer with ReLU MLPs whose
forward pass approximately implements preconditioned Richardson
iteration \eqref{eq: precond_richardson} for the kernel system \eqref{eq: unnormalized kernel system}, with explicit bounds on the
required depth and MLP width.

\textbf{Assumptions.}
We impose a \emph{bounded data} assumption: there exist $B_x, B_y > 0$ such that $\normTwo{\vx_i} \leq B_x$ and $\abs{y_i} \leq B_y$ for all $i \in [N+1]$. The kernel is the Gaussian kernel $\mathcal{K}(\vx,\vx') = \exp(-\|\vx-\vx'\|_2^2 / (2v^2))$ with bandwidth $v$, and the dual coefficients are initialized to zero, $\vw^{(0)} = \vzero$. 

\begin{theorem}[Informal]\label{theorem: informal version of KRR ICL}
Fix \(\lambda_0\in(0,1]\), \(c\in(0,1)\), $\epsilon\in(0,c)$, and a step size \(\eta\) satisfying \eqref{eq: final eta condition}. Let $\vw^* = [w_i^* : i \in [N]]$ denote the solution to the KRR linear system~\eqref{eq: unnormalized kernel system}. There exists a single-head transformer $\TF$ with $2L+5$ blocks (three read-in blocks, two blocks for each
of the $L$ inexact preconditioned Richardson iterations, and two read-out blocks) for $L = \mathcal{O}(\log(1/\epsilon))$ and MLP width $W=\mathcal{O}(\sqrt{N/\epsilon})$ such that 
\begin{equation*}
\abs{\mathrm{readout}(\TF(\mZ)) - \sum_{i=1}^N w_i^* \,\mathcal{K}(\vx_i, \vx_{N+1})} \;\le\; C_{\mathrm{sys}}\,\epsilon.
\end{equation*}
The system constant $C_{\mathrm{sys}}$ depends on regularization parameter $\lambda_0$, kernel bandwidth $v$, step size $\eta$, data bounds $B_x, B_y$, and constant $c$, but is independent of the prompt length $N$.
\end{theorem}

The formal version of the above theorem appears as
Theorem~\ref{theorem: formal version of KRR ICL} in
Appendix~\ref{app: construction transformers analysis}, with explicit
expressions \eqref{eq: explicit Csys}--\eqref{eq: explicit width}
for the depth, MLP width, and the $N$-independent constant
$C_{\mathrm{sys}}$, including their parameter dependence. The proof is constructive, and we sketch the network
construction here; a full overview is given in
Appendix~\ref{app: overview of transformer construction}. The transformer is
built in three phases. The \emph{read-in phase}
(Appendix~\ref{app: read-in phase}) uses three blocks to prepare quantities
needed by the iterative solver, including approximations of $\mD^{-1}$ and
$\mD^{-1}\vy$. The \emph{iteration phase}
(Appendix~\ref{app: iteration phase}) uses pairs of transformer blocks to
approximately implement preconditioned Richardson updates for the
coefficients $\vw$. Finally, the \emph{read-out phase}
(Appendix~\ref{app: read-out phase}) extracts the prediction at the test
token $\vx_{N+1}$, yielding an approximation of
$\sum_{i=1}^N w_i^*\mathcal{K}(\vx_i,\vx_{N+1})$.

\paragraph{Algorithm-Architecture correspondence.}
Across these phases, the two components of a standard transformer play complementary algorithmic roles. Softmax attention handles the
\emph{cross-token} operations: it produces row-normalized kernel interactions
and computes products such as $\mD^{-1}\KerMat\vw^{(\ell)}$. ReLU MLPs act
locally at each token, approximating the \emph{intra-token} scalar arithmetic,
including elementwise operations such as $\mD^{-1}\vy$ and
$\mD^{-1}\vw^{(\ell)}$. This separation of roles is central to our transformer construction: its forward
pass implements an inexact preconditioned Richardson iteration within the
standard softmax-attention/MLP architecture.

\paragraph{Scaling of depth and width.}
The depth and width scalings in Theorem~\ref{theorem: informal version of KRR ICL}
come from two different sources. The depth
$\mathcal{O}(\log(1/\epsilon))$ is an \emph{algorithmic count} obtained from
the error analysis of the inexact preconditioned Richardson iteration
(Appendix~\ref{app: error analysis of inexact preconditioned Richardson iterations}).
With a suitable choice of step size $\eta$, the error-propagation matrix of the
underlying iteration is contractive. Consequently, a logarithmic number of
inexact iteration steps suffices to reach the desired prediction accuracy. In contrast, the MLP width scaling is an \emph{approximation count}: ReLU MLPs
are used to approximate the tokenwise scalar arithmetic operations required by
the construction
(Appendix~\ref{app: approximation of single-layer ReLU Networks}). Common single-layer ReLU approximation results used in the ICL literature are
input-dimension-free, thereby avoiding the curse of dimensionality
\citep{bach2017breaking,bai2023transformers}. In our construction, however, the
MLPs only need to approximate a few specific one-dimensional smooth functions,
such as $x/(1-x)$, $x^2$, and $1/x$, on intervals of interest. We can therefore
use explicit linear spline approximations, which are representable by
single-layer ReLU networks and achieve sharper interpolation rates
\citep{de1978practical,devore1998nonlinear,devore2021neural}. By carefully
estimating the relevant approximation intervals and error tolerances, we obtain
the required MLP width scaling $\mathcal{O}(\sqrt{N/\epsilon})$ for each
scalar approximation. Together, the depth scaling reflects the convergence rate
of the inexact preconditioned Richardson iteration, while the width scaling
reflects the cost of approximating the local scalar arithmetic.

\paragraph{Variation of the regularization parameter.}
Theorem~\ref{theorem: informal version of KRR ICL} assumes that
$\lambda_0$ is a constant independent of the prompt length $N$.
Recall that the regularization parameter in KRR is related to
$\lambda_0$ through $\lambda=\lambda_0N$. Thus, if $\lambda$ is
held fixed as $N$ varies, then $\lambda_0=\lambda/N$ decreases with
the prompt length. Substituting $\lambda_0=\lambda/N$ into the explicit
expressions for the system constant, depth, and MLP width in
\eqref{eq: explicit Csys}, \eqref{eq: explicit depth}, and
\eqref{eq: explicit width} shows that, to guarantee a final prediction
error of at most $\tilde{\epsilon}$, it suffices to use depth $L=\mathcal{O}\!\left(N\log(N/\tilde{\epsilon})\right)$, MLP width $W=\mathcal{O}\!\left(N^{11/4}/\sqrt{\tilde{\epsilon}}\right)$. Thus, unlike the constant-$\lambda_0$ regime, fixing the KRR
regularization parameter $\lambda$, our sufficient depth bound increases with the number $N$ of in-context examples. As we
will see in the numerical experiments, this prediction is qualitatively consistent with the observation that more transformer blocks are required to implement Gaussian KRR for longer prompts when $\lambda$ is fixed at the training noise variance.

\section{Experiments}
\label{sec:experiments}

We empirically show that trained transformers progressively approach exact KRR estimator across depth and further investigate how this refinement
relates to classical iterative KRR solvers.

\textbf{Data generation.} We consider dynamically generated Gaussian process (GP) regression tasks. For each training and evaluation sequence, we sample inputs
$\vx_i\stackrel{\mathrm{i.i.d.}}{\sim}\mathcal{P}$, $i\in[N+1]$,
and $f\sim\mathrm{GP}(0,\mathcal{K})$ with Gaussian kernel bandwidth $v$. Context labels are $y_i=f(\vx_i)+\epsilon_i$, with $\epsilon_i\sim\mathcal{N}(0,\sigma_{\mathrm{train}}^2)$ for $i\in[N]$. We set $N=40$, $d=5$, $\sigma_{\mathrm{train}}=0.05$,
and $v=1$. We use three input distributions $\mathcal{P}$: Uniform
$\mathrm{Unif}([-1,1]^d)$, Gaussian $\mathcal{N}(0,0.6^2\mI_d)$,
and Spherical $\mathrm{Unif}(\mathbb{S}^{d-1})$.

\textbf{Transformer architecture and pretraining.}
For each input distribution, we train a GPT-2-style decoder-only
transformer, with $12$ layers and $8$ softmax-attention heads, from scratch \citep{radford2019language,garg2022can},
minimizing noisy-label MSE at every in-context position $n\in[N]$. Pretraining runs for $500{,}000$ optimizer steps with batch size $64$. All methods are evaluated on the same $256$ fresh GP sequences per
seed across $5$ fixed evaluation seeds. Full details are in Appendix~\ref{app:training_details}.

\subsection{Progressive alignment with exact KRR across depth}
\label{subsec:convergence}

\paragraph{Prediction-error alignment.} Figure~\ref{fig:convergence-a} reports the transformer's layer-wise prediction
MSE on spherical inputs for selected context lengths
$n\in\{2,10,15,20,25,30,35,40\}$, where $N=40$ is the maximum context length. The exact-KRR predictor is obtained by directly solving the kernel linear system
in \eqref{eq: unnormalized kernel system}. For a fixed context length $n$, let $\mathcal{A}_{n}^{(s,i)}$ denote the predictor produced using the first $n$ labeled context examples $\{\vx_{j}^{(s,i)},y_j^{(s,i)}\}_{j=1}^n$ of sequence $i$ under evaluation seed $s$, where $\mathcal{A}=\mathrm{TF}^{(\ell)}$ denotes the
predictor from transformer block $\ell$ and $\mathcal{A}=\mathrm{KRR}$ denotes
the exact-KRR predictor. We compute the mean-squared error as 
$
\mathrm{MSE}_{\mathcal{A}}(n)
=
\mathop{\mathbb{E}}\limits_{\substack{s\in[5],i\in[256]}}
\left[
\left(
\mathcal{A}_{n}^{(s,i)}(\vx_{n+1}^{(s,i)})
-y_{n+1}^{(s,i)}
\right)^2
\right],
$
where the expectation denotes the empirical average over the $5$ evaluation
seeds and $256$ held-out sequences per seed. Solid curves show transformer MSEs and horizontal dashed lines the exact-KRR references, evaluated on the same prediction problems. The transformer's MSE decreases overall with depth, most
sharply in the first few layers. Shorter contexts approach their
KRR references sooner, while longer contexts benefit from additional
depth. By the final layers, transformer MSEs approach the corresponding
exact-KRR values across all displayed context lengths. 

\paragraph{Induced-weight alignment.} The exact-KRR predictor admits an explicit representation through its dual
coefficient vector, whereas a transformer layer exposes only scalar predictions
on query inputs. To compare them in the same dual coefficient space, for each regression problem defined by a given context length, we view each transformer block as a kernel machine, and interpret its dual coefficients as the block's \emph{induced weights} at that context. Following
prior work on output-based weight recovery in linear-regression ICL
\citep{akyurek2023learning,fu2024transformers}, we adapt this procedure to KRR. For each sequence and fixed context length $n$, we sample $P=200$ independent
probe queries $\vq_r\in\mathbb{R}^d$, $r\in[P]$, from the same input distribution and
construct the probe matrix $\mPhi\in\mathbb{R}^{P\times n}$, with $\mPhi_{rj}:=\mathcal{K}(\vq_r,\vx_j)$. Applying the trained block-wise linear probe to the query-position hidden
states produces the prediction vector
$\bm{F}\in\mathbb{R}^{P}$ at block $\ell$. Then the transformer-induced weight vector can be computed by 
$\underset{\vw\in\mathbb{R}^{n}}{\operatorname{argmin}}
\left\{
\left\|
\mF
-
\mPhi\vw
\right\|_2^2+\tau\left\|\vw\right\|_2^2
\right\}
=\left(\mPhi^\top\mPhi+\tau\mI\right)^{-1}
\mPhi^\top\bm{F}$,
with $\tau=10^{-5}$. For sequence $i$ under evaluation seed $s$, let
$\vw_{\mathrm{TF}^{(\ell)},n}^{(s,i)}$ denote the transformer-induced weights and let
$\vw_{\mathrm{KRR},n}^{(s,i)}$ denote the exact-KRR dual coefficients
for the same context of length $n$. We define their similarity as $\mathrm{SimW}\parens{\mathrm{TF}^{(\ell)},\mathrm{KRR}\ ;\ n}=\mathop{\mathbb{E}}\limits_{\substack{s\in[5],i\in[256]}}
\mathcal{C}\parens{
\vw_{\mathrm{TF}^{(\ell)},n}^{(s,i)},
\vw_{\mathrm{KRR},n}^{(s,i)}}$, where $\mathcal{C}(\vu,\vv):=\frac{\langle\vu,\vv\rangle}{\|\vu\|_2 \|\vv\|_2}$ is the cosine similarity. Figure~\ref{fig:convergence-b} reports $\mathrm{SimW}$ at the maximal considered context length, $n=40$, showing that the transformer's induced weights become progressively more aligned with the exact-KRR coefficients as depth increases across all three input distributions. Figure~\ref{fig-app:simw-depth-by-context} extends this comparison across context lengths for all three input distributions and five random seeds, showing that longer contexts generally require greater depth to reach comparable levels of alignment. Our experiments use a fixed KRR regularization parameter $\lambda$, determined by the noise variance. As discussed after Theorem~\ref{theorem: informal version of KRR ICL}, fixing $\lambda$ implies $\lambda_0=\lambda/N$, and our theoretical construction therefore predicts that the depth increases with the number of in-context examples. The empirical dependence on context length is qualitatively consistent with this prediction. We finally remark that the induced weights used to compute $\mathrm{SimW}$ are externally recovered dual coefficients from the transformer's predictions; they are not transformer parameters, attention weights, or directly observed internal states.

\subsection{Ablation studies}\label{section: ablation studies}
\textbf{Robustness to linear probing.} Our layer-wise predictions are obtained using linear probes trained with the
transformer backbone frozen, raising the question whether the trends in Figure~\ref{fig:convergence} are induced by the probes rather than by the
backbone itself. We first test this using two ablations that replace the trained
backbone with a randomly initialized one: one inherits the probes trained on
the trained backbone, while the other retrains the probes on the random
backbone. Figure~\ref{fig-app: random backbones} shows that, at $n=40$, neither
control reproduces the trained model's behavior: SimW remains near zero across
all layers, while the MSE stays close to its initial-layer value. As a complementary control, we remove the separately trained probes and read
each intermediate hidden state using the trained transformer's native final
LayerNorm and prediction head. At depth $\ell$, we apply this native readout
directly to the hidden states and treat the resulting scalar as the layer-$\ell$
prediction. Figure~\ref{fig-app: probing vs native readout} compares this readout with the
per-layer linear probes. Although the native readout yields a noisier MSE
trajectory at intermediate depths, its predictions progressively approach
those of the probes and agree closely in the final layers. The
layer-wise SimW trend seems robust to the choice of readout. These results suggest that the probes mainly \emph{calibrate} intermediate hidden states into scalar predictions. Together, we conclude that the linear probe does not manufacture the observed trend; rather, the transformer backbone progressively refines its
predictions toward the KRR predictor across depth.

\textbf{Mismatched estimator kernel.} We examine whether the observed alignment is specific to the Gaussian kernel used to generate the training data. To this end, we keep the model, spherical inputs, and Gaussian labels fixed, but replace the Gaussian estimator kernel with the arc-cosine kernel \citep{cho2009kernel} in both the exact-KRR and the induced-weight recovery comparison. Figure~\ref{fig1-app: kernel mismatched} shows substantially lower $\mathrm{SimW}$ with arc-cosine KRR: alignment initially increases but declines at later layers, while alignment with Gaussian KRR continues to improve. The prediction error shows a similar separation: the transformer remains close to the Gaussian KRR error but clearly separated from that of arc-cosine KRR. These results suggest that the deeper layers progressively refine the prediction toward Gaussian KRR specifically, rather than toward KRR solutions based on arbitrary kernels.

\textbf{Noise-level ablation.} We examine whether a Transformer trained at noise level
$\sigma_{\mathrm{train}}$ behaves like KRR with
$\lambda$ being $\sigma_{\mathrm{train}}^{2}$. To this end, we vary the noise level $\sigma_{\mathrm{test}}$, measure
the error against the noiseless labels, and report the ratio
$\mathrm{MSE}_{\mathrm{TF}}/\mathrm{MSE}_{\mathrm{KRR}}$, where KRR uses
$\lambda=\sigma_{\mathrm{train}}^{2}$. We report the results in Figure \ref{fig1-app: noise level abalation}. We find that, across all three input distributions, the ratio generally moves closer to one in the later layers. Whether the late-layer ratio lies above or below one depends on the test
noise: it remains above one for small $\sigma_{\mathrm{test}}$ but falls below
one for large $\sigma_{\mathrm{test}}$. Both behaviors may be explained by
finite depth (early stopping) acting as additional regularization, which is harmful when the labels
are nearly clean but beneficial when they are noisy.

\subsection{Comparison with classical KRR solvers} \label{subsection: comparison with other solvers}

\textbf{Classical baselines.}
We compare the per-layer transformer predictions against the exact KRR solution and five classical iterative methods for KRR. The exact KRR solution is
$\vw^\star = (\KerMat + \lambda\mI)^{-1}\vy$, which is also the common limiting solution of the five iterative methods considered below. Three are linear-system solvers applied to the kernel system \eqref{eq: unnormalized kernel system}: \emph{plain Richardson iteration}, \emph{preconditioned Richardson iteration}, and \emph{conjugate gradient}. The other two are loss-function optimizers applied to the RKHS-regularized loss $L(\vw) = \frac{1}{2} \|\KerMat\vw - \vy\|^{2}
+ \frac{1}{2}\lambda\vw^{\top}\KerMat\vw$: \emph{gradient descent} and \emph{Nesterov gradient descent}. Full details for each method are provided in Appendix~\ref{app:solvers}.

\textbf{Best-matching step.}
We explore which iterate of each classical solver is most similar to the
transformer-induced weights. For a solver $\mathcal{A}$, let
$\vw_{\mathcal{A}^{(t)},n}^{(s,i)}$ denote its coefficient vector after
iteration $t$ on sequence $i$ under evaluation seed $s$. We define the best-matching iteration as
$
t_{\mathcal{A}}(\ell;n)
:=
\mathop{\mathbb{E}}\limits_{s\in[5]}
\left[
\underset{t\in[T]}{\operatorname{argmax}}
\;
\mathop{\mathbb{E}}\limits_{i\in[256]}
\mathcal{C}\parens{
\vw_{\mathrm{TF}^{(\ell)},n}^{(s,i)},
\vw_{\mathcal{A}^{(t)},n}^{(s,i)}
}
\right],
$
where $T$ is chosen sufficiently large to provide enough iterations for each solver to converge.

\begin{figure}[!t]
\centering
\subfiguretopcaptrue

\makeatletter
\renewcommand*{\@thesubfigure}{%
  \hspace*{4em}\thesubfigure\hskip\subfiglabelskip
}
\makeatother

\subfigure[All references]{%
  \includegraphics[width=0.32\textwidth]
    {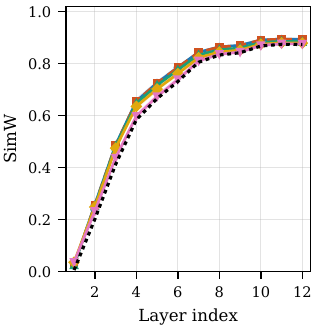}%
  \label{fig:nonid-a}%
}\hfill
\subfigure[Late-layer zoom]{%
  \includegraphics[width=0.32\textwidth]
    {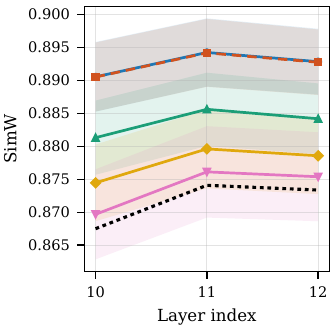}%
  \label{fig:nonid-b}%
}\hfill
\subfigure[Best-matching step]{%
  \includegraphics[width=0.32\textwidth]
    {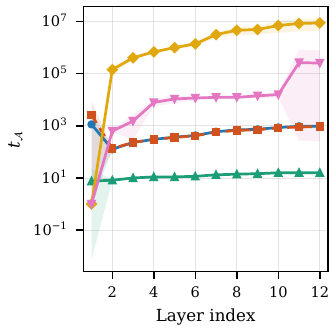}%
  \label{fig:nonid-c}%
}

\par\vspace{1pt}

\includegraphics[width=0.8\textwidth]
  {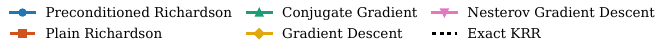}

\caption{\textbf{Best-match SimW and the corresponding solver iterations for spherical inputs}
at $n=40$. SimW generally increases with depth, while its late-layer values
are closely clustered across the five classical solvers. These trends are
consistent with progressive refinement toward their shared exact KRR solution.}
\label{fig:nonidentification}
\end{figure}

\textbf{Results.}
Figure~\ref{fig:nonid-a} reports the best-match SimW at $n=40$ for spherical inputs, Figure~\ref{fig:nonid-b} enlarges the late-layer values, and Figure~\ref{fig:nonid-c} shows the corresponding best-matching steps. The best-match SimW generally increases with depth for all five classical methods, while the late-layer values remain closely clustered. This is consistent with the fact that all five methods converge to the same exact KRR solution: as the transformer-induced weights approach this common endpoint, they become simultaneously similar to late iterates of multiple solvers. Thus, SimW supports progressive refinement toward the KRR solution but does not, by itself, identify the underlying iterative algorithm. Figure~\ref{fig:nonid-c} primarily reflects the relative convergence speeds of the solvers; in particular, plain and preconditioned Richardson converge at nearly the same rate in this setting. The role of preconditioned Richardson in our analysis instead comes from the constructive theory, through the correspondence between softmax attention and row-sum preconditioning. We finally remark that the same pattern holds for Gaussian and Uniform inputs in Figures \ref{fig2-app: gaussian nonidentification} and \ref{fig2-app: uniform nonidentification}.

\section{Discussion}
\label{sec:discussion}
Our construction establishes inexact preconditioned Richardson iteration as a mechanism by which standard softmax-attention transformers can approximate in-context Gaussian KRR. Our experiments support progressive refinement toward KRR, but may not conclusively establish that trained transformers internally implement this particular solver. Since our analysis relies on the RKHS representer theorem, one possible direction is to investigate whether analogous solver constructions extend to reproducing kernel Banach spaces \citep{parhi2021banach,bartolucci2023understanding,wang2024sparse,wang2025hypothesis}.

\bibliographystyle{plainnat}
\bibliography{references}

\begin{thebibliography}{47}
\providecommand{\natexlab}[1]{#1}
\providecommand{\url}[1]{\texttt{#1}}
\expandafter\ifx\csname urlstyle\endcsname\relax
  \providecommand{\doi}[1]{doi: #1}\else
  \providecommand{\doi}{doi: \begingroup \urlstyle{rm}\Url}\fi

\bibitem[Ahn et~al.(2023)Ahn, Cheng, Daneshmand, and Sra]{ahn2023transformers}
Kwangjun Ahn, Xiang Cheng, Hadi Daneshmand, and Suvrit Sra.
\newblock Transformers learn to implement preconditioned gradient descent for in-context learning.
\newblock In \emph{Advances in Neural Information Processing Systems}, volume~36, pages 45614--45650, 2023.

\bibitem[Aky{\"u}rek et~al.(2023)Aky{\"u}rek, Schuurmans, Andreas, Ma, and Zhou]{akyurek2023learning}
Ekin Aky{\"u}rek, Dale Schuurmans, Jacob Andreas, Tengyu Ma, and Denny Zhou.
\newblock What learning algorithm is in-context learning? investigations with linear models.
\newblock In \emph{The Eleventh International Conference on Learning Representations}, 2023.

\bibitem[Aky{\"u}rek et~al.(2024)Aky{\"u}rek, Wang, Kim, and Andreas]{akyurek2024incontext}
Ekin Aky{\"u}rek, Bailin Wang, Yoon Kim, and Jacob Andreas.
\newblock In-context language learning: Architectures and algorithms.
\newblock In \emph{International Conference on Machine Learning}. PMLR, 2024.

\bibitem[Bach(2017)]{bach2017breaking}
Francis Bach.
\newblock Breaking the curse of dimensionality with convex neural networks.
\newblock \emph{Journal of Machine Learning Research}, 18\penalty0 (19):\penalty0 1--53, 2017.

\bibitem[Bach(2024)]{bach2024book}
Francis Bach.
\newblock \emph{Learning Theory from First Principles}.
\newblock MIT Press, 2024.

\bibitem[Bai et~al.(2023)Bai, Chen, Wang, Xiong, and Mei]{bai2023transformers}
Yu~Bai, Fan Chen, Huan Wang, Caiming Xiong, and Song Mei.
\newblock Transformers as statisticians: Provable in-context learning with in-context algorithm selection.
\newblock In \emph{Advances in Neural Information Processing Systems}, volume~36, pages 57125--57211, 2023.

\bibitem[Bartolucci et~al.(2023)Bartolucci, De~Vito, Rosasco, and Vigogna]{bartolucci2023understanding}
Francesca Bartolucci, Ernesto De~Vito, Lorenzo Rosasco, and Stefano Vigogna.
\newblock Understanding neural networks with reproducing kernel {B}anach spaces.
\newblock \emph{Applied and Computational Harmonic Analysis}, 62:\penalty0 194--236, 2023.

\bibitem[Bhattamishra et~al.(2024)Bhattamishra, Patel, Blunsom, and Kanade]{bhattamishra2024understanding}
Satwik Bhattamishra, Arkil Patel, Phil Blunsom, and Varun Kanade.
\newblock Understanding in-context learning in transformers and {LLM}s by learning to learn discrete functions.
\newblock In \emph{International Conference on Learning Representations}, 2024.

\bibitem[Brown et~al.(2020)Brown, Mann, Ryder, Subbiah, Kaplan, Dhariwal, Neelakantan, Shyam, Sastry, Askell, et~al.]{brown2020language}
Tom Brown, Benjamin Mann, Nick Ryder, Melanie Subbiah, Jared~D Kaplan, Prafulla Dhariwal, Arvind Neelakantan, Pranav Shyam, Girish Sastry, Amanda Askell, et~al.
\newblock Language models are few-shot learners.
\newblock In \emph{Advances in Neural Information Processing Systems}, volume~33, pages 1877--1901, 2020.

\bibitem[Caponnetto and De~Vito(2007)]{caponnetto2007optimal}
Andrea Caponnetto and Ernesto De~Vito.
\newblock Optimal rates for the regularized least-squares algorithm.
\newblock \emph{Foundations of Computational Mathematics}, 7\penalty0 (3):\penalty0 331--368, 2007.

\bibitem[Chen et~al.(2024)Chen, Sheen, Wang, and Yang]{chen2024training}
Siyu Chen, Heejune Sheen, Tianhao Wang, and Zhuoran Yang.
\newblock Training dynamics of multi-head softmax attention for in-context learning: Emergence, convergence, and optimality.
\newblock In \emph{The Thirty Seventh Annual Conference on Learning Theory}, volume 247, pages 4573--4573. PMLR, 2024.

\bibitem[Cheng et~al.(2024)Cheng, Chen, and Sra]{cheng2023transformers}
Xiang Cheng, Yuxin Chen, and Suvrit Sra.
\newblock Transformers implement functional gradient descent to learn non-linear functions in context.
\newblock In \emph{International Conference on Machine Learning}, volume 235, pages 8002--8037. PMLR, 2024.

\bibitem[Cho and Saul(2009)]{cho2009kernel}
Youngmin Cho and Lawrence~K. Saul.
\newblock Kernel methods for deep learning.
\newblock In \emph{Advances in Neural Information Processing Systems 22}, 2009.

\bibitem[Cole et~al.(2026)Cole, Zhao, Lu, and Zhang]{cole2026context}
Frank Cole, Yuxuan Zhao, Yulong Lu, and Tianhao Zhang.
\newblock In-context learning of linear dynamical systems with transformers: Approximation bounds and depth-separation.
\newblock \emph{Advances in Neural Information Processing Systems}, 38:\penalty0 169441--169479, 2026.

\bibitem[Collins et~al.(2024)Collins, Parulekar, Mokhtari, Sanghavi, and Shakkottai]{collins2024context}
Liam Collins, Advait Parulekar, Aryan Mokhtari, Sujay Sanghavi, and Sanjay Shakkottai.
\newblock In-context learning with transformers: Softmax attention adapts to function lipschitzness.
\newblock In \emph{Advances in Neural Information Processing Systems}, volume~37, pages 92638--92696, 2024.

\bibitem[De~Boor(1978)]{de1978practical}
Carl De~Boor.
\newblock \emph{A practical guide to splines}, volume~27.
\newblock springer New York, 1978.

\bibitem[Demir and Dogan(2025)]{demir2025asymptotic}
Samet Demir and Zafer Dogan.
\newblock Asymptotic study of in-context learning with random transformers through equivalent models, 2025.

\bibitem[DeVore(1998)]{devore1998nonlinear}
Ronald DeVore.
\newblock Nonlinear approximation.
\newblock \emph{Acta numerica}, 7:\penalty0 51--150, 1998.

\bibitem[DeVore et~al.(2021)DeVore, Hanin, and Petrova]{devore2021neural}
Ronald DeVore, Boris Hanin, and Guergana Petrova.
\newblock Neural network approximation.
\newblock \emph{Acta Numerica}, 30:\penalty0 327--444, 2021.

\bibitem[Dragutinovi{\'c} et~al.(2025)Dragutinovi{\'c}, Saxe, and Singh]{dragutinovic2025softmax}
Sara Dragutinovi{\'c}, Andrew~M. Saxe, and Aaditya~K. Singh.
\newblock Softmax $\geq$ linear: Transformers may learn to classify in-context by kernel gradient descent.
\newblock \emph{arXiv preprint arXiv:2510.10425}, 2025.

\bibitem[Fu et~al.(2024)Fu, Chen, Jia, and Sharan]{fu2024transformers}
Deqing Fu, Tian-Qi Chen, Robin Jia, and Vatsal Sharan.
\newblock Transformers learn to achieve second-order convergence rates for in-context linear regression.
\newblock In \emph{Advances in Neural Information Processing Systems}, volume~37, pages 98675--98716, 2024.

\bibitem[Garg et~al.(2022)Garg, Tsipras, Liang, and Valiant]{garg2022can}
Shivam Garg, Dimitris Tsipras, Percy~S. Liang, and Gregory Valiant.
\newblock What can transformers learn in-context? a case study of simple function classes.
\newblock In \emph{Advances in Neural Information Processing Systems}, volume~35, pages 30583--30598, 2022.

\bibitem[Goel and Bartlett(2024)]{goel2024can}
Gautam Goel and Peter Bartlett.
\newblock Can a transformer represent a {K}alman filter?
\newblock In \emph{6th Annual Learning for Dynamics \& Control Conference}, pages 1502--1512. PMLR, 2024.

\bibitem[Han et~al.(2025)Han, Wang, Zhao, and Ji]{han2025understanding}
Chi Han, Ziqi Wang, Han Zhao, and Heng Ji.
\newblock Understanding emergent in-context learning from a kernel regression perspective.
\newblock \emph{Transactions on Machine Learning Research}, 2025.

\bibitem[Huang et~al.(2024)Huang, Cheng, and Liang]{huang2023context}
Yu~Huang, Yuan Cheng, and Yingbin Liang.
\newblock In-context convergence of transformers.
\newblock In \emph{Proceedings of the 41st International Conference on Machine Learning}, volume 235, pages 19660--19722. PMLR, 2024.

\bibitem[Kang et~al.(2026)Kang, Lee, and Cheng]{kang2026transformers}
Gyeonghun Kang, Changwoo~J Lee, and Xiang Cheng.
\newblock Transformers can learn posterior predictive distributions in-context.
\newblock \emph{arXiv preprint arXiv:2605.26713}, 2026.

\bibitem[Li et~al.(2024)Li, Cao, Gao, He, Liu, Klusowski, Fan, and Wang]{li2024one}
Zihao Li, Yuan Cao, Cheng Gao, Yihan He, Han Liu, Jason~M Klusowski, Jianqing Fan, and Mengdi Wang.
\newblock One-layer transformer provably learns one-nearest neighbor in context.
\newblock \emph{Advances in Neural Information Processing Systems}, 37:\penalty0 82166--82204, 2024.

\bibitem[Lu et~al.(2025)Lu, Letey, Zavatone-Veth, Maiti, and Pehlevan]{lu2024asymptotic}
Yue~M. Lu, Mary~I. Letey, Jacob~A. Zavatone-Veth, Anindita Maiti, and Cengiz Pehlevan.
\newblock Asymptotic theory of in-context learning by linear attention.
\newblock \emph{Proceedings of the National Academy of Sciences}, 122\penalty0 (28):\penalty0 e2502599122, 2025.

\bibitem[Mahankali et~al.(2024)Mahankali, Hashimoto, and Ma]{mahankali2024optimal}
Arvind~V. Mahankali, Tatsunori~B. Hashimoto, and Tengyu Ma.
\newblock One step of gradient descent is provably the optimal in-context learner with one layer of linear self-attention.
\newblock In \emph{International Conference on Learning Representations}, 2024.

\bibitem[Parhi and Nowak(2021)]{parhi2021banach}
Rahul Parhi and Robert~D Nowak.
\newblock Banach space representer theorems for neural networks and ridge splines.
\newblock \emph{Journal of Machine Learning Research}, 22\penalty0 (43):\penalty0 1--40, 2021.

\bibitem[Radford et~al.(2019)Radford, Wu, Child, Luan, Amodei, and Sutskever]{radford2019language}
Alec Radford, Jeffrey Wu, Rewon Child, David Luan, Dario Amodei, and Ilya Sutskever.
\newblock Language models are unsupervised multitask learners.
\newblock \emph{OpenAI Blog}, 2019.

\bibitem[Rasmussen and Williams(2006)]{rasmussen2006gaussian}
Carl~Edward Rasmussen and Christopher K.~I. Williams.
\newblock \emph{Gaussian Processes for Machine Learning}.
\newblock MIT Press, 2006.

\bibitem[Ryaben’kii and Tsynkov(2006)]{richardson_book}
Victor Ryaben’kii and Semyon Tsynkov.
\newblock \emph{A Theoretical Introduction to Numerical Analysis}.
\newblock Chapman \& Hall, 2006.

\bibitem[Sander and Peyr{\'e}(2024)]{sander2024towards}
Michael~E Sander and Gabriel Peyr{\'e}.
\newblock Towards understanding the universality of transformers for next-token prediction.
\newblock \emph{arXiv preprint arXiv:2410.03011}, 2024.

\bibitem[Sch{\"o}lkopf et~al.(2001)Sch{\"o}lkopf, Herbrich, and Smola]{scholkopf2001generalized}
Bernhard Sch{\"o}lkopf, Ralf Herbrich, and Alex~J Smola.
\newblock A generalized representer theorem.
\newblock In \emph{International conference on computational learning theory}, pages 416--426. Springer, 2001.

\bibitem[Shen et~al.(2025)Shen, Hsu, Lai, and Liao]{shen2025understanding}
Zhaiming Shen, Alexander Hsu, Rongjie Lai, and Wenjing Liao.
\newblock Understanding in-context learning on structured manifolds: Bridging attention to kernel methods.
\newblock \emph{arXiv preprint arXiv:2506.10959}, 2025.

\bibitem[Shewchuk(1994)]{cg_rate}
Jonathan~R Shewchuk.
\newblock An introduction to the conjugate gradient method without the agonizing pain.
\newblock 1994.

\bibitem[Sun et~al.(2025)Sun, Jadbabaie, and Azizan]{sun2025feedforward}
Haoyuan Sun, Ali Jadbabaie, and Navid Azizan.
\newblock On the role of transformer feed-forward layers in nonlinear in-context learning, 2025.

\bibitem[Touvron et~al.(2023)Touvron, Lavril, Izacard, Martinet, Lachaux, Lacroix, Rozi{\`e}re, Goyal, Hambro, Azhar, et~al.]{touvron2023llama}
Hugo Touvron, Thibaut Lavril, Gautier Izacard, Xavier Martinet, Marie-Anne Lachaux, Timoth{\'e}e Lacroix, Baptiste Rozi{\`e}re, Naman Goyal, Eric Hambro, Faisal Azhar, et~al.
\newblock Llama: Open and efficient foundation language models.
\newblock \emph{arXiv preprint arXiv:2302.13971}, 2023.

\bibitem[Vaswani et~al.(2017)Vaswani, Shazeer, Parmar, Uszkoreit, Jones, Gomez, Kaiser, and Polosukhin]{vaswani2017attention}
Ashish Vaswani, Noam Shazeer, Niki Parmar, Jakob Uszkoreit, Llion Jones, Aidan~N Gomez, {\L}ukasz Kaiser, and Illia Polosukhin.
\newblock Attention is all you need.
\newblock \emph{Advances in neural information processing systems}, 30, 2017.

\bibitem[Vladymyrov et~al.(2024)Vladymyrov, von Oswald, Sandler, and Ge]{vladymyrov2024linear}
Max Vladymyrov, Johannes von Oswald, Mark Sandler, and Rong Ge.
\newblock Linear transformers are versatile in-context learners.
\newblock In \emph{Advances in Neural Information Processing Systems}, volume~37, pages 48784--48809, 2024.

\bibitem[von Oswald et~al.(2023)von Oswald, Niklasson, Randazzo, Sacramento, Mordvintsev, Zhmoginov, and Vladymyrov]{vonoswald2023transformers}
Johannes von Oswald, Eyvind Niklasson, Ettore Randazzo, Jo{\~a}o Sacramento, Alexander Mordvintsev, Andrey Zhmoginov, and Max Vladymyrov.
\newblock Transformers learn in-context by gradient descent.
\newblock In \emph{International Conference on Machine Learning}, pages 35151--35174. PMLR, 2023.

\bibitem[Wang et~al.(2024)Wang, Xu, and Yan]{wang2024sparse}
Rui Wang, Yuesheng Xu, and Mingsong Yan.
\newblock Sparse representer theorems for learning in reproducing kernel {B}anach spaces.
\newblock \emph{Journal of Machine Learning Research}, 25\penalty0 (93):\penalty0 1--45, 2024.

\bibitem[Wang et~al.(2025)Wang, Xu, and Yan]{wang2025hypothesis}
Rui Wang, Yuesheng Xu, and Mingsong Yan.
\newblock Hypothesis spaces for deep learning.
\newblock \emph{Neural Networks}, page 107995, 2025.

\bibitem[Zhang et~al.(2024{\natexlab{a}})Zhang, Frei, and Bartlett]{zhang2024trained}
Ruiqi Zhang, Spencer Frei, and Peter~L. Bartlett.
\newblock Trained transformers learn linear models in-context.
\newblock \emph{Journal of Machine Learning Research}, 25\penalty0 (49):\penalty0 1--55, 2024{\natexlab{a}}.

\bibitem[Zhang et~al.(2024{\natexlab{b}})Zhang, Wu, and Bartlett]{zhang2024linearblock}
Ruiqi Zhang, Jingfeng Wu, and Peter~L. Bartlett.
\newblock In-context learning of a linear transformer block: Benefits of the {MLP} component and one-step {GD} initialization.
\newblock In \emph{Advances in Neural Information Processing Systems}, volume~37, pages 18310--18361, 2024{\natexlab{b}}.

\bibitem[Zhang et~al.(2025)Zhang, Singh, Latham, and Saxe]{zhang2025training}
Yedi Zhang, Aaditya~K Singh, Peter~E Latham, and Andrew Saxe.
\newblock Training dynamics of in-context learning in linear attention.
\newblock \emph{arXiv preprint arXiv:2501.16265}, 2025.

\end{thebibliography}

\clearpage
\appendix

\section{Error Analysis of Inexact Preconditioned Richardson Iterations}\label{app: error analysis of inexact preconditioned Richardson iterations}
In this section, we present an error analysis of the inexact preconditioned Richardson iteration, an algorithm we will subsequently implement via a specifically constructed transformer. We proceed with a helpful lemma. 
\begin{lemma}\label{lemma: estimate of Kw and w*}
Let $\lambda_0>0$ and $\lambda=\lambda_0N$. Suppose that $\vw^*=[w_i^*:i\in[N]]$ is the solution of the linear system \eqref{eq: unnormalized kernel system} with the bounded data assumption (i.e., $\normMax{\vy}\leq B_y$). Then it holds that 
    \begin{equation*}
        \normMax{\mK\vw^*}\leq \frac{B_y}{\sqrt{\lambda_0}},\quad \normMax{\vw^*}\leq \frac{\parens{\frac{1}{\sqrt{\lambda_0}}+1}B_y}{\lambda_0 }  \frac{1}{N}.
    \end{equation*}
\end{lemma}
\begin{proof}
It is known by representer theorem that $f^*(x)=\sum_{i=1}^N w_i^*\mathcal{K}(\vx_i,\cdot)$ is the optimal solution of \eqref{eq: optimization problem in RKHS}. We denote $\mathcal{L}(f):=\frac{1}{N}\sum_{i\in[N]}\parens{f(\vx_i)-y_i}^2 + \lambda_0\norm{f}_{\mathcal{H}}^2$. The optimality of $f^*$ implies 
\begin{equation*}
\lambda_0\norm{f^*}_{\mathcal{H}}^2\leq \mathcal{L}(f^*)\leq \mathcal{L}(0)=\frac{1}{N}\sum_{i=1}^N y_i^2 \leq B_y^2,
\end{equation*}
and hence $\norm{f^*}_{\mathcal{H}}\leq \frac{B_y}{\sqrt{\lambda_0}}$. By the reproducing property in RKHS, i.e., $f^*(\vx)=\iprod{f^*,\mathcal{K}(\vx,\cdot)}_{\mathcal{H}}$, we get that 
\begin{equation*}
    \abs{f^*(\vx)}\leq \norm{f^*}_{\mathcal{H}}\norm{\mathcal{K}(\vx,\cdot)}_{\mathcal{H}}=\norm{f^*}_{\mathcal{H}}\sqrt{\mathcal{K}(\vx,\vx)}\leq \frac{B_y}{\sqrt{\lambda_0}},\quad \text{for all }\vx\in\mathbb{R}^d.
\end{equation*}
Note that $\mK\vw^*=[f^*(\vx_i):i\in[N]]$. Consequently, we obtain from the previous estimate that 
\begin{equation*}
\normMax{\mK\vw^*}=\max_{i\in[N]}\abs{f^*(\vx_i)}\leq \frac{B_y}{\sqrt{\lambda_0}}. 
\end{equation*}
Moreover, from the linear system \eqref{eq: unnormalized kernel system}, we get 
\begin{equation*}
\normMax{\vw^*}=\frac{\normMax{\mK\vw^*-\vy}}{\lambda}\leq \frac{\normMax{\mK\vw^*}+\normMax{\vy}}{\lambda}\leq \frac{\parens{\frac{1}{\sqrt{\lambda_0}}+1}B_y}{\lambda_0 }  \frac{1}{N}.
\end{equation*}
This completes the proof.
\end{proof}
We recall that the bounded data assumption (i.e., $\normTwo{\vx_i}\leq B_x$ for all samples $\vx_i$) and boundedness of the Gaussian kernel imply that for all $i\in[N]$, 
\begin{equation}\label{eq: density assumption}
N\kappa_{\min}\leq\sum_{j\in[N]} \mathcal{K}(\vx_i,\vx_j)\leq N
\end{equation}
where $\kappa_{\min}:=\exp\parens{-\frac{2B_x^2}{v^2}}$.
\begin{theorem}[Error Analysis for Inexact Preconditioned Richardson Iterations]\label{theorem: Error Analysis for Inexact Richardson iteration}
Suppose that the bounded data assumption holds with constants $B_x$ and $B_y$. Let initial value $\vw^{(0)}=\vzero$. Let $\epsflip,\epssq,\tildeepssq\in(0,1)$. Let $\lambda_0,\eta>0$ and $\lambda=\lambda_0N$. Consider inexact preconditioned Richardson iteration 
\begin{equation}\label{eq: TF inexact Richardson iteration}
    \vw\myl[\ell+1] = \vw^{(\ell)} + \eta \parens{\mD^{-1}\vy - \mD^{-1}\parens{\KerMat+\lambda\mI}\vw^{(\ell)}} + \eta\parens{\parens{\vy-\lambda\vw^{(\ell)}}\odot\scalingerror + \frac{\bm{\tau}_+ - \bm{\tau}_- - \lambda\widetilde{\bm{\tau}}_+^{(\ell)} + \lambda\widetilde{\bm{\tau}}_-^{(\ell)}}{4}},
\end{equation}
with
\begin{equation}\label{assumption: error magnitude /N /N2}
\normMax{\scalingerror}\leq\frac{\epsflip}{N},\quad \normMax{\bm{\tau}_+}, \normMax{\bm{\tau}_-}\leq\frac{\epssq}{N},\quad \normMax{\widetilde{\bm{\tau}}_+^{(\ell)}},\normMax{\widetilde{\bm{\tau}}_-^{(\ell)}}\leq\frac{\tildeepssq}{N^2}. 
\end{equation}
Define 
\begin{equation}\label{eq: matrix A, tilde K lambda, tilde K}
\mA:=\mI-\eta\parens{\lambda\ \diag(\scalingerror)+\mD^{-1/2}\parens{\KerMat+\lambda\mI}\mD^{-1/2}}.
\end{equation}
If the step size $\eta$ satisfies
\begin{equation}\label{eq: eta condition}
0<\eta<\frac{1}{\lambda_0\epsflip+1+\lambda_0/\kappa_{\min}}, 
\end{equation}
then it holds that
\begin{equation}\label{eq: norm of A is in (0,1)}
\normTwoToTwo{\mA}\leq 1-\eta\lambda_0\parens{1-\epsflip}\in(0,1),
\end{equation}
and \begin{equation}\label{eq: bound of norm w ell - w *}
\sqrt{N}\normTwo{\vw\myl[\ell] - \vw^*}\leq  \normTwoToTwo{\mA}^{\ell}\frac{\parens{\frac{1}{\sqrt{\lambda_0}}+1}B_y}{\lambda_0\sqrt{\kappa_{\min}}} + \frac{\eta\parens{\frac{B_y}{\sqrt{\lambda_0}}\epsflip+\parens{1+\lambda_0}\parens{\epssq+\tildeepssq}}}{\parens{1-\normTwoToTwo{\mA}}\sqrt{\kappa_{\min}}}.
    \end{equation}

\end{theorem}
\begin{proof}
We first prove \eqref{eq: norm of A is in (0,1)}. We denote $\mK_\lambda:=\lambda\diag(\scalingerror) + \mD^{-1/2}\parens{\KerMat+\lambda\mI}\mD^{-1/2}$. Then the matrix $\mA$ can be rewritten as $\mA=\mI-\eta \mK_\lambda$. Since $\mA$ is symmetric, $\normTwoToTwo{\mA}$ equals its maximum absolute eigenvalue. Note that the eigenvalues of $\mA$ satisfy 
$$\eig(\mA)=1-\eta\ \eig(\mK_\lambda).$$
By Weyl's inequality, the maximum eigenvalue of $\mK_\lambda$ is bounded by
\begin{align}
    \eig_{\max}(\mK_\lambda)&\leq \eig_{\max}(\lambda\diag(\scalingerror)) + \eig_{\max}(\mD^{-1/2}\KerMat\mD^{-1/2}) + \eig_{\max}(\lambda\mD^{-1})\nonumber\\
    &\leq \lambda_0\epsflip + 1 + \frac{\lambda_0}{\kappa_{\min}},\label{in the proof: upper bound of eigmax K lambda}
\end{align}
where we used assumption \eqref{assumption: error magnitude /N /N2}, relation \eqref{eq: density assumption}, $\lambda=\lambda_0N$ and $\eig_{\max}(\mD^{-1/2}\KerMat\mD^{-1/2})\leq1$. Assumption \eqref{eq: eta condition} of $\eta$ guarantees $\eta\ \eig_{\max}(\mK_\lambda) < 1$. Consequently, all eigenvalues of $\mA$ are strictly positive, i.e.,
$$\eig_{\min}(\mA) = 1-\eta\ \eig_{\max}(\mK_\lambda) > 0.$$
It follows that the operator norm of $\mA$ is exactly its maximum eigenvalue. Hence, 
$$\normTwoToTwo{\mA} = \eig_{\max}(\mA) = 1-\eta\ \eig_{\min}(\mK_\lambda).$$
Applying Weyl's inequality to the minimum eigenvalue gives
\begin{equation}\label{in the proof: lower bound of eigmin K lambda}
    \eig_{\min}(\mK_\lambda)\geq \eig_{\min}(\lambda\diag(\scalingerror)) + \eig_{\min}(\mD^{-1/2}\KerMat\mD^{-1/2}) + \eig_{\min}(\lambda\mD^{-1}) \geq -\lambda_0\epsflip + \lambda_0,
\end{equation}
where we used assumption \eqref{assumption: error magnitude /N /N2}, relation \eqref{eq: density assumption}, $\lambda=\lambda_0N$ and the fact that matrix $\mD^{-1/2}\KerMat\mD^{-1/2}$ is positive semi-definite. Therefore, the operator norm is bounded as
$$\normTwoToTwo{\mA} \leq 1 - \eta\lambda_0\parens{1-\epsflip}.$$
We finally verify the upper bound $1 - \eta\lambda_0\parens{1-\epsflip}$ is in $(0,1)$. It suffices to show that 
\begin{equation}\label{in the proof: in 0,1}
\eta\lambda_0\parens{1-\epsflip}\in(0,1).    
\end{equation}
We obtain from \eqref{in the proof: upper bound of eigmax K lambda} and \eqref{in the proof: lower bound of eigmin K lambda} that 
$\lambda_0(1-\epsflip)\leq \lambda_0\epsflip + 1 + \frac{\lambda_0}{\kappa_{\min}}$, and hence \eqref{in the proof: in 0,1} is clearly true due to assumption \eqref{eq: eta condition} for $\eta$. This completes the proof of \eqref{eq: norm of A is in (0,1)}.

In the following, we present the proof of \eqref{eq: bound of norm w ell - w *}. We subtract $\vw^*$ for both sides of \eqref{eq: TF inexact Richardson iteration}, and note that the exact solution $\vw^*$ satisfies $(\mK+\lambda\mI)\vw^*=\vy$. Then
\begin{align}
    \vw\myl[\ell+1] - \vw^* &= \vw^{(\ell)} - \vw^* + \eta \mD^{-1}\parens{\KerMat+\lambda\mI}\parens{\vw^*-\vw^{(\ell)}} + \eta\parens{\parens{(\mK+\lambda\mI)\vw^*-\lambda\vw^{(\ell)}}\odot\scalingerror + \bm{\Delta}^{(\ell)}}\nonumber\\
    &=\parens{\mI-\eta\parens{\lambda\diag(\scalingerror) + \mD^{-1}\parens{\KerMat+\lambda\mI}}}\parens{\vw^{(\ell)} - \vw^*} + \eta\parens{(\mK\vw^*)\odot\scalingerror + \bm{\Delta}^{(\ell)}}\label{in the proof: TF inexact Richardson iteration 2}
\end{align}
where 
\begin{equation*}
\bm{\Delta}^{(\ell)}:=\frac{\bm{\tau}_+ - \bm{\tau}_- - \lambda\widetilde{\bm{\tau}}_+^{(\ell)} + \lambda\widetilde{\bm{\tau}}_-^{(\ell)}}{4}. 
\end{equation*}
Note that 
\begin{align*}
    \mI-\eta\parens{\lambda\diag(\scalingerror) + \mD^{-1}\parens{\KerMat+\lambda\mI}} &= \mD^{-1/2}\parens{\mI-\eta\parens{\lambda\diag(\scalingerror)+\mD^{-1/2}\parens{\KerMat+\lambda\mI}\mD^{-1/2}}}\mD^{1/2},\\
    &= \mD^{-1/2}\mA\mD^{1/2}.
\end{align*}
This, combined with \eqref{in the proof: TF inexact Richardson iteration 2}, implies 
\begin{equation*}
\mD^{1/2}\parens{\vw\myl[\ell+1] - \vw^*}=\mA\underbrace{\mD^{1/2}\parens{\vw^{(\ell)} - \vw^*}}_{\text{denoted by }\ve\myl} + \underbrace{\mD^{1/2}\eta\parens{(\mK\vw^*)\odot\scalingerror + \bm{\Delta}^{(\ell)}}}_{\text{denoted by }\veps\myl}. 
\end{equation*}
We rewrite the above equation as $\ve\myl[\ell+1]=\mA\ve\myl + \veps\myl$, which  with an argument of recursion gives 
\begin{equation}\label{in the proof: eps ell <= bound}
\ve\myl[\ell+1] = \mA^{\ell+1}\ve^{(0)} + \sum_{j=0}^{\ell} \mA^j \veps^{(\ell-j)}. 
\end{equation}
By definition of $\veps\myl$, we have  
\begin{align*}
    \normTwo{\veps\myl} \leq \sqrt{N}\normMax{\veps\myl}&\leq \sqrt{N}\eta\normMax{\diag(\mD^{1/2})}\parens{\normMax{\mK\vw^*}\normMax{\scalingerror}+\normMax{\bm{\Delta}^{(\ell)}}}.
\end{align*}
Recall that \eqref{eq: density assumption} implies $\normMax{\diag(\mD^{1/2})}\leq \sqrt{N}$; Lemma \ref{lemma: estimate of Kw and w*} yields $\normMax{\mK\vw^*}\leq B_y/\sqrt{\lambda_0}$; assumption \eqref{assumption: error magnitude /N /N2} yields $\normMax{\scalingerror}\leq \epsflip/N$ and 
\begin{equation*}
    \normMax{\bm{\Delta}^{(\ell)}}\leq \frac{\normMax{\bm{\tau}_+}+\normMax{\bm{\tau}_-}+\lambda\parens{\normMax{\widetilde{\bm{\tau}}_+}+\normMax{\widetilde{\bm{\tau}}_-}}}{4}\leq \frac{\epssq+\lambda_0\tildeepssq}{2N}\leq \frac{\parens{1+\lambda_0}\parens{\epssq+\tildeepssq}}{N}.
\end{equation*}
Consequently, we have 
\begin{equation*}
\normTwo{\veps\myl}\leq \eta\parens{\frac{B_y}{\sqrt{\lambda_0}}\epsflip+\parens{1+\lambda_0}\parens{\epssq+\tildeepssq}} \equiv B_\epsilon.
\end{equation*}
Therefore, for all $\ell$, we obtain from \eqref{in the proof: eps ell <= bound} that 
    \begin{align*}
        \normTwo{\ve\myl}\leq \normTwoToTwo{\mA}^{\ell}\normTwo{\ve\myl[0]} + B_\epsilon\sum_{j=0}^{\ell-1}\normTwoToTwo{\mA}^j   \leq\normTwoToTwo{\mA}^{\ell}\normTwo{\ve\myl[0]} + \frac{B_\epsilon}{1-\normTwoToTwo{\mA}},
    \end{align*}
    where the last inequality follows from \eqref{eq: norm of A is in (0,1)}, i.e., $\normTwoToTwo{\mA}\in(0,1)$. Note that by \eqref{eq: density assumption}, we have 
    \begin{align*}
\normTwo{\ve\myl}&=\normTwo{\mD^{1/2}\parens{\vw\myl[\ell] - \vw^*}}\geq \sqrt{\kappa_{\min}N}\normTwo{\vw\myl[\ell] - \vw^*},
    \end{align*}
    and by Lemma \ref{lemma: estimate of Kw and w*}, we have 
    \begin{equation*}
        \normTwo{\ve\myl[0]}= \normTwo{\mD^{1/2}\parens{\vw\myl[0] - \vw^*}}=\normTwo{\mD^{1/2} \vw^*}\leq N\normMax{\vw^*}\leq \frac{\parens{\frac{1}{\sqrt{\lambda_0}}+1}B_y}{\lambda_0 }  
    \end{equation*}
    which combining with the previous estimate yields \eqref{eq: bound of norm w ell - w *}.
\end{proof}

\begin{theorem}[Entry-wise boundedness of $\vw^{(\ell)}$]\label{theorem: entry-wise boundedness of w ell}
Suppose that the bounded data assumption holds with constants $B_x$ and $B_y$. Let initial value $\vw^{(0)}=\vzero$. Let $\epssq,\tildeepssq\in(0,1)$ and $\epsflip\in(0,c)$ for a constant $c\in(0,1)$. Let $\lambda_0>0$ and $\eta$ satisfy \eqref{eq: eta condition}. Let $\vw^{(\ell)}$ be generated from inexact preconditioned Richardson iteration \eqref{eq: TF inexact Richardson iteration} with condition \eqref{assumption: error magnitude /N /N2}. Then for all $\ell\in\mathbb{N}$, 
\begin{equation}\label{eq: w ell is elementwise bounded}
\normMax{\vw^{(\ell)}}\leq B_w:= \frac{1}{\sqrt{N}} \parens{\frac{\parens{\frac{1}{\sqrt{\lambda_0}}+1}B_y}{\lambda_0\sqrt{\kappa_{\min}}} + \frac{\frac{B_y}{\sqrt{\lambda_0}}+2\parens{1+\lambda_0}}{\lambda_0\parens{1-c}\sqrt{\kappa_{\min}}}} + \frac{1}{N}\parens{\frac{\parens{\frac{1}{\sqrt{\lambda_0}}+1}B_y}{\lambda_0}}.
    \end{equation}
\end{theorem}

\begin{proof}
    We obtain from \eqref{eq: bound of norm w ell - w *} in Theorem \ref{theorem: Error Analysis for Inexact Richardson iteration} that 
    \begin{equation*}
\normMax{\vw^{(\ell)}-\vw^*}\leq \normTwo{\vw^{(\ell)}-\vw^*}\leq \frac{1}{\sqrt{N}} \parens{\normTwoToTwo{\mA}^{\ell}\frac{\parens{\frac{1}{\sqrt{\lambda_0}}+1}B_y}{\lambda_0\sqrt{\kappa_{\min}}} + \frac{\eta\parens{\frac{B_y}{\sqrt{\lambda_0}}+2\parens{1+\lambda_0}}}{\parens{1-\normTwoToTwo{\mA}}\sqrt{\kappa_{\min}}}},
    \end{equation*}
    where we used assumptions of $\epsflip,\epssq,\tildeepssq<1$. Furthermore, by \eqref{eq: norm of A is in (0,1)} in Theorem \ref{theorem: Error Analysis for Inexact Richardson iteration}, we apply the estimate of $\normTwoToTwo{\mA}$ and get 
\begin{equation*}
\normMax{\vw^{(\ell)}-\vw^*}\leq \frac{1}{\sqrt{N}} \parens{\frac{\parens{\frac{1}{\sqrt{\lambda_0}}+1}B_y}{\lambda_0\sqrt{\kappa_{\min}}} + \frac{\frac{B_y}{\sqrt{\lambda_0}}+2\parens{1+\lambda_0}}{\lambda_0\parens{1-c}\sqrt{\kappa_{\min}}}},
    \end{equation*}
where we used the assumption of $\epsflip<c$. The desired result \eqref{eq: w ell is elementwise bounded} can be immediately obtained by the triangle inequality $\normMax{\vw^{(\ell)}}\leq\normMax{\vw^{(\ell)}-\vw^*}+\normMax{\vw^*}$ and estimate of $\normMax{\vw^*}$ established in Lemma \ref{lemma: estimate of Kw and w*}. 
\end{proof}

\begin{proposition}\label{proposition: estimate of wil - wi* K(xi,xN+1)} Suppose that assumptions of Theorem \ref{theorem: entry-wise boundedness of w ell} hold and $\vw^*=[w_i^*:i\in[N]]$ is the solution of linear system \eqref{eq: unnormalized kernel system}. Then for all $\ell\in\mathbb{N}$, it holds that
\begin{align*}
\abs{\sum_{i=1}^N\parens{w_i^{(\ell)}-w_i^*}\mathcal{K}(\vx_i,\vx_{N+1})}\leq\parens{1-\eta\lambda_0\parens{1-c}}^{\ell}\frac{\parens{\frac{1}{\sqrt{\lambda_0}}+1}B_y}{\lambda_0\sqrt{\kappa_{\min}}} + \frac{\frac{B_y}{\sqrt{\lambda_0}}\epsflip+\parens{1+\lambda_0}\parens{\epssq+\tildeepssq}}{\lambda_0\parens{1-c}\sqrt{\kappa_{\min}}}.
\end{align*}    
\end{proposition}
\begin{proof}
By the Cauchy-Schwarz inequality, we have 
    \begin{align*}
\abs{\sum_{i=1}^N\parens{w_i^{(\ell)}-w_i^*}\mathcal{K}(\vx_i,\vx_{N+1})}\leq \norm{\vw^{(\ell)}-\vw^*}_2 \parens{\sum_{i=1}^N \parens{\mathcal{K}(\vx_i,\vx_{N+1})}^2}^{\frac{1}{2}}\leq \sqrt{N}\norm{\vw^{(\ell)}-\vw^*}_2. 
    \end{align*}
Moreover, it follows from \eqref{eq: norm of A is in (0,1)} and \eqref{eq: bound of norm w ell - w *} in Theorem \ref{theorem: Error Analysis for Inexact Richardson iteration} that 
    \begin{equation*}
\sqrt{N}\normTwo{\vw\myl[\ell] - \vw^*}\leq  \parens{1-\eta\lambda_0\parens{1-c}}^{\ell}\frac{\parens{\frac{1}{\sqrt{\lambda_0}}+1}B_y}{\lambda_0\sqrt{\kappa_{\min}}} + \frac{\frac{B_y}{\sqrt{\lambda_0}}\epsflip+\parens{1+\lambda_0}\parens{\epssq+\tildeepssq}}{\lambda_0\parens{1-c}\sqrt{\kappa_{\min}}}.
    \end{equation*}
The desired result immediately follows from the above estimates. 
\end{proof}   
\section{Approximation Power of Single-Layer ReLU Networks}\label{app: approximation of single-layer ReLU Networks}

In the construction of Transformers for in-context learning, we typically use the MLP layer (a single-layer ReLU neural network) to approximate necessary arithmetic operations. For general dimension-free approximation (which holds for arbitrary input dimensions), achieving a target accuracy of $\epsilon$ requires a network width that scales as $\mathcal{O}(\log(1/\epsilon)/\epsilon^2)$ \citep{bach2017breaking, bai2023transformers}. While broadly applicable, these existential guarantees are inherently non-constructive and suffer from a relatively loose dependence on the error tolerance. However, because the specific operations we need to approximate are strictly one-dimensional smooth functions, we can use continuous piecewise-linear splines to reduce this width requirement to $\mathcal{O}(1/\sqrt{\epsilon})$ \citep{de1978practical,devore2021neural,devore1998nonlinear}. Furthermore, unlike dimension-free approaches, spline-based methods provide an explicit mathematical recipe for constructing the exact neural parameters. Therefore, in this section, we adopt these constructive spline arguments to establish sharper, explicitly realizable network width bounds.

To explicitly control the approximation error, we present the following constructive lemma. This result follows directly from standard error analysis of piecewise linear interpolation and well-known ReLU representations of splines~\citep{de1978practical, devore2021neural}. We state it here in the form required for our theoretical constructions and omit the standard proof.

\begin{lemma}[Constructive Approximation of Single-layer ReLU network]\label{lemma: shallow ReLU general grid}
    Let $f \in C^2[M_1, M_2]$ and let $M_1 = x_0 < x_1 < \dots < x_n = M_2$ be an arbitrary partition of $[M_1, M_2]$ with $n \in \mathbb{N}$. We define the network parameters explicitly as follows:
    \begin{align*}
    d &= f(M_1),\quad a_s = 1,\quad b_s = -x_{s-1}, \quad \text{for } s = 1, \dots, n, \\
    c_s &= \displaystyle\begin{cases}
        \displaystyle\frac{f(x_1) - f(x_0)}{x_1 - x_0}, & s=1, \\
        \displaystyle\frac{f(x_s) - f(x_{s-1})}{x_s - x_{s-1}} - \frac{f(x_{s-1}) - f(x_{s-2})}{x_{s-1} - x_{s-2}}, & s = 2, \dots, n.
    \end{cases}
    \end{align*}
    Then the neural network $\phi(x) = \sum_{s=1}^n c_s \mathrm{ReLU}(a_s x + b_s) + d$ satisfies the uniform approximation bound
    \begin{equation*}
        \sup_{x \in [M_1, M_2]} |f(x) - \phi(x)| \le \frac{1}{8} \max_{k \in [n]} \left( (x_k - x_{k-1})^2 \sup_{x \in [x_{k-1}, x_k]} |f''(x)| \right).
    \end{equation*}
\end{lemma}
We now apply the general bound from Lemma \ref{lemma: shallow ReLU general grid} to derive explicit network widths for approximating specific one-dimensional smooth functions. Because our transformer blocks rely on these functions to compute intermediate variables, we formulate their approximation bounds as direct corollaries below.
\begin{corollary}[Approximation of $x^2$ on $\lbrack-\delta,\delta\rbrack$]\label{coro: approx x2}
    Let $\delta > 0$ and target accuracy $\epsilon > 0$. If the network width $n \in \mathbb{N}$ satisfies $n \ge \delta / \sqrt{\epsilon}$, then there exist parameters $a_s, b_s, c_s \in \mathbb{R}$ for $s \in [n]$, and $d \in \mathbb{R}$, such that the neural network $\phi(x) = \sum_{s=1}^n c_s \mathrm{ReLU}(a_s x + b_s) + d$ satisfies the uniform bound $\sup_{x \in [-\delta, \delta]} |x^2 - \phi(x)| \le \epsilon$.
\end{corollary}
\begin{proof}
    We apply Lemma \ref{lemma: shallow ReLU general grid} using a uniform partition $x_k = -\delta + 2\delta k / n$. The interval width is constantly $2\delta/n$. By assumption on $n$ and the fact of $f''(x) = 2$, we have  $$\frac{1}{8}(x_k - x_{k-1})^2 \sup_{x\in[x_{k-1},x_k]} |f''(x)|=\frac{1}{8} (4\delta^2 / n^2)(2) = \delta^2 / n^2\leq\epsilon.$$
\end{proof}

\begin{corollary}[Approximation of $x/(1-x)$ on $\lbrack0,\delta\rbrack$]\label{coro: approx x over 1-x adaptive}
    Let $\delta \in (0, 1)$ and target accuracy $\epsilon\in(0,1)$. If the network width $n \in \mathbb{N}$ satisfies 
    $$n \ge \frac{2 \left( (1-\delta)^{-1/2} - 1 \right)}{\sqrt{\epsilon}}=\mathcal{O}(1/\sqrt{\epsilon(1-\delta)}),$$ 
    then there exist parameters $a_s, b_s, c_s \in \mathbb{R}$ for $s \in [n]$, and $d \in \mathbb{R}$, such that the neural network $\phi(x) = \sum_{s=1}^n c_s \mathrm{ReLU}(a_s x + b_s) + d$ satisfies $\sup_{x \in [0, \delta]} \left| \frac{x}{1-x} - \phi(x) \right| \le \epsilon$. 
\end{corollary}
\begin{proof}
    Apply Lemma \ref{lemma: shallow ReLU general grid} to $f(x)=x/(1-x)$ on $[0, \delta]$. Following de Boor's equidistribution principle \citep{de1978practical}, we select the adaptive grid nodes $x_k = 1 - y_k^{-2}$ where $y_k = 1 + \frac{k}{n}((1-\delta)^{-1/2} - 1)$. Then it follows from the assumption on $n$ that $$\frac{1}{8}(x_k - x_{k-1})^2 \sup_{x\in[x_{k-1},x_k]} |f''(x)| \le \frac{4}{n^2} \left( (1-\delta)^{-1/2} - 1 \right)^2\leq \epsilon.$$
\end{proof}

\begin{corollary}[Approximation of $1/x$ on $\lbrack\delta,1\rbrack$]\label{coro: approx 1 over x}
    Let $\delta \in (0, 1)$ and target accuracy $\epsilon\in(0,1)$. If the network width $n \in \mathbb{N}$ satisfies $$n \ge 3 / \sqrt{\delta \epsilon}=\mathcal{O}(1/\sqrt{\delta\epsilon}),$$ 
    then there exist parameters $a_s, b_s, c_s \in \mathbb{R}$ for $s \in [n]$, and $d \in \mathbb{R}$, such that the neural network $\phi(x) = \sum_{s=1}^n c_s \mathrm{ReLU}(a_s x + b_s) + d$ satisfies $\sup_{x \in [\delta, 1]} |1/x - \phi(x)| \le \epsilon$.
\end{corollary}
\begin{proof}
    Apply Lemma \ref{lemma: shallow ReLU general grid} to $f(x)=1/x$ on $[\delta, 1]$. Following de Boor's equidistribution principle \citep{de1978practical}, we select the adaptive grid nodes $x_k = y_k^{-2}$ where $y_k = \delta^{-1/2} + \frac{k}{n}(1-\delta^{-1/2})$. Then by the assumption on $n$, we have  $$\frac{1}{8}(x_k - x_{k-1})^2 \sup_{x\in[x_{k-1},x_k]} |f''(x)| \leq  \frac{9}{\delta n^2}\leq \epsilon.$$
\end{proof}

\section{Construction of Transformers}\label{app: construction of transformers}

\subsection{Overview of the Construction}\label{app: overview of transformer construction}
We provide an overview of the transformer construction designed to implement the inexact preconditioned Richardson iteration. The construction includes three phases: read-in phase, iteration phase and read-out phase, where their detailed construction can be found in Appendices \ref{app: read-in phase}, \ref{app: iteration phase}, and \ref{app: read-out phase}, respectively. 

\paragraph{Read-in Phase.} This phase serves as a preparatory step for the subsequent inexact Richardson iteration. Across three blocks, it approximates the preconditioner $\mD^{-1}$ and $\mD^{-1}\vy$:
\begin{itemize}
    \item One transformer block approximates $\alpha_i \approx \mD_{ii}^{-1}$. Specifically, its single-head attention layer exactly computes $k_i \equiv 1 / (1+\sum_{j'=1}^{N}\mathcal{K}(\vx_i,\vx_{j'}))$, and its MLP layer approximates $\alpha_i \approx k_i / (1-k_i) = \mD_{ii}^{-1}$.
    \item One MLP-only transformer block zeros out $\alpha_{N+1}$.
    \item One MLP-only transformer block approximates $\beta_i \approx \eta y_i\alpha_i$.
\end{itemize}

\paragraph{Iteration Phase.} The core optimization executes the iterative updates on $\vw$:
\begin{itemize}
    \item One transformer block exactly computes $-\mD^{-1}\mK\vw$. Its single-head attention layer computes $p_i \equiv -\mD_{ii}^{-1}\sum_{j=1}^N \mathcal{K}(\vx_i,\vx_j)w_j$, while its MLP layer zeros out $p_{N+1}$.
    \item One MLP-only transformer block implements the one-step inexact iteration update
    \begin{equation*}
        \vw^{(\ell+1)} = \vw^{(\ell)} + \eta \parens{\mD^{-1}\vy - \mD^{-1}\parens{\KerMat+\lambda\mI}\vw^{(\ell)}} + \ve^{(\ell)},
    \end{equation*}
    where $\ve^{(\ell)}$ is the approximation error vector. This block also clears the $p_i$ coordinate to reset the cache for future iterations.
\end{itemize}
    
\paragraph{Read-out Phase.} The final blocks extract the prediction for the test token:
\begin{itemize}
    \item One transformer block computes the normalized prediction and its inverse scaling factor for the test token. Its single-head attention layer exactly computes $\widehat{p}_{N+1} \equiv k_{N+1}\sum_{j=1}^{N}\mathcal{K}(\vx_{N+1},\vx_{j})w_j$, and its MLP layer approximates $\widehat{k}_{N+1} \approx 1/k_{N+1}$.
    
    \item One MLP-only transformer block approximates the product $\widehat{p}_{N+1}\widehat{k}_{N+1}$ and writes the result to the target coordinate as the approximation for the Bayes optimal prediction $\sum_{j=1}^{N} w_j^* \mathcal{K}(\vx_{N+1},\vx_{j})$.
\end{itemize}

As detailed across the three phases above, each token $\vz_i$ is explicitly augmented to track these intermediate coordinates, structured as follows
\begin{equation*}
\vz_i = [\vx_i^\top, y_i, w_i, \norm{\vx_i}_2^2, k_i, \alpha_i, \beta_i, p_i \text{ (or } \widehat{p}_i\text{)}, \widehat{k}_i, s_i, t_i, 1]^\top.
\end{equation*}

\subsection{Read-in Phase}\label{app: read-in phase}
\paragraph{One transformer block approximates $\mD^{-1}$.} We first construct a single-head attention layer. Set $\attnQ,\attnK,\attnV\in\mathbb{R}^{D\times D}$ such that 
\begin{equation*}
        \attnQ\vz_i=[\vx_i/v;\norm{\vx_i}^2/v;-(1-s_i)/(2v);\vzero],\ \attnK\vz_j=[\vx_j/v;-(1-s_j)/(2v);\norm{\vx_j}^2/v;\vzero],\ \attnV \vz_j=s_j\ve_{d+4}. 
\end{equation*}
Note that for $i,j\in\mathbb{Z}_{N+2}$, 
\begin{equation*}
    \iprod{\attnQ\vz_i,\attnK\vz_j}=\begin{cases}
    -\frac{\norm{\vx_i-\vx_j}^2}{2v^2},&\text{if }i,j\in[N+1],\\
    0,& \text{if }i=0\text{ or }j=0,
    \end{cases}\qquad
    \attnV\vz_j=\begin{cases}
        \vzero, &\text{if }j\in[N+1],\\
        \ve_{d+4}, &\text{if }j=0.
    \end{cases}
\end{equation*}
We mask out the test token, i.e., setting mask matrix $\mM\in\mathbb{R}^{(N+2)\times(N+2)}$ to be 
\begin{equation*}
    \mM_{ij}=\begin{cases}
        -\infty,& \text{if }j=N+1,\\
        0, &\text{otherwise}. 
    \end{cases}
\end{equation*}
It follows that for all $i\in[N]$, 
\begin{align*}
    \widetilde{\vz}_i &= \vz_i + \sum_{j=0}^{N+1} \frac{\exp\parens{\iprod{ \attnQ \vz_i, \attnK \vz_j } + \mM_{ij}}}{\sum_{j'=0}^{N+1}\exp\parens{\iprod{ \attnQ \vz_i, \attnK \vz_{j'} } + \mM_{ij'}}}\attnV \vz_j\\
     &= \vz_i + \sum_{j=0}^{N} \frac{\exp\parens{\iprod{ \attnQ \vz_i, \attnK \vz_j }}}{\sum_{j'=0}^{N}\exp\parens{\iprod{ \attnQ \vz_i, \attnK \vz_{j'} } }}\attnV \vz_j\\
     &= \vz_i + \frac{\exp\parens{\iprod{ \attnQ \vz_i, \attnK \vz_0 }}}{\sum_{j'=0}^{N}\exp\parens{\iprod{ \attnQ \vz_i, \attnK \vz_{j'} } }}\attnV \vz_0 + \sum_{j=1}^{N} \frac{\exp\parens{\iprod{ \attnQ \vz_i, \attnK \vz_j }}}{\sum_{j'=0}^{N}\exp\parens{\iprod{ \attnQ \vz_i, \attnK \vz_{j'} } }}\attnV \vz_j\\
     &=\vz_i + \frac{1}{1+\sum_{j'=1}^N \mathcal{K}(\vx_i,\vx_{j'})}\ve_{d+4}
\end{align*}
and when $i=0$, 
\begin{align*}
    \widetilde{\vz}_0 &= \vz_0 + \sum_{j=0}^{N+1} \frac{\exp\parens{\iprod{ \attnQ \vz_0, \attnK \vz_j } + \mM_{0j}}}{\sum_{j'=0}^{N+1}\exp\parens{\iprod{ \attnQ \vz_0, \attnK \vz_{j'} } + \mM_{0j'}}}\attnV \vz_j\\
    &=\vz_0 + \frac{1}{\sum_{j'=0}^{N}\exp\parens{\iprod{ \attnQ \vz_0, \attnK \vz_{j'} } }}\ve_{d+4} = \vz_0 + \frac{1}{1+N}\ve_{d+4}
\end{align*}
and when $i=N+1$, 
\begin{align*}
    \widetilde{\vz}_{N+1} &= \vz_{N+1} + \sum_{j=0}^{N+1} \frac{\exp\parens{\iprod{ \attnQ \vz_{N+1}, \attnK \vz_j } + \mM_{N+1,j}}}{\sum_{j'=0}^{N+1}\exp\parens{\iprod{ \attnQ \vz_{N+1}, \attnK \vz_{j'} } + \mM_{N+1,j'}}}\attnV \vz_j\\
    &= \vz_{N+1} +  \frac{1}{\sum_{j'=0}^{N}\exp\parens{\iprod{ \attnQ \vz_{N+1}, \attnK \vz_{j'} }}}\ve_{d+4} = \vz_{N+1} +  \frac{1}{1+ \sum_{j'=1}^{N}\mathcal{K}(\vx_{N+1},\vx_{j'})}\ve_{d+4}
\end{align*}
Let 
\begin{equation}
k_i:=\begin{cases}\displaystyle
    \frac{1}{1+N}, & \text{if }i=0,\\
    \displaystyle\frac{1}{1+\sum_{j'=1}^N\mathcal{K}(\vx_i,\vx_{j'})},& \text{if }i\in[N+1].
\end{cases}\label{def: feature ki}
\end{equation}
We remark that the value of $k_0$ is not important in the later analysis, and hence we replace this value by $*$. Then, after this attention layer, we have output\footnotemark
\footnotetext{Here, the entries wrapped in a box indicate the values that have been updated compared to the input.}
\begin{equation*}
\begin{bmatrix}
\vx_0 & \vx_1 & \vx_2 & \dots & \vx_N & \vx_{N+1} \\
0 & y_1 & y_2 & \dots & y_N & 0\\
0 & w_1 & w_2 & \dots & w_N & 0\\
0 & \norm{\vx_1}^2 & \norm{\vx_2}^2 & \dots & \norm{\vx_N}^2  & \norm{\vx_{N+1}}^2 \\
\highlight{*} & \highlight{k_1} & \highlight{k_2} & \dots & \highlight{k_N} & \highlight{k_{N+1}}\\
\vzero_4 & \vzero_4 & \vzero_4 & \dots & \vzero_4 & \vzero_4\\
s_0 & s_1 & s_2 & \dots & s_N & s_{N+1}\\
t_0 & t_1 & t_2 & \dots & t_N & t_{N+1}\\
1 & 1 & 1 & \dots & 1 & 1
\end{bmatrix}.
\end{equation*}
We next construct the MLP layer in this transformer block. Let $\epsflip\in(0,1)$. According to Corollary \ref{coro: approx x over 1-x adaptive} with $\delta=1/(1+N\kappa_{\min})$, there exist $a_s,b_s,c_s\in\mathbb{R}$ for $s\in[\nflip]$ with $\nflip :=\lceil\frac{2 \left( (1-\frac{1}{1+N\kappa_{\min}})^{-1/2} - 1 \right)}{\sqrt{\epsflip/N}}\rceil$ such that 
\begin{equation}\label{eq: approximate flip function}
\sup_{x\in\left[0,\frac{1}{1+N\kappa_{\min}}\right]}\abs{\frac{x}{1-x}-\phi_{\mathrm{flip}}(x)}<\frac{\epsflip}{N}
\end{equation}
where $\phi_{\mathrm{flip}}(x):=\sum_{s=1}^{\nflip}c_s\ReLU\parens{a_s x+b_s}$. 
Set $\mlpin\in\mathbb{R}^{\nflip\times D}$ and $\mlpout\in\mathbb{R}^{D\times\nflip}$ such that 
\begin{equation*}
    \mlpin\vz_i=\begin{bmatrix}
        a_1k_i + b_1\\
        a_2k_i + b_2\\
        \vdots\\
        a_{\nflip}k_i+b_{\nflip}
    \end{bmatrix},\quad 
    [\mlpout]_{j,:}=\begin{cases}
[c_1,c_2,\dots,c_{\nflip}], &\text{if }j=d+5, \\
    \vzero, &\text{otherwise}. 
\end{cases}
\end{equation*}
For all $i\in\mathbb{Z}_{N+2}$, let $\alpha_i:=\phi_{\mathrm{flip}}(k_i)$. We ignore the estimate for $i=0$ as this value is not essential. Therefore, after this MLP layer, we have output 
\begin{align*}
\begin{bmatrix}
\vx_0 & \vx_1 & \vx_2 & \dots & \vx_N & \vx_{N+1} \\
0 & y_1 & y_2 & \dots & y_N & 0\\
0 & w_1 & w_2 & \dots & w_N & 0\\
0 & \norm{\vx_1}^2 & \norm{\vx_2}^2 & \dots & \norm{\vx_N}^2 & \norm{\vx_{N+1}}^2 \\
* & k_1 & k_2 & \dots & k_N & k_{N+1}\\
\highlight{*} & \highlight{\alpha_1} & \highlight{\alpha_2} & \dots & \highlight{\alpha_N} & \highlight{\alpha_{N+1}}\\
\vzero_3 & \vzero_3 & \vzero_3 & \dots & \vzero_3 & \vzero_3\\
s_0 & s_1 & s_2 & \dots & s_N & s_{N+1}\\
t_0 & t_1 & t_2 & \dots & t_N & t_{N+1}\\
1 & 1 & 1 & \dots & 1 & 1
\end{bmatrix}.
\end{align*}
As a remark at the end of this block, we let 
\begin{equation}\label{def: ri}
r_i:=\alpha_i-\frac{k_i}{1-k_i},
\qquad i\in[N+1].
\end{equation}
For $i\in[N]$, this agrees with
$r_i=\alpha_i-\mD_{ii}^{-1}$. Since $0<k_i<\frac{1}{1+N\kappa_{\min}}$, 
we obtain from  \eqref{eq: approximate flip function} that
\begin{equation}\label{eq: di approximates 1/sum K}
\abs{r_i}=\abs{\frac{k_i}{1-k_i} - \alpha_i} < \frac{\epsflip}{N},\quad i\in[N+1].
\end{equation}

\paragraph{One MLP-only transformer block zeros out $\alpha_{N+1}$.}
It follows from \eqref{eq: di approximates 1/sum K} and the triangle inequality that for all $i\in[N+1]$, 
\begin{equation}
    \abs{\alpha_i} < \frac{1}{\sum_{j'=1}^N\mathcal{K}(\vx_i,\vx_{j'})} + \frac{\epsflip}{N} < \frac{1}{N\kappa_{\min}}+\frac{1}{N}\equiv B_\alpha\label{def: B alpha}
\end{equation}
Then we set $\mlpin\in\mathbb{R}^{2\times D}$ and $\mlpout\in\mathbb{R}^{D\times 2}$ such that 
\begin{equation*}
    \mlpin\vz_i=\begin{bmatrix}
        -\alpha_i-B_\alpha(1-t_i)\\
        \alpha_i-B_\alpha(1-t_i)
    \end{bmatrix}, \quad [\mlpout]_{j,:}=\begin{cases}
[1,-1],&\text{if }j=d+5,\\
\mathbf{0},&\text{otherwise}.
    \end{cases}
\end{equation*}
Recalling the definition of $t_i$, when $i\in[N]$, we have $t_i=0$, and hence 
\begin{equation*}
\vz_i+\mlpout\ReLU(\mlpin\vz_i) = \vz_i + \ReLU(-\alpha_i-B_\alpha)\ve_{d+5} - \ReLU(\alpha_i-B_\alpha)\ve_{d+5} = \vz_i;
\end{equation*}
moreover, when $i=N+1$, we have $t_{N+1}=1$, which leads to 
\begin{equation*}
\vz_{N+1}+\mlpout\ReLU(\mlpin\vz_{N+1}) = \vz_{N+1} + \ReLU(-\alpha_{N+1})\ve_{d+5} - \ReLU(\alpha_{N+1})\ve_{d+5} = \vz_{N+1}-\alpha_{N+1}\ve_{d+5}.
\end{equation*}
So, after this MLP layer, we have output
\begin{align*}
\begin{bmatrix}
\vx_0 & \vx_1 & \vx_2 & \dots & \vx_N & \vx_{N+1} \\
0 & y_1 & y_2 & \dots & y_N & 0\\
0 & w_1 & w_2 & \dots & w_N & 0\\
0 & \norm{\vx_1}^2 & \norm{\vx_2}^2 & \dots & \norm{\vx_N}^2 & \norm{\vx_{N+1}}^2 \\
* & k_1 & k_2 & \dots & k_N & k_{N+1}\\
* & \alpha_1 & \alpha_2 & \dots & \alpha_N & \highlight{0}\\
\vzero_3 & \vzero_3 & \vzero_3 & \dots & \vzero_3 & \vzero_3\\
s_0 & s_1 & s_2 & \dots & s_N & s_{N+1}\\
t_0 & t_1 & t_2 & \dots & t_N & t_{N+1}\\
1 & 1 & 1 & \dots & 1 & 1
\end{bmatrix}.
\end{align*}
\paragraph{One MLP-only transformer block approximates $\eta y_i\alpha_i$.} Let $\epssq\in(0,1)$. According to Corollary \ref{coro: approx x2} with $\delta=B_y+B_\alpha$, there exist $a_s,b_s,c_s\in\mathbb{R}$ for $s\in[\nsq]$ with $\nsq := \lceil\frac{B_y+B_\alpha}{\sqrt{\epssq/N}}\rceil$, and $d_{\mathrm{sq}}\in\mathbb{R}$, such that 
\begin{equation}\label{eq: approximate square function y alpha}
    \sup_{|x|\leq B_y+B_\alpha}\abs{x^2-\phisq(x)}\leq \frac{\epssq}{N},
\end{equation}
where $\phisq(x):=d_{\mathrm{sq}}+\sum_{s=1}^{\nsq}c_s\ReLU\parens{a_s x+b_s}$. We set $\mlpin\in\mathbb{R}^{2\nsq\times D}$ and $\mlpout\in\mathbb{R}^{D\times 2\nsq}$ such that 
\begin{equation*}
    \mlpin\vz_i=\begin{bmatrix}
        a_1\parens{y_i + \alpha_i} + b_1\\
        \vdots\\
        a_{\nsq}\parens{y_i + \alpha_i} + b_{\nsq}\\
        a_1\parens{y_i - \alpha_i} + b_1\\
        \vdots\\
        a_{\nsq}\parens{y_i - \alpha_i} + b_{\nsq}
    \end{bmatrix}, \quad [\mlpout]_{j,:}=\begin{cases}
\frac{\eta}{4}[c_1,\ldots,c_{\nsq},-c_1,\ldots,-c_{\nsq}],&\text{if }j=d+6,\\
\mathbf{0},&\text{otherwise}.
    \end{cases}
\end{equation*}
It turns out that for $i\in\mathbb{Z}_{N+2}$,
\begin{align*}
    \beta_i\equiv [\mlpout\ReLU(\mlpin\mZ)]_{d+6,i}&=\frac{\eta}{4}\sum_{s=1}^{\nsq}c_s\ReLU\parens{a_s(y_i+\alpha_i) + b_s} -  \frac{\eta}{4}\sum_{s=1}^{\nsq}c_s\ReLU\parens{a_s(y_i-\alpha_i) + b_s}\\
    &=\frac{\eta}{4}\parens{\phisq(y_i+\alpha_i) - \phisq(y_i-\alpha_i)}.
\end{align*}
Since $y_{N+1}=0$ and $\alpha_{N+1}=0$, we have 
$\beta_{N+1} = \frac{\eta}{4}\parens{\phisq(0) - \phisq(0)} = 0$. So, after this MLP layer, we have output
\begin{align*}
\begin{bmatrix}
\vx_0 & \vx_1 & \vx_2 & \dots & \vx_N & \vx_{N+1} \\
0 & y_1 & y_2 & \dots & y_N & 0\\
0 & w_1 & w_2 & \dots & w_N & 0\\
0 & \norm{\vx_1}^2 & \norm{\vx_2}^2 & \dots & \norm{\vx_N}^2 & \norm{\vx_{N+1}}^2 \\
* & k_1 & k_2 & \dots & k_N & k_{N+1}\\
* & \alpha_1 & \alpha_2 & \dots & \alpha_N & 0\\
\highlight{*} & \highlight{\beta_1} & \highlight{\beta_2} & \dots & \highlight{\beta_N} & \highlight{0}\\
\vzero_2 & \vzero_2 & \vzero_2 & \dots & \vzero_2 & \vzero_2\\
s_0 & s_1 & s_2 & \dots & s_N & s_{N+1}\\
t_0 & t_1 & t_2 & \dots & t_N & t_{N+1}\\
1 & 1 & 1 & \dots & 1 & 1
\end{bmatrix}.
\end{align*}
Moreover, when $i\in[N]$, let 
\begin{equation*}
    \tau_{+,i}:=\phisq\parens{y_i + \alpha_i} - \parens{y_i+\alpha_i}^2,\quad \tau_{-,i}:=\phisq\parens{y_i - \alpha_i} - \parens{y_i-\alpha_i}^2. 
\end{equation*}
Then we can rewrite $\beta_i$ as 
\begin{equation}\label{def: beta i}
\beta_i=\frac{\eta}{4}\parens{(y_i+\alpha_i)^2 + \tau_{+,i} - (y_i-\alpha_i)^2 - \tau_{-,i}}=\eta y_i\alpha_i + \frac{\eta}{4}\parens{\tau_{+,i}-\tau_{-,i}}. 
\end{equation}
Noting that $\abs{y_i\pm\alpha_i}\leq \abs{y_i}+\abs{\alpha_i}\leq B_y+B_{\alpha}$, and by \eqref{eq: approximate square function y alpha}, we have 
\begin{equation}\label{eq: condition for tau+-}
\abs{\tau_{+,i}}\leq\frac{\epssq}{N},\quad \abs{\tau_{-,i}}\leq\frac{\epssq}{N}.
\end{equation}

\subsection{Iteration Phase}\label{app: iteration phase}
\paragraph{One transformer block exactly computes $-\mD^{-1}\mK\vw$ and zeros out the corresponding feature for the test token.} 
We first construct a single-head attention layer. Set $\sigma$ to be Softmax, and $\nhead=1$. Set $\attnQ$, $\attnK$ and $\attnV$ such that 
\begin{equation*}
        \attnQ\vz_i=\frac{1}{v}[\vx_i;\norm{\vx_i}^2;-1/2;\vzero],\quad\attnK\vz_j=\frac{1}{v}[\vx_j;-1/2;\norm{\vx_j}^2;\vzero],\quad \attnV\vz_j=- w_j\ve_{d+7}
\end{equation*}
Note that $
\iprod{\attnQ\vz_i,\attnK\vz_j}=-\frac{\norm{\vx_i-\vx_j}^2}{2v^2}$ for $i,j\in\mathbb{Z}_{N+2}$. We mask out the dummy token and the test token, i.e., setting mask matrix $\mM\in\mathbb{R}^{(N+2)\times (N+2)}$ to be 
\begin{equation*}
    \mM_{ij}=\begin{cases}
        -\infty,& \text{if }j=0 \text{ or }j=N+1, \\
        0, &\text{otherwise}. 
    \end{cases}
\end{equation*}
It follows that for all $i\in\mathbb{Z}_{N+2}$, 
\begin{align*}
&\sum_{j=0}^{N+1} \frac{\exp\parens{\iprod{ \attnQ \vz_i, \attnK \vz_j } + \mM_{ij}}}{\sum_{j'=0}^{N+1} \exp\parens{\iprod{ \attnQ \vz_i, \attnK \vz_{j'} } + \mM_{ij'}}}\attnV \vz_j=-\sum_{j=1}^{N}\frac{\exp\parens{-\frac{\norm{\vx_i-\vx_{j}}^2}{2v^2}}w_j}{\sum_{j'=1}^N \exp\parens{-\frac{\norm{\vx_i-\vx_{j'}}^2}{2v^2}}}\ve_{d+7} = p_i \ve_{d+7}
\end{align*}
where we denote 
\begin{equation}\label{def: pi}
    p_i:=-\sum_{j=1}^N\frac{\mathcal{K}(\vx_i,\vx_j)w_j}{\sum_{j'=1}^N \mathcal{K}(\vx_i,\vx_{j'})}.
\end{equation}
Again, the value of $p_0$ is not important, hence we will hide it as $*$. So, after this attention layer, we have output
\begin{align*}
\begin{bmatrix}
\vx_0 & \vx_1 & \vx_2 & \dots & \vx_N & \vx_{N+1} \\
0 & y_1 & y_2 & \dots & y_N & 0\\
0 & w_1 & w_2 & \dots & w_N & 0\\
0 & \norm{\vx_1}^2 & \norm{\vx_2}^2 & \dots & \norm{\vx_N}^2 & \norm{\vx_{N+1}}^2 \\
* & k_1 & k_2 & \dots & k_N & k_{N+1}\\
* & \alpha_1 & \alpha_2 & \dots & \alpha_N & 0\\
* & \beta_1 & \beta_2 & \dots & \beta_N & 0\\
\highlight{*} & \highlight{p_1} & \highlight{p_2} & \dots & \highlight{p_N} & \highlight{p_{N+1}}\\
0 & 0 & 0 & \dots & 0 & 0\\
s_0 & s_1 & s_2 & \dots & s_N & s_{N+1}\\
t_0 & t_1 & t_2 & \dots & t_N & t_{N+1}\\
1 & 1 & 1 & \dots & 1 & 1
\end{bmatrix}.
\end{align*}
We next construct an MLP layer to explicitly zero out $p_{N+1}$. Even if $p_{N+1}$ is currently zero due to the initialization of the weights $w_i = 0$, it may become non-zero in later iterations. As we will see in the construction of the next block, the network executes the inexact preconditioned Richardson iteration \eqref{eq: TF inexact Richardson iteration}, which updates the weights $w_i$ for $i \in [N]$. Therefore, we would still need this layer to enforce $p_{N+1} = 0$ across all iterations. 

We fix $B_w$ as defined in \eqref{eq: w ell is elementwise bounded} and
verify inductively that the constructed iterates satisfy
$\normMax{\vw^{(\ell)}}\leq B_w$. The base case holds since
$\vw^{(0)}=\vzero$. We assume by induction that
$\normMax{\vw^{(\ell)}}\leq B_w$, i.e., $|w_j|$ for $j\in[N]$ in the current step is bounded above by $B_w$. This guarantees that $|p_{i}|$ is also bounded by $B_w$, namely 
\begin{equation*}
    |p_i|=\abs{-\sum_{j=1}^N\frac{\mathcal{K}(\vx_i,\vx_{j})w_j}{\sum_{j'=1}^N \mathcal{K}(\vx_i,\vx_{j'})}}\leq \max_{j\in[N]}\abs{w_j}\leq B_w, \quad\forall i\in[N+1]. 
\end{equation*}
Then we set $\mlpin\in\mathbb{R}^{2\times D}$ and $\mlpout\in\mathbb{R}^{D\times 2}$ such that 
\begin{equation*}
    \mlpin\vz_i=\begin{bmatrix}
        -p_i-B_w(1-t_i)\\
        p_i-B_w(1-t_i)
    \end{bmatrix}, \quad [\mlpout]_{j,:}=\begin{cases}
        [1,-1],&\text{if }j=d+7,\\
        \mathbf{0},&\text{otherwise}.
    \end{cases}
\end{equation*}
Recalling the definition of $t_i$, when $i\in[N]$, $t_i=0$, and hence 
\begin{equation*}
\vz_i+\mlpout\ReLU(\mlpin\vz_i) = \vz_i + \ReLU(-p_i-B_w)\ve_{d+7} - \ReLU(p_i-B_w)\ve_{d+7} = \vz_i,
\end{equation*}
moreover, when $i=N+1$, $t_{N+1}=1$ gives 
\begin{equation*}
\vz_{N+1}+\mlpout\ReLU(\mlpin\vz_{N+1}) = \vz_{N+1} + \ReLU(-p_{N+1})\ve_{d+7} - \ReLU(p_{N+1})\ve_{d+7} = \vz_{N+1}-p_{N+1}\ve_{d+7}.
\end{equation*}
So, after this MLP layer, we have output
\begin{align*}
\begin{bmatrix}
\vx_0 & \vx_1 & \vx_2 & \dots & \vx_N & \vx_{N+1} \\
0 & y_1 & y_2 & \dots & y_N & 0\\
0 & w_1 & w_2 & \dots & w_N & 0\\
0 & \norm{\vx_1}^2 & \norm{\vx_2}^2 & \dots & \norm{\vx_N}^2 & \norm{\vx_{N+1}}^2 \\
* & k_1 & k_2 & \dots & k_N & k_{N+1}\\
* & \alpha_1 & \alpha_2 & \dots & \alpha_N & 0\\
* & \beta_1 & \beta_2 & \dots & \beta_N & 0\\
* & p_1 & p_2 & \dots & p_N & \highlight{0}\\
0 & 0 & 0 & \dots & 0 & 0\\
s_0 & s_1 & s_2 & \dots & s_N & s_{N+1}\\
t_0 & t_1 & t_2 & \dots & t_N & t_{N+1}\\
1 & 1 & 1 & \dots & 1 & 1
\end{bmatrix}.
\end{align*}

\paragraph{An MLP-only transformer block implements one-step inexact precondition Richardson update.} Let $\tildeepssq\in(0,1)$. According to Corollary \ref{coro: approx x2} with $\delta=B_w+B_\alpha$, there exist $\widetilde{a}_s, \widetilde{b}_s, \widetilde{c}_s\in\mathbb{R}$ for $s\in[\tildensq]$ with $\tildensq:=\lceil\frac{B_w+B_\alpha}{\sqrt{\tildeepssq/N^2}}\rceil$,
and $\widetilde{d}_{\mathrm{sq}}\in\mathbb{R}$, such that 
\begin{equation}\label{eq: approximate square function}
    \sup_{|x|\leq B_w + B_\alpha}\abs{x^2-\tildephisq(x)}\leq \frac{\tildeepssq}{N^2},
\end{equation}
where $\tildephisq(x):=\widetilde{d}_{\mathrm{sq}}+\sum_{s=1}^{\tildensq}\widetilde{c}_s\ReLU\parens{\widetilde{a}_s x+\widetilde{b}_s}$. Set $\mlpin\in\mathbb{R}^{(2\tildensq+4)\times D}$ such that 
\begin{equation*}
    \mlpin\vz_i=\begin{bmatrix}
    \widetilde{a}_1\parens{w_i +  \alpha_i} + \widetilde{b}_1\\
        \vdots\\
        \widetilde{a}_{\tildensq}\parens{w_i + \alpha_i} + \widetilde{b}_{\tildensq}\\
         \widetilde{a}_1\parens{w_i - \alpha_i} + \widetilde{b}_1\\
        \vdots\\
        \widetilde{a}_{\tildensq}\parens{w_i - \alpha_i} + \widetilde{b}_{\tildensq}\\
        \beta_i\\
        -\beta_i\\
        p_i\\
        -p_i
    \end{bmatrix}
\end{equation*}
and set $\mlpout\in\mathbb{R}^{D\times (2\tildensq+4)}$ such that 
\begin{align*}
    [\mlpout]_{j,:}=\begin{cases}
        \left[
        -\frac{\eta\lambda \widetilde{c}_1}{4},\dots,-\frac{\eta\lambda \widetilde{c}_{\tildensq}}{4},\frac{\eta\lambda \widetilde{c}_1}{4},\dots,\frac{\eta\lambda \widetilde{c}_{\tildensq}}{4},1,-1,\eta,-\eta\right], &\text{if }j=d+2,\\
    [0,0,\dots,0,-1,1], &\text{if }j=d+7, \\
    \vzero, &\text{otherwise}.
    \end{cases}
\end{align*}
It turns out that for $i\in\mathbb{Z}_{N+2}$,
\begin{align}
    &[\mlpout\ReLU(\mlpin\mZ)]_{d+2,i}\nonumber\\
=\ & -\frac{\eta\lambda}{4}\parens{\sum_{s=1}^{\tildensq}\widetilde{c}_s\ReLU\parens{\widetilde{a}_s \parens{w_i + \alpha_i}+\widetilde{b}_s} - \sum_{s=1}^{\tildensq}\widetilde{c}_s\ReLU\parens{\widetilde{a}_s \parens{w_i - \alpha_i}+\widetilde{b}_s}}\nonumber\\
&+\parens{\ReLU\parens{\beta_i}-\ReLU\parens{-\beta_i}}+\eta\parens{\ReLU\parens{p_i}-\ReLU\parens{-p_i}}\nonumber\\
=\ &  - \frac{\eta\lambda}{4}\parens{\tildephisq\parens{w_i + \alpha_i} - \tildephisq\parens{w_i - \alpha_i}} + \beta_i + \eta p_i\label{in the proof: rewrite one-step product error process}
\end{align}
and 
\begin{align*}
[\mlpout\ReLU(\mlpin\mZ)]_{d+7,i}=-\ReLU(p_i)+\ReLU(-p_i)=-p_i. 
\end{align*}
We first look at the update at the $(d+2)$-th row. 
\begin{itemize}
    \item When $i=N+1$, note that $\beta_{N+1}=0$, $w_{N+1}=0$, $\alpha_{N+1}=0$ and $p_{N+1}=0$. Hence \eqref{in the proof: rewrite one-step product error process} with $i=N+1$ gives $$\Delta w_{N+1}\equiv [\mlpout\ReLU(\mlpin\mZ)]_{d+2,N+1} = - \frac{\eta\lambda}{4}\parens{\tildephisq\parens{0} - \tildephisq\parens{0}} = 0.$$

\item Now we consider $i\in[N]$. Let 
\begin{align*}
\widetilde{\tau}_{+,i}:=\tildephisq\parens{w_i + \alpha_i} - \parens{w_i+\alpha_i}^2,\quad \widetilde{\tau}_{-,i}:=\tildephisq\parens{w_i - \alpha_i} - \parens{w_i-\alpha_i}^2.
\end{align*}
Note that $\abs{w_i\pm\alpha_i}\leq \abs{w_i}+\abs{\alpha_i}\leq B_w + B_\alpha$, which with \eqref{eq: approximate square function} implies 
\begin{equation}\label{eq: condition for tilde tau+-}
\abs{\widetilde{\tau}_{+,i}}\leq \frac{\tildeepssq}{N^2},\quad \abs{\widetilde{\tau}_{-,i}}\leq \frac{\tildeepssq}{N^2}.
\end{equation} 
Then with the above notations and using forms \eqref{def: ri}, \eqref{def: beta i}, \eqref{def: pi} for $r_i$, $\beta_i$ and $p_i$, we proceed from \eqref{in the proof: rewrite one-step product error process} to obtain that for $i\in[N]$, 
\begin{align*}
\Delta w_{i}&\equiv[\mlpout\ReLU(\mlpin\mZ)]_{d+2,i}\\
&= - \frac{\eta\lambda}{4}\parens{\parens{w_i + \alpha_i}^2+\widetilde{\tau}_{+,i} - \parens{w_i - \alpha_i}^2 - \widetilde{\tau}_{-,i}} + \beta_i + \eta p_i\\
&=\eta(y_i-\lambda w_i)\alpha_i + \eta p_i + \frac{\eta}{4}\parens{\tau_{+,i}-\tau_{-,i}} - \frac{\eta\lambda}{4}\parens{\widetilde{\tau}_{+,i}-\widetilde{\tau}_{-,i}}\\
&=\eta(y_i-\lambda w_i)\mD_{ii}^{-1} -\eta\mD_{ii}^{-1}\sum_{j=1}^N\mathcal{K}(\vx_i,\vx_{j})w_j + \frac{\eta}{4}\parens{\tau_{+,i}-\tau_{-,i}} - \frac{\eta\lambda}{4}\parens{\widetilde{\tau}_{+,i}-\widetilde{\tau}_{-,i}} + \eta(y_i-\lambda w_i)r_i
\end{align*}
\end{itemize}

Therefore, with the residual connection, we have updated at the $(d+2)$-th row 
\begin{equation*}
    w_i\leftarrow w_i + \Delta w_i. 
\end{equation*}
Moreover, the update at the $(d+7)$-th row is 
\begin{equation*}
    p_i\leftarrow p_i-p_i=0.
\end{equation*}
After this MLP layer, we have output 
\begin{align*}
\begin{bmatrix}
\vx_0 & \vx_1 & \vx_2 & \dots & \vx_N & \vx_{N+1} \\
0 & y_1 & y_2 & \dots & y_N & 0\\
\highlight{*} & \highlight{w_1+\Delta w_1} & \highlight{w_2+\Delta w_2} & \dots & \highlight{w_N+\Delta w_N} & \highlight{0}\\
0 & \norm{\vx_1}^2 & \norm{\vx_2}^2 & \dots & \norm{\vx_N}^2 & \norm{\vx_{N+1}}^2 \\
* & k_1 & k_2 & \dots & k_N & k_{N+1}\\
* & \alpha_1 & \alpha_2 & \dots & \alpha_N & 0\\
* & \beta_1 & \beta_2 & \dots & \beta_N & 0\\
\highlight{0} & \highlight{0} & \highlight{0} & \dots & \highlight{0} & \highlight{0}\\
0 & 0 & 0 & \dots & 0 & 0\\
s_0 & s_1 & s_2 & \dots & s_N & s_{N+1}\\
t_0 & t_1 & t_2 & \dots & t_N & t_{N+1}\\
1 & 1 & 1 & \dots & 1 & 1
\end{bmatrix}
\end{align*}
We notice that the constructed transformer implements the following update 
\begin{equation*}
    w_i\leftarrow w_i + \eta\parens{(y_i-\lambda w_i)\mD_{ii}^{-1}-\mD_{ii}^{-1}\sum_{j=1}^N\mathcal{K}(\vx_i,\vx_{j})w_j} + \eta\parens{(y_i-\lambda w_i)r_i + \frac{\tau_{+,i}-\tau_{-,i}-\lambda\widetilde{\tau}_{+,i}+\lambda\widetilde{\tau}_{-,i}}{4}}. 
\end{equation*}
In vector form, it is the inexact preconditioned Richardson iteration
\eqref{eq: TF inexact Richardson iteration} with
\begin{align*}
&\scalingerror:=[r_i:i\in[N]],\quad
\bm{\tau}_{+}:=[\tau_{+,i}:i\in[N]],\quad
\bm{\tau}_{-}:=[\tau_{-,i}:i\in[N]],\\
&\widetilde{\bm{\tau}}_+^{(\ell)}
:=[\widetilde{\tau}_{+,i}^{(\ell)}:i\in[N]],\quad
\widetilde{\bm{\tau}}_-^{(\ell)}
:=[\widetilde{\tau}_{-,i}^{(\ell)}:i\in[N]].
\end{align*}
Here, the error vectors $\scalingerror$, $\bm{\tau}_+$, and
$\bm{\tau}_-$ do not depend on $\vw^{(\ell)}$ and therefore remain unchanged
throughout the iterations. On the other hand,
$\widetilde{\bm{\tau}}_+^{(\ell)}$ and
$\widetilde{\bm{\tau}}_-^{(\ell)}$ depend on $\vw^{(\ell)}$, which is why
they carry the index $\ell$.

We verify the required bounds by induction, starting from
$\vw^{(0)}=\vzero$. The bounds on $\scalingerror$, $\bm{\tau}_+$, and
$\bm{\tau}_-$ follow directly from
\eqref{eq: di approximates 1/sum K} and
\eqref{eq: condition for tau+-}. Moreover, the induction hypothesis
$\normMax{\vw^{(\ell)}}\leq B_w$ implies
\eqref{eq: condition for tilde tau+-} at iteration $\ell$. Thus, all the
conditions in \eqref{assumption: error magnitude /N /N2} hold through
iteration $\ell$. The recursive estimate in the proof of
Theorem~\ref{theorem: entry-wise boundedness of w ell} then gives
$\normMax{\vw^{(\ell+1)}}\leq B_w$. This completes the induction.

We remark that the two blocks constructed in this phase implement a single step of the inexact preconditioned Richardson iteration \eqref{eq: TF inexact Richardson iteration}. By sequentially repeating this construction, the transformer executes multiple iterations of the algorithm.

\subsection{Read-out Phase}\label{app: read-out phase}
\paragraph{One transformer block computes the test token's normalized prediction and inverse scaling factor.} We first construct a single-head attention layer. Set $\attnQ$, $\attnK$ and $\attnV$ such that 
\begin{equation*}
        \attnQ\vz_i=\frac{1}{v}[\vx_i;\norm{\vx_i}^2;-1/2;\vzero],\quad\attnK\vz_j=\frac{1}{v}[\vx_j;-1/2;\norm{\vx_j}^2;\vzero],\quad \attnV\vz_j=w_j\ve_{d+7}.
\end{equation*}
Note that $\iprod{\attnQ\vz_i,\attnK\vz_j}=-\frac{\norm{\vx_i-\vx_j}^2}{2v^2}$. We mask out the dummy token, i.e., setting mask matrix $\mM\in\mathbb{R}^{(N+2)\times (N+2)}$ to be 
\begin{equation*}
    \mM_{ij}=\begin{cases}
        -\infty,& \text{if }j=0, \\
        0, &\text{otherwise}. 
    \end{cases}
\end{equation*}
It follows that for all $i\in[N+1]$, 
\begin{align*}
&\sum_{j=0}^{N+1} \frac{\exp\parens{\iprod{ \attnQ \vz_i, \attnK \vz_j } + \mM_{ij}}}{\sum_{j'=0}^{N+1} \exp\parens{\iprod{ \attnQ \vz_i, \attnK \vz_{j'} } + \mM_{ij'}}}\attnV \vz_j=\sum_{j=1}^{N+1}\frac{\exp\parens{-\frac{\norm{\vx_i-\vx_{j}}^2}{2v^2}}w_j}{\sum_{j'=1}^{N+1} \exp\parens{-\frac{\norm{\vx_i-\vx_{j'}}^2}{2v^2}}}\ve_{d+7}\\
=\ & \sum_{j=1}^{N+1}\frac{\mathcal{K}(\vx_i,\vx_{j})w_j}{\sum_{j'=1}^{N+1} \mathcal{K}(\vx_i,\vx_{j'})}\ve_{d+7}= \underbrace{\sum_{j=1}^{N}\frac{\mathcal{K}(\vx_i,\vx_{j})w_j}{\sum_{j'=1}^{N+1} \mathcal{K}(\vx_i,\vx_{j'})}}_{\text{denoted by }\widehat{p}_i}\ve_{d+7}
\end{align*}
where the last equality uses $w_{N+1}=0$. In particular, when $i=N+1$, the value is 
\begin{align*}
\widehat{p}_{N+1}=\sum_{j=1}^{N}\frac{\mathcal{K}(\vx_{N+1},\vx_{j})w_j}{\sum_{j'=1}^{N+1} \mathcal{K}(\vx_{N+1},\vx_{j'})}&=\frac{1}{1+\sum_{j'=1}^{N} \mathcal{K}(\vx_{N+1},\vx_{j'})}\sum_{j=1}^{N}\mathcal{K}(\vx_{N+1},\vx_{j})w_j\\
&=k_{N+1}\sum_{j=1}^{N}\mathcal{K}(\vx_{N+1},\vx_{j})w_j
\end{align*}
where the last equation follows from definition \eqref{def: feature ki} of $k_{N+1}$. Note that the $(d+7)$-th row of the input is a zero vector, and we only care about the feature $\widehat{p}_{N+1}$ for the test token (the features for the training tokens are hidden as $*$). Therefore, after this attention layer, the output is
\begin{align*}
\begin{bmatrix}
\vx_0 & \vx_1 & \vx_2 & \dots & \vx_N & \vx_{N+1} \\
0 & y_1 & y_2 & \dots & y_N & 0\\
* & w_1 & w_2 & \dots & w_N & 0\\
0 & \norm{\vx_1}^2 & \norm{\vx_2}^2 & \dots & \norm{\vx_N}^2 & \norm{\vx_{N+1}}^2 \\
* & k_1 & k_2 & \dots & k_N & k_{N+1}\\
* & \alpha_1 & \alpha_2 & \dots & \alpha_N & 0\\
* & \beta_1 & \beta_2 & \dots & \beta_N & 0\\
\highlight{*} & \highlight{*} & \highlight{*} & \dots & \highlight{*} & \highlight{\widehat{p}_{N+1}}\\
0 & 0 & 0 & \dots & 0 & 0\\
s_0 & s_1 & s_2 & \dots & s_N & s_{N+1}\\
t_0 & t_1 & t_2 & \dots & t_N & t_{N+1}\\
1 & 1 & 1 & \dots & 1 & 1
\end{bmatrix}.
\end{align*}

The MLP layer of this block approximates $1/k_{N+1}$. Let $\epsinv\in(0,1)$. According to Corollary \ref{coro: approx 1 over x} with $\delta=1/(N+1)$, there exist $a_s,b_s,c_s\in\mathbb{R}$ for $s\in[\ninv]$ with $\ninv:=\lceil3\sqrt{(N+1)/\epsinv}\rceil$ and $d_{\mathrm{inv}}\in\mathbb{R}$, such that
\begin{equation}\label{eq: approximate division}
\sup_{x\in [\frac{1}{N+1},1]}\abs{\frac{1}{x}-\phiinv(x)}<\epsinv
\end{equation}
where $\phiinv(x):=d_{\mathrm{inv}}+\sum_{s=1}^{\ninv}c_s\ReLU\parens{a_s x + b_s}.$ Set $\mlpin\in\mathbb{R}^{(\ninv+1)\times D}$ and $\mlpout\in\mathbb{R}^{D\times (\ninv+1)}$ such that 
\begin{equation*}
\mlpin\vz_i=\begin{bmatrix}
        a_1 k_i+ b_1\\
        a_2 k_i + b_2\\
        \vdots\\
        a_{\ninv} k_i + b_{\ninv}\\
        1
    \end{bmatrix},
\quad 
    [\mlpout]_{j,:}=\begin{cases}
    [c_1,c_2,\dots,c_{\ninv},d_{\mathrm{inv}}], &\text{if }j=d+8,\\
    \vzero, &\text{otherwise}.
    \end{cases}
\end{equation*}
Let $\widehat{k}_i:=\phiinv(k_i)$. It turns out that for $i\in\mathbb{Z}_{N+2}$, $[\mlpout\ReLU(\mlpin\mZ)]_{d+8,i} = \widehat{k}_i$. In particular, we only need the case $i=N+1$, and the output of this MLP layer is 
\begin{equation*}
    \begin{bmatrix}
\vx_0 & \vx_1 & \vx_2 & \dots & \vx_N & \vx_{N+1} \\
0 & y_1 & y_2 & \dots & y_N & 0\\
* & w_1 & w_2 & \dots & w_N & 0\\
0 & \norm{\vx_1}^2 & \norm{\vx_2}^2 & \dots & \norm{\vx_N}^2 & \norm{\vx_{N+1}}^2 \\
* & k_1 & k_2 & \dots & k_N & k_{N+1}\\
* & \alpha_1 & \alpha_2 & \dots & \alpha_N & 0\\
* & \beta_1 & \beta_2 & \dots & \beta_N & 0\\
* & * & * & \dots & * & \widehat{p}_{N+1}\\
\highlight{*} & \highlight{*} & \highlight{*} & \cdots & \highlight{*} & \highlight{\widehat{k}_{N+1}}\\
s_0 & s_1 & s_2 & \dots & s_N & s_{N+1}\\
t_0 & t_1 & t_2 & \dots & t_N & t_{N+1}\\
1 & 1 & 1 & \dots & 1 & 1
\end{bmatrix}.
\end{equation*}
Moreover, we remark that
\begin{equation*}
    \frac{1}{1+N}\leq k_i=\frac{1}{1+\sum_{j'=1}^{N}\mathcal{K}(\vx_i,\vx_j)} \leq 1,\quad\forall i\in[N+1]. 
\end{equation*}
The above estimate with \eqref{eq: approximate division} implies 
\begin{equation}\label{eq: estimate of tilde ki}
\abs{\frac{1}{k_i}-\widehat{k}_i}<\epsinv. 
\end{equation}

\paragraph{One MLP-only transformer block approximates the product $\widehat{k}_{N+1}\widehat{p}_{N+1}$.}
Let $\hatepssq\in(0,1)$. According to Corollary \ref{coro: approx x2} with $\delta=B_w+3$, there exist $\widehat{a}_s,\widehat{b}_s,\widehat{c}_s\in\mathbb{R}$ for $s\in[\hatnsq]$ with $\hatnsq:=\lceil
\frac{B_w+3}{\sqrt{\hatepssq/N}}
\rceil$ and $\widehat d_{\mathrm{sq}}\in\mathbb{R}$, such that 
\begin{equation}\label{eq: approximate square last}
\sup_{\abs{x}\leq B_w+3}\abs{x^2-\hatphisq(x)}\leq \frac{\hatepssq}{N},
\end{equation}
where $\hatphisq(x):=\widehat d_{\mathrm{sq}}+\sum_{s=1}^{\hatnsq}\widehat{c}_s\ReLU\parens{\widehat{a}_s x + \widehat{b}_s}.$ Set $\mlpin\in\mathbb{R}^{2\hatnsq\times D}$ and $\mlpout\in\mathbb{R}^{D\times 2\hatnsq}$ such that 
\begin{equation*}
\mlpin\vz_i=\begin{bmatrix}
        \widehat{a}_1 \parens{\widehat{k}_i/N+\widehat{p}_i} + \widehat{b}_1\\
        \widehat{a}_2 \parens{\widehat{k}_i/N+\widehat{p}_i} + \widehat{b}_2\\
        \vdots\\
        \widehat{a}_{\hatnsq} \parens{\widehat{k}_i/N+\widehat{p}_i} + \widehat{b}_{\hatnsq}\\
        \widehat{a}_1 \parens{\widehat{k}_i/N - \widehat{p}_i} + \widehat{b}_1\\
        \widehat{a}_2 \parens{\widehat{k}_i/N - \widehat{p}_i} + \widehat{b}_2\\
        \vdots\\
        \widehat{a}_{\hatnsq} \parens{\widehat{k}_i/N - \widehat{p}_i} + \widehat{b}_{\hatnsq}
    \end{bmatrix},
\end{equation*}
and
\begin{align*}
    [\mlpout]_{j,:}=\begin{cases}
        \frac{N}{4}[\widehat{c}_1,\widehat{c}_2,\dots,\widehat{c}_{\hatnsq},-\widehat{c}_1,-\widehat{c}_2,\dots,-\widehat{c}_{\hatnsq}], &\text{if }j=d+1,\\
    \vzero, &\text{otherwise}.
    \end{cases}
\end{align*}
Therefore, for $i\in\mathbb{Z}_{N+2}$, 
\begin{align}\label{def: oi}
o_i:=&[\mlpout\ReLU(\mlpin\mZ)]_{d+1,i} 
=\frac{N}{4}\parens{\hatphisq(\widehat{k}_i/N + \widehat{p}_i) - \hatphisq(\widehat{k}_i/N - \widehat{p}_i)}.
\end{align}
Thus, after this layer, we have output 
\begin{equation*}
    \begin{bmatrix}
\vx_0 & \vx_1 & \vx_2 & \dots & \vx_N & \vx_{N+1} \\
\highlight{*} & \highlight{y_1+*} & \highlight{y_2+*} & \dots & \highlight{y_N+*} & \highlight{o_{N+1}}\\
* & w_1 & w_2 & \dots & w_N & 0\\
0 & \norm{\vx_1}^2 & \norm{\vx_2}^2 & \dots & \norm{\vx_N}^2 & \norm{\vx_{N+1}}^2 \\
* & k_1 & k_2 & \dots & k_N & k_{N+1}\\
* & \alpha_1 & \alpha_2 & \dots & \alpha_N & 0\\
* & \beta_1 & \beta_2 & \dots & \beta_N & 0\\
* & * & * & \dots & * & \widehat{p}_{N+1}\\
* & * & * & \cdots & * & \widehat{k}_{N+1}\\
s_0 & s_1 & s_2 & \dots & s_N & s_{N+1}\\
t_0 & t_1 & t_2 & \dots & t_N & t_{N+1}\\
1 & 1 & 1 & \dots & 1 & 1
\end{bmatrix}.
\end{equation*}

Moreover, when $i\in[N+1]$, let 
\begin{equation*}
\widehat{\tau}_{+,i} := \hatphisq(\widehat{k}_i/N + \widehat{p}_i) - (\widehat{k}_i/N + \widehat{p}_i)^2,\quad \widehat{\tau}_{-,i} := \hatphisq(\widehat{k}_i/N - \widehat{p}_i) - (\widehat{k}_i/N - \widehat{p}_i)^2. 
\end{equation*}
It follows that $o_i$ can be rewritten as 
\begin{align}\label{in the proof: rewrite oi}
    o_i=\frac{N}{4}\parens{(\widehat{k}_i/N + \widehat{p}_i)^2 - (\widehat{k}_i/N - \widehat{p}_i)^2 + \widehat{\tau}_{+,i} - \widehat{\tau}_{-,i}}=\widehat{k}_i\widehat{p}_i + \frac{N}{4}\parens{\widehat{\tau}_{+,i}-\widehat{\tau}_{-,i}}.
\end{align}
Recall that $|w_j|$, $j\in[N]$ are bounded by $B_w$, hence $\widehat{p}_{i}$ is also bounded by $B_w$. That is, 
\begin{equation*}
|\widehat{p}_i|=\abs{\sum_{j=1}^{N}\frac{\mathcal{K}(\vx_i,\vx_{j})w_j}{\sum_{j'=1}^{N+1} \mathcal{K}(\vx_i,\vx_{j'})}}\leq \max_{j\in[N]}\abs{w_j}\leq B_w, \quad\forall i\in[N+1]. 
\end{equation*}
Furthermore, it follows from \eqref{eq: estimate of tilde ki} that
\begin{equation*}
    \abs{\frac{\widehat{k}_i}{N}}<\frac{1}{N}\parens{\frac{1}{k_i}+\epsinv}=\frac{1}{N}\parens{1+\sum_{j'=1}^{N}\mathcal{K}(\vx_i,\vx_j)+\epsinv}<\frac{2+N}{N}\leq 3. 
\end{equation*}
Therefore, $\abs{\frac{\widehat{k}_i}{N}\pm\widehat{p}_i}\leq \abs{\frac{\widehat{k}_i}{N}}+\abs{\widehat{p}_i}< 3+B_w$. Combining this with \eqref{eq: approximate square last} implies \begin{equation}\label{in the proof: estimate of hat tau}
\abs{\widehat{\tau}_{\pm,i}}<\hatepssq/N.
\end{equation}

\subsection{Analysis}\label{app: construction transformers analysis}

\begin{proposition}\label{proposition: readout TF wi*K}
Suppose that assumptions of Theorem \ref{theorem: entry-wise boundedness of w ell} hold and $\vw^*=[w_i^*:i\in[N]]$ is the solution of linear system \eqref{eq: unnormalized kernel system}. Let $\TF$ denote the complete transformer network constructed by composing the read-in blocks from Appendix~\ref{app: read-in phase}, $\ell$ repetitions of the iteration blocks from Appendix~\ref{app: iteration phase}, and the read-out blocks from Appendix~\ref{app: read-out phase}. Then 
\begin{align*}
&\abs{\mathrm{readout}(\TF(\mZ))-\sum_{i=1}^Nw_i^*\mathcal{K}(\vx_i,\vx_{N+1})}\\
&\leq \parens{1-\eta\lambda_0\parens{1-c}}^{\ell}\frac{\parens{\frac{1}{\sqrt{\lambda_0}}+1}B_y}{\lambda_0\sqrt{\kappa_{\min}}} + \frac{\frac{B_y}{\sqrt{\lambda_0}}\epsflip+\parens{1+\lambda_0}\parens{\epssq+\tildeepssq}}{\lambda_0\parens{1-c}\sqrt{\kappa_{\min}}} + \frac{\hatepssq}{2} + \epsinv \widetilde{B}_w,
\end{align*}
where 
\begin{equation*}
\widetilde{B}_w:= \frac{\parens{\frac{1}{\sqrt{\lambda_0}}+1}B_y}{\lambda_0\sqrt{\kappa_{\min}}} + \frac{\frac{B_y}{\sqrt{\lambda_0}}+2\parens{1+\lambda_0}}{\lambda_0\parens{1-c}\sqrt{\kappa_{\min}}} + \frac{\parens{\frac{1}{\sqrt{\lambda_0}}+1}B_y}{\lambda_0}.
\end{equation*}
\end{proposition}
\begin{proof}
It follows from \eqref{in the proof: rewrite oi} and \eqref{in the proof: estimate of hat tau} with $i=N+1$ that
\begin{equation}\label{in the proof: estimate of product hat k hat p}
\abs{o_{N+1} - \widehat{k}_{N+1}\widehat{p}_{N+1}}=\frac{N}{4}\abs{\widehat{\tau}_{+,i}-\widehat{\tau}_{-,i}}\leq \frac{\hatepssq}{2}.
\end{equation}
Note that by definition of $\widehat{p}_{N+1}$ and relation \eqref{eq: estimate of tilde ki}, we have 
\begin{align}
&\abs{\widehat{k}_{N+1}\widehat{p}_{N+1} - \sum_{j=1}^{N}\mathcal{K}(\vx_{N+1},\vx_{j})w_j\myl} = \abs{\widehat{k}_{N+1}k_{N+1}-1}\abs{\sum_{j=1}^{N}\mathcal{K}(\vx_{N+1},\vx_{j})w_j\myl}\nonumber\\
&< \epsinv k_{N+1} \abs{\sum_{j=1}^{N}\mathcal{K}(\vx_{N+1},\vx_{j})w_j\myl}\leq \epsinv  \abs{\sum_{j=1}^{N}\frac{\mathcal{K}(\vx_{N+1},\vx_{j})}{1+\sum_{j=1}^N \mathcal{K}(\vx_{N+1},\vx_j)}}\normMax{\vw\myl}   \leq \epsinv B_w\leq \epsinv \widetilde{B}_w, \label{in the proof: estimate hat k hat p and sum kernel}
\end{align}
where in the last inequality, we relax the bound $B_w$ to a strictly larger value $\widetilde{B}_w$ that is independent of $N$. Hence, by triangle inequality and estimates \eqref{in the proof: estimate of product hat k hat p}, \eqref{in the proof: estimate hat k hat p and sum kernel}, we obtain
\begin{equation*}
\abs{o_{N+1} - \sum_{j=1}^{N}\mathcal{K}(\vx_{N+1},\vx_{j})w_j\myl} \leq \frac{\hatepssq}{2} + \epsinv \widetilde{B}_w. 
\end{equation*}
Combining this with Proposition \ref{proposition: estimate of wil - wi* K(xi,xN+1)} and triangle inequality gives the estimate of $\abs{o_{N+1}-\sum_{i=1}^Nw_i^*\mathcal{K}(\vx_i,\vx_{N+1})}$. The desired result immediately follows by noting that 
\begin{equation*}
o_{N+1}=\mathrm{readout}(\TF). 
\end{equation*}
\end{proof}

\begin{theorem}[Formal version of Theorem \ref{theorem: informal version of KRR ICL}]\label{theorem: formal version of KRR ICL}
Fix a constant \(c\in(0,1)\) and a step size \(\eta\), satisfying 
\begin{equation}\label{eq: final eta condition}
0<\eta<
\frac{1}{c+1+\exp(2B_x^2/v^2)}.
\end{equation}
Let \(\lambda_0\in(0,1]\) and \(\epsilon\in(0,c)\). There exists a single-head transformer network $\TF$ consisting of $(2L+5)$ blocks with MLP width $W$ such that the final prediction satisfies
\begin{equation*}
\abs{\mathrm{readout}(\TF(\mZ)) - \sum_{i=1}^N w_i^*\mathcal{K}(\vx_i,\vx_{N+1})} \leq C_{\mathrm{sys}}\epsilon,
\end{equation*}
where the constant $C_{\mathrm{sys}}$ is defined by 
\begin{equation}\label{eq: explicit Csys}
C_{\mathrm{sys}}:=\frac{\parens{\frac{1}{\sqrt{\lambda_0}}+1}B_y}{\lambda_0}\parens{\frac{2}{\sqrt{\kappa_{\min}}}+1} + 2\frac{\frac{B_y}{\sqrt{\lambda_0}}+2\parens{1+\lambda_0}}{\lambda_0\parens{1-c}\sqrt{\kappa_{\min}}} + \frac{1}{2}=\mathcal{O}\parens{\frac{1}{\lambda_0^{3/2}\sqrt{\kappa_{\min}}}}, 
\end{equation}
the iteration count $L$ is \begin{equation}\label{eq: explicit depth}
L:=\lceil\frac{\log(1/\epsilon)}{\log(1/(1-\eta\lambda_0(1-c)))}\rceil=\mathcal{O}\parens{\frac{\log(1/\epsilon)}{\lambda_0}}, 
\end{equation}
the maximum MLP width is 
\begin{equation}\label{eq: explicit width}
W:=\max\left\{\nflip,2\nsq,2\tildensq+4,\ninv+1,2\hatnsq,2\right\}=\mathcal{O}\parens{
\frac{\sqrt{N/\epsilon}}{\lambda_0^{3/2}\kappa_{\min}}},
\end{equation}
in which
\begin{align*}
    \nflip&:=\lceil\frac{2 \left( (1-\frac{1}{1+N\kappa_{\min}})^{-1/2} - 1 \right)}{\sqrt{\epsilon/N}}\rceil,\ \nsq:=\lceil\frac{B_y+B_\alpha}{\sqrt{\epsilon/N}}\rceil,\\ \tildensq&:=\lceil\frac{B_w+B_\alpha}{\sqrt{\epsilon/N^2}}\rceil,\ \ninv:=\lceil3\sqrt{(N+1)/\epsilon}\rceil,\ \hatnsq:=\lceil\frac{B_w+3}{\sqrt{\epsilon/N}}\rceil,
\end{align*}
with $B_w$ and $B_\alpha$ defined in \eqref{eq: w ell is elementwise bounded} and \eqref{def: B alpha}, respectively. 

\end{theorem}

\begin{proof}
Let $\TF$ denote the complete transformer network constructed by composing the read-in blocks from Appendix~\ref{app: read-in phase}, $\ell$ repetitions of the iteration blocks from Appendix~\ref{app: iteration phase}, and the read-out blocks from Appendix~\ref{app: read-out phase}. We set $\epssq$, $\tildeepssq$, $\hatepssq$, $\epsinv$ and $\epsflip$ in Proposition  \ref{proposition: readout TF wi*K} to be $\epsilon$. By definition $\kappa_{\min}=\exp(-2B_x^2/v^2)$ and
$\lambda_0\in(0,1]$, the step-size condition \eqref{eq: final eta condition} results to
\[
0<\eta<\frac{1}{c+1+\exp(2B_x^2/v^2)}=
\frac{1}{c+1+1/\kappa_{\min}}
\leq
\frac{1}{\lambda_0c+1+\lambda_0/\kappa_{\min}}.
\]
Together with $\epsilon<c$, this implies the step-size
condition~\eqref{eq: eta condition}. Moreover, the induction in Appendix~\ref{app: iteration phase} verifies the
assumptions of Theorem~\ref{theorem: entry-wise boundedness of w ell}.
Thus, Proposition~\ref{proposition: readout TF wi*K} applies, which gives 
    \begin{align*}
    &\abs{\mathrm{readout}(\TF(\mZ))-\sum_{i=1}^Nw_i^*\mathcal{K}(\vx_i,\vx_{N+1})}\\
    &\leq \parens{1-\eta\lambda_0\parens{1-c}}^{\ell}\frac{\parens{\frac{1}{\sqrt{\lambda_0}}+1}B_y}{\lambda_0\sqrt{\kappa_{\min}}} + \parens{\frac{\frac{B_y}{\sqrt{\lambda_0}}+2\parens{1+\lambda_0}}{\lambda_0\parens{1-c}\sqrt{\kappa_{\min}}} + \frac{1}{2} +  \widetilde{B}_w}\epsilon.
    \end{align*}
    Moreover, we set $\ell=L$ for $L$ defined in \eqref{eq: explicit depth}, and it follows that $\parens{1-\eta\lambda_0\parens{1-c}}^{\ell}\leq \epsilon$. Therefore, 
    \begin{align*}
\abs{\mathrm{readout}(\TF(\mZ))-\sum_{i=1}^Nw_i^*\mathcal{K}(\vx_i,\vx_{N+1})}\leq  \parens{\frac{\parens{\frac{1}{\sqrt{\lambda_0}}+1}B_y}{\lambda_0\sqrt{\kappa_{\min}}} + \frac{\frac{B_y}{\sqrt{\lambda_0}}+2\parens{1+\lambda_0}}{\lambda_0\parens{1-c}\sqrt{\kappa_{\min}}} + \frac{1}{2} +  \widetilde{B}_w}\epsilon.
    \end{align*}
The overall MLP width $W$ is determined by taking the maximum of the widths required for the zero-out layers and each of the individual approximations \eqref{eq: approximate flip function}, \eqref{eq: approximate square function y alpha}, \eqref{eq: approximate square function}, \eqref{eq: approximate division}, and \eqref{eq: approximate square last}.

\end{proof}

\section{Classical solvers}
\label{app:solvers}

In this section, we review the classical KRR solvers with their standard convergence rates. For each method, we use its standard theoretically tuned parameters and initialize $\vw^{(0)} = \vzero$. By $\kappa(\mA)$ we denote the condition number of a matrix $\mA$. 

\textbf{Plain Richardson.} Plain Richardson numerically solves the kernel system \eqref{eq: unnormalized kernel system} with the update rule $ \vw^{(\ell+1)} = \vw^{(\ell)} + \eta\,\bigl( \vy - (\KerMat+\lambda \mI)\vw^{(\ell)} \bigr)$. Set $\eta=2/(\mathrm{eig}_{\max}(\mS)+\mathrm{eig}_{\min}(\mS))$, where $\mS:=\KerMat+\lambda\mI$, and it requires $\mathcal{O}(\kappa(\mS)\log(1/\varepsilon))$ steps to achieve $\varepsilon$ error \citep{richardson_book}.

\textbf{Preconditioned Richardson.} Preconditioned Richardson also solves the kernel system \eqref{eq: unnormalized kernel system} with the update rule given in \eqref{eq: precond_richardson}. Set $\eta = 2/(\mathrm{eig}_{\max}(\mS)+ \mathrm{eig}_{\min}(\mS))$, where $\mS:=\mD^{-1/2}(\KerMat + \lambda \mI)\mD^{-1/2}$, and it requires $\mathcal{O}(\kappa(\mS)\log(1/\varepsilon))$ steps to achieve $\varepsilon$ error \citep{richardson_book}.

\textbf{Conjugate Gradient.} Conjugate Gradient is applied to the kernel system \eqref{eq: unnormalized kernel system}. Let $\mS:=\KerMat+\lambda\mI$. With $\vr^{(0)}=\vy$ and $\vp^{(0)}=\vr^{(0)}$, the update rule is $\vw^{(k+1)}=\vw^{(k)} + \eta_k \vp^{(k)},
$ where $\eta_{k}= \frac{\|\vr^{(k)}\|^{2}}{(\vp^{(k)})^{\top}\mS\vp^{(k)}}$,  $\vp^{(k+1)}=\vr^{(k+1)}+\frac{\|\vr^{(k+1)}\|^{2}}{\|\vr^{(k)}\|^{2}}\,\vp^{(k)}$, and $\vr^{(k+1)} = \vr^{(k)} - \eta_k\mS\vp^{(k)}$. It requires $\mathcal{O}(\sqrt{\kappa(\mS)}\,\log(1/\varepsilon))$ steps to achieve $\varepsilon$ error \citep{cg_rate}.

\textbf{Gradient Descent.} Standard Gradient Descent is applied to the RKHS loss $L(\vw)= \tfrac{1}{2}\|\KerMat\vw - \vy\|^{2}+\tfrac{1}{2}\lambda\,\vw^{\top}\KerMat\vw,$ with update rule $\vw^{(k+1)}=\vw^{(k)}-\eta\,\KerMat((\KerMat + \lambda \mI)\vw^{(k)} - \vy)$. Set $\eta = 2/(\mathrm{eig}_{\max}(\mS)+ \mathrm{eig}_{\min}(\mS))$, where $\mS:=\KerMat(\KerMat+\lambda\mI)$. It requires $\mathcal{O}(\kappa(\mS)\log(1/\varepsilon))$ steps to achieve $\varepsilon$ error \citep{bach2024book}.

\textbf{Nesterov Gradient Descent.} Nesterov Gradient Descent is applied to the same RKHS loss $L(\vw)$. Let $\mS:=\KerMat(\KerMat + \lambda \mI)$. It has update rule $\vw^{(k+1)}=\vz^{(k+1)} + \beta\,(\vz^{(k+1)} - \vz^{(k)})$, where $\vz^{(k+1)}=\vw^{(k)} - \eta\,\KerMat((\KerMat + \lambda \mI)\vw^{(k)} - \vy)$. Initializing $\vz^{(0)} = \vzero$, we set $\eta = 1/\mathrm{eig}_{\max}(\mS)$ and $\beta = (\sqrt{\kappa(\mS)}-1)/(\sqrt{\kappa(\mS)}+1)$. It requires approximately $\mathcal{O}(\sqrt{\kappa(\mS)}\log(1/\varepsilon))$ steps to achieve $\varepsilon$ error \citep{bach2024book}.
          
\section{Training and implementation details}
\label{app:training_details}

This appendix gives the architectural, optimization, curriculum, probing, and evaluation details used for all experiments in Section~\ref{sec:experiments}.

\textbf{Transformer architecture.} The detailed structure of the GPT-2-style transformer is listed in Table~\ref{tab:architecture_details}.

\textbf{Optimization.} All models are trained using hyperparameters listed in Table \ref{tab:optimization_details}.

\textbf{Context-length curriculum.} Following the curriculum strategy used in \citep{bai2023transformers,fu2024transformers}, we train with a 
gradually increasing context length. This stabilizes optimization at the maximum context length 
\(N=40\). At training step \(s\), the active context length is
\[
n(s)=\min\{n_{\mathrm{start}}+\lfloor s/\Delta s\rfloor\Delta n,\; N\},
\]
with \(n_{\mathrm{start}}=11\), \(\Delta n=2\), and \(\Delta s=2000\). The curriculum reaches the 
maximum context length \(N=40\) after approximately \(3\times 10^4\) steps, after which all 
remaining training is performed at full context length. 

\textbf{Probe training.} After training the transformer, we freeze its weights and train a separate linear probe with weight
\(\mW_{\mathrm{probe}}^{(\ell)} \in \mathbb{R}^{d_{\mathrm{model}}\times 1}\) and bias $b_{\mathrm{probe}}^{(\ell)} \in \mathbb{R}$
for each layer \(\ell \in [L]\). Each probe maps the hidden state at the final query 
\(x\)-token to a scalar prediction. Probes are trained for 500,000 steps on the same data distribution 
and context-length curriculum as the base model, using learning rate \(10^{-4}\). 

\textbf{Computational resources.} For all experiments, the transformer is trained on a single NVIDIA V100 or RTX A4000 GPU. Layer-wise probes are trained in parallel across 4 RTX A4000 GPUs.

\begin{table}[h]
\centering
\caption{Default transformer architecture used in the main experiments.}
\label{tab:architecture_details}
\begin{tabular}{ll}
\toprule
Hyperparameter & Value \\
\midrule
Layers \(L\) & 12 \\
Heads \(H\) & 8 \\
Hidden width \(d_{\mathrm{model}}\) & 256 \\
Feedforward width \(d_{\mathrm{ff}}\) & 1024 \((=4d_{\mathrm{model}})\) \\
Dropout & 0.0 \\
Layer normalization & Pre-LN \\
Activation & GELU \\
Positional encoding & Learned absolute embeddings \\
Attention mask & Causal \\
\bottomrule
\end{tabular}
\end{table}

\begin{table}[h]
\centering
\caption{Optimization hyperparameters for transformer training runs.}
\label{tab:optimization_details}
\begin{tabular}{ll}
\toprule
Hyperparameter & Value \\
\midrule
Optimizer & AdamW \\
\(\beta_1,\beta_2\) & \(0.9,0.999\) \\
Weight decay & \(10^{-4}\) \\
Initial learning rate & \(10^{-4}\) \\
Final learning rate & \(10^{-5}\) \\
Schedule & Cosine annealing \\
Training steps & \(500{,}000\) \\
Batch size & 64 \\
Gradient clipping & 1.0 \\
Loss & MSE on causal \(x\)-token predictions \\
\bottomrule
\end{tabular}
\end{table}             
\clearpage
\begingroup
\setlength{\textfloatsep}{14pt plus 2pt minus 2pt}
\setlength{\intextsep}{10pt plus 2pt minus 2pt}
\makeatletter
\setlength{\@fptop}{0pt}
\setlength{\@fpsep}{12pt plus 2pt minus 2pt}
\makeatother
\Needspace{8\baselineskip}
\section{Additional experimental results}
\label{app:experiments}

Throughout this appendix we report full results for all experiments developed in the main text.

\subsection{Plain Richardson versus Preconditioned Richardson}
We compare plain and preconditioned Richardson iteration across Gaussian,
spherical, and uniform inputs, examining step-size thresholds
(Figure~\ref{fig:lr2-all-distributions}), spectral condition numbers
(Figure~\ref{fig:condition-all-distributions}), and relative coefficient
error over iterations under each method's optimal step size
(Figure~\ref{fig:convergence-all-distributions}). We use $5$ seeds with
$256$ sequences per seed. Curves show the mean of the within-seed medians,
with shaded bands indicating standard deviation across seeds. 

\begin{figure}[!htbp]
  \centering
  \subfiguretopcaptrue

  \subfigure[Gaussian]{%
    \label{fig:lr2-gaussian}%
    \includegraphics[width=0.32\linewidth]
    {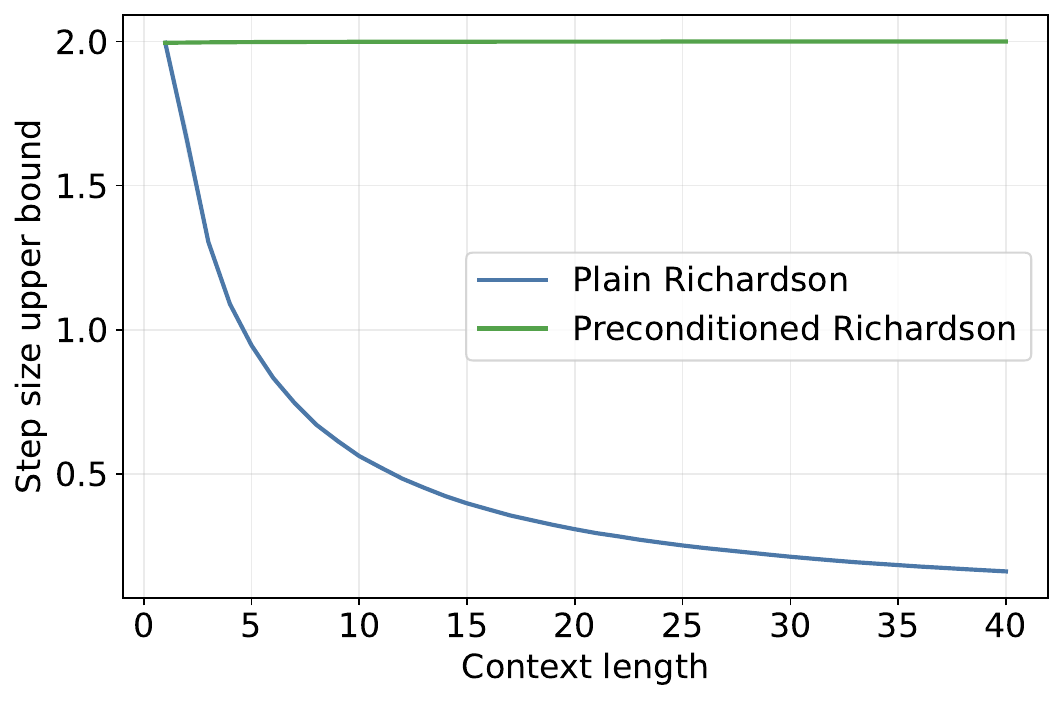}%
  }\hfill%
  \subfigure[Spherical]{%
    \label{fig:lr2-spherical}%
    \includegraphics[width=0.32\linewidth]
    {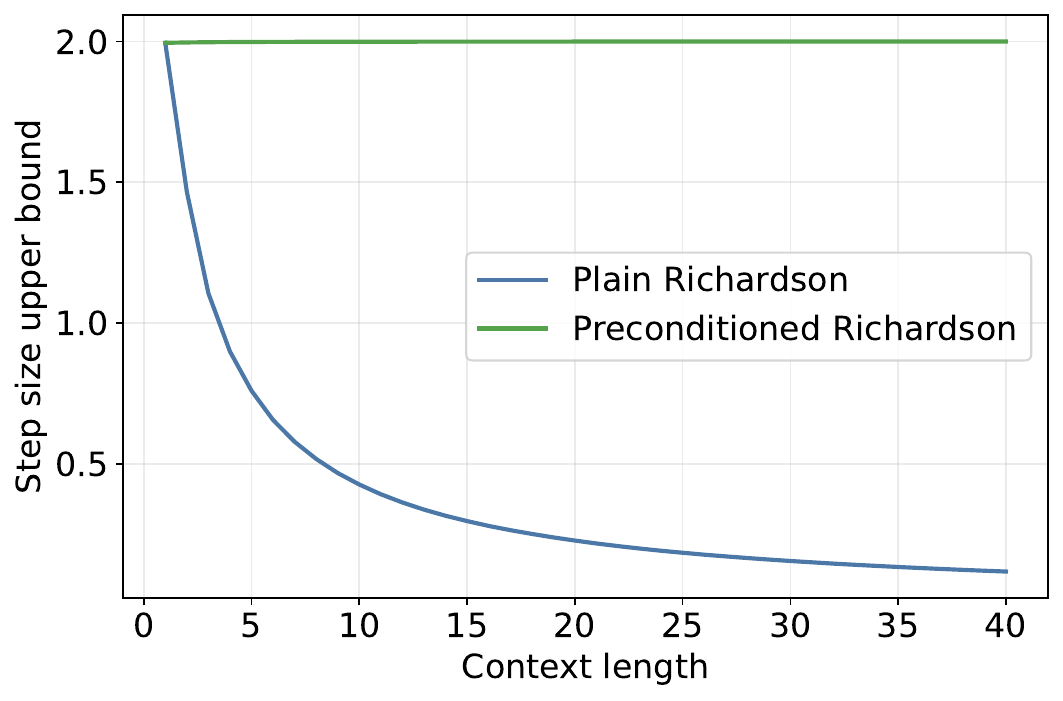}%
  }\hfill%
  \subfigure[Uniform]{%
    \label{fig:lr2-uniform}%
    \includegraphics[width=0.32\linewidth]
    {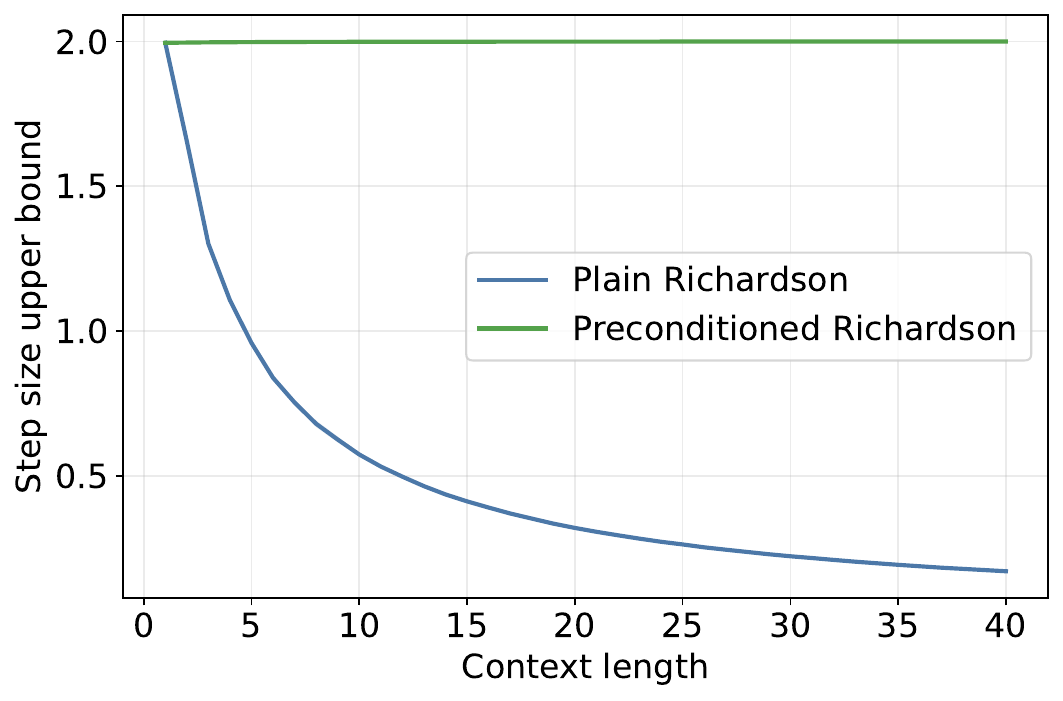}%
  }

  \caption{The step-size threshold $2/\eig_{\max}(\mS)$ versus context length.
The threshold remains approximately constant for preconditioned Richardson but decreases with context length for plain Richardson.}
  \label{fig:lr2-all-distributions}
\end{figure}

\begin{figure}[!htbp]
  \centering
  \subfiguretopcaptrue

  \subfigure[Gaussian]{%
    \label{fig:condition-gaussian}%
    \includegraphics[width=0.32\linewidth]
    {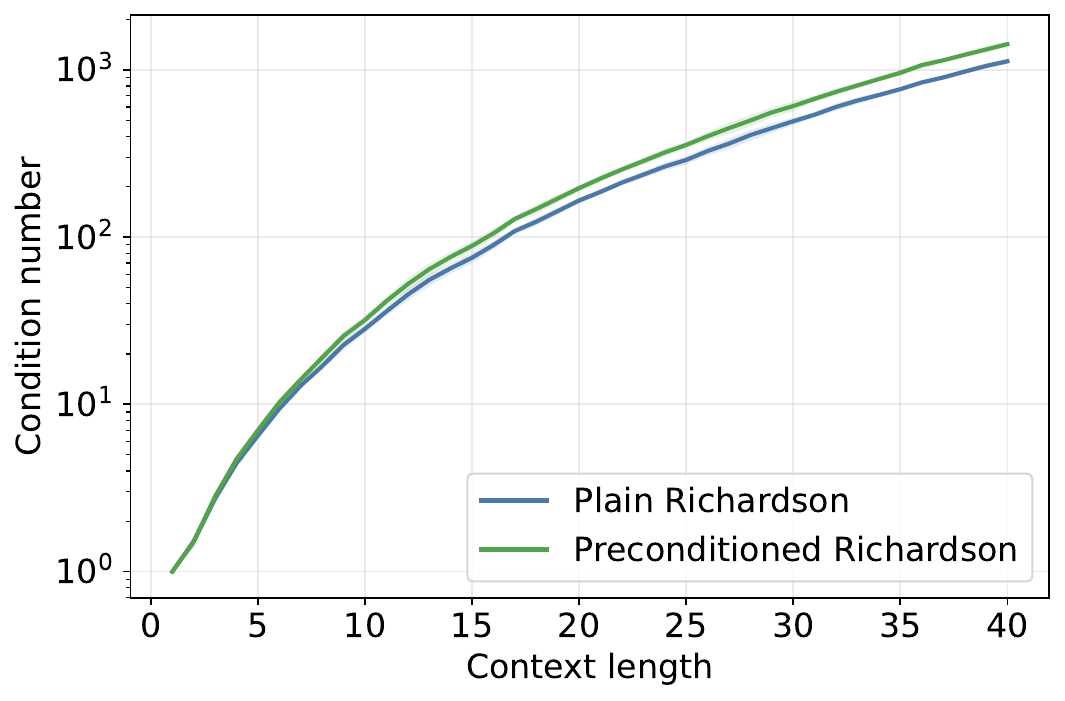}%
  }\hfill%
  \subfigure[Spherical]{%
    \label{fig:condition-spherical}%
    \includegraphics[width=0.32\linewidth]
    {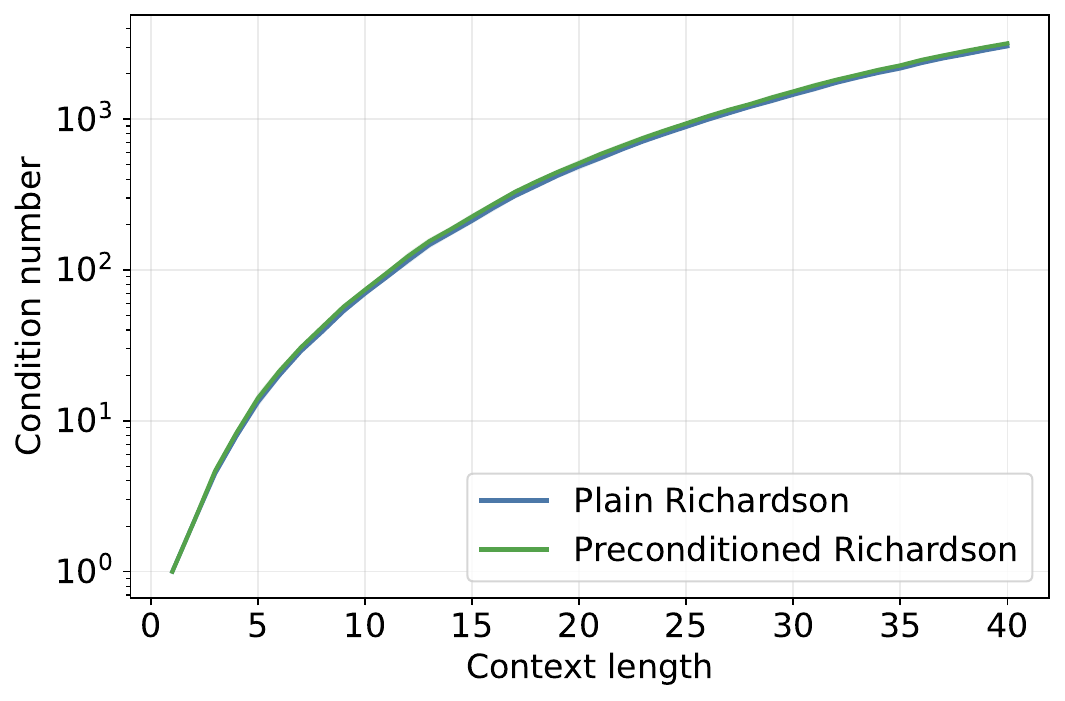}%
  }\hfill%
  \subfigure[Uniform]{%
    \label{fig:condition-uniform}%
    \includegraphics[width=0.32\linewidth]
    {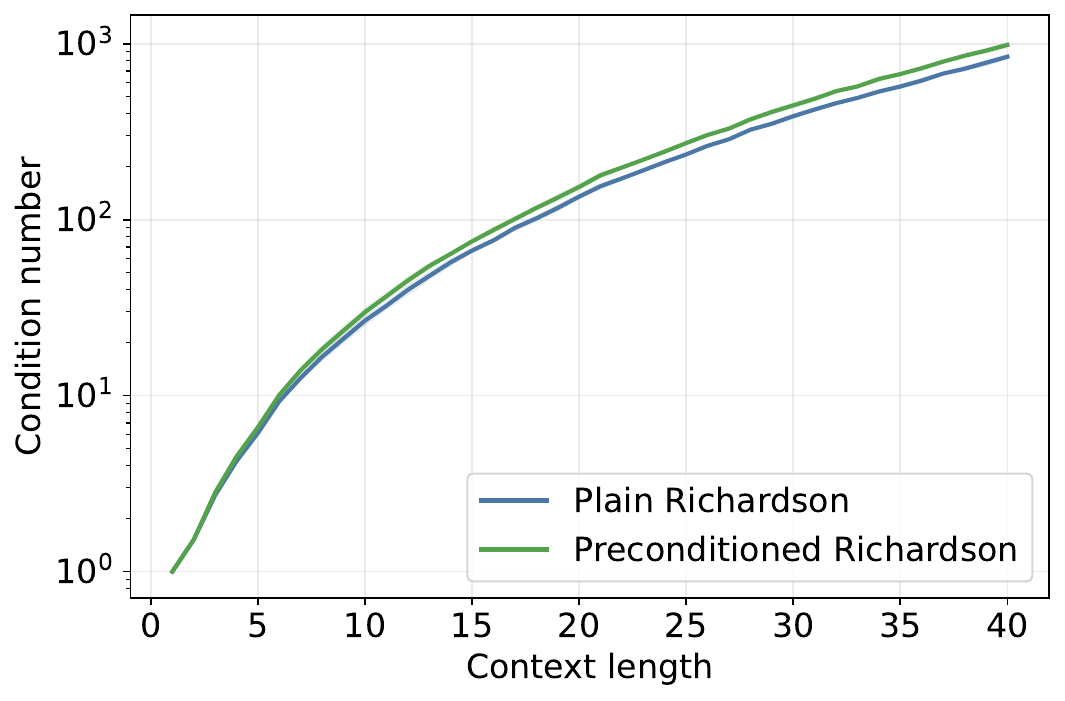}%
  }

  \caption{Spectral condition number $\kappa(\mS)$ versus context length. Plain and preconditioned Richardson have comparable condition numbers.}
  \label{fig:condition-all-distributions}
\end{figure}

\begin{figure}[!htbp]
  \centering
  \subfiguretopcaptrue

  \subfigure[Gaussian]{%
    \label{fig:convergence-gaussian}%
    \includegraphics[width=0.32\linewidth]
    {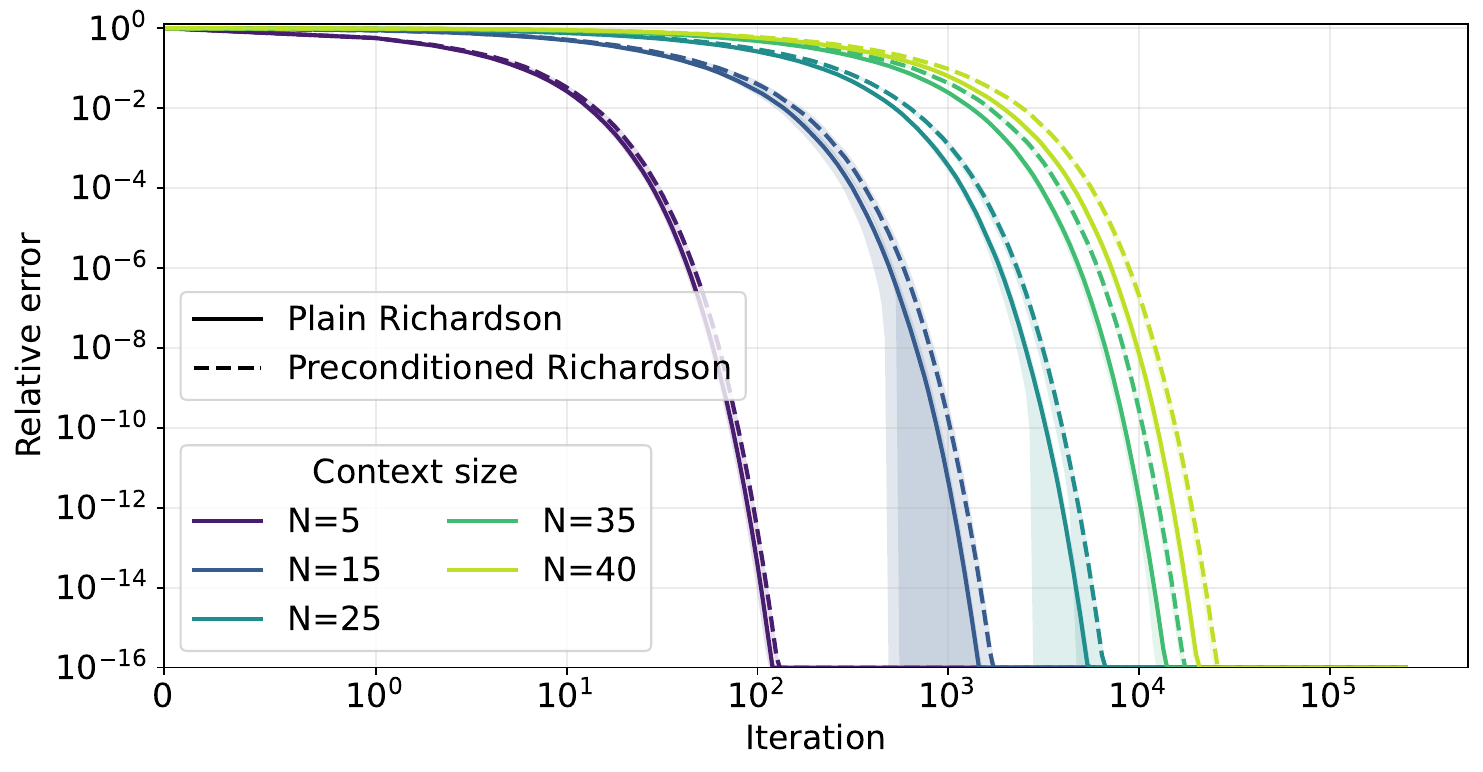}%
  }\hfill%
  \subfigure[Spherical]{%
    \label{fig:convergence-spherical}%
    \includegraphics[width=0.32\linewidth]
    {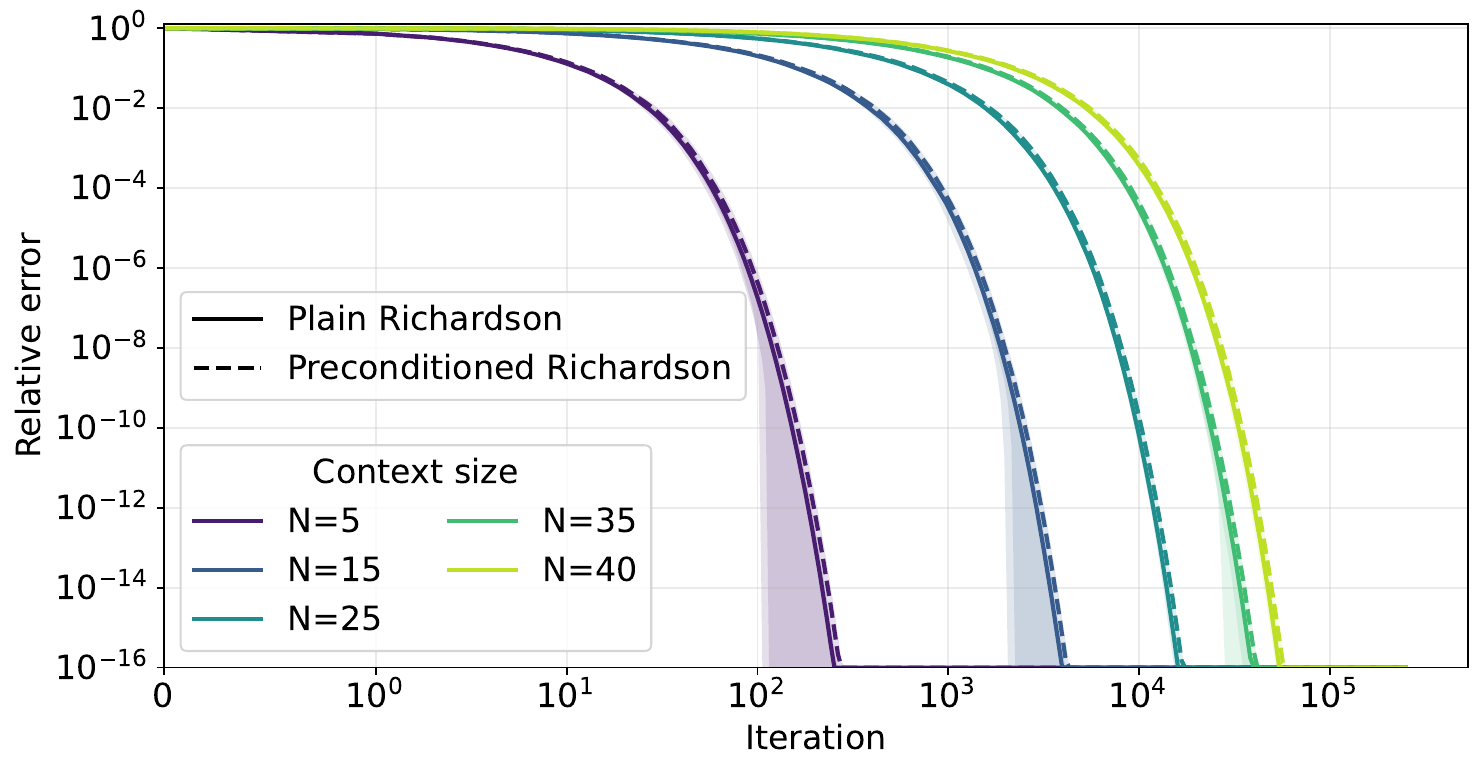}%
  }\hfill%
  \subfigure[Uniform]{%
    \label{fig:convergence-uniform}%
    \includegraphics[width=0.32\linewidth]
    {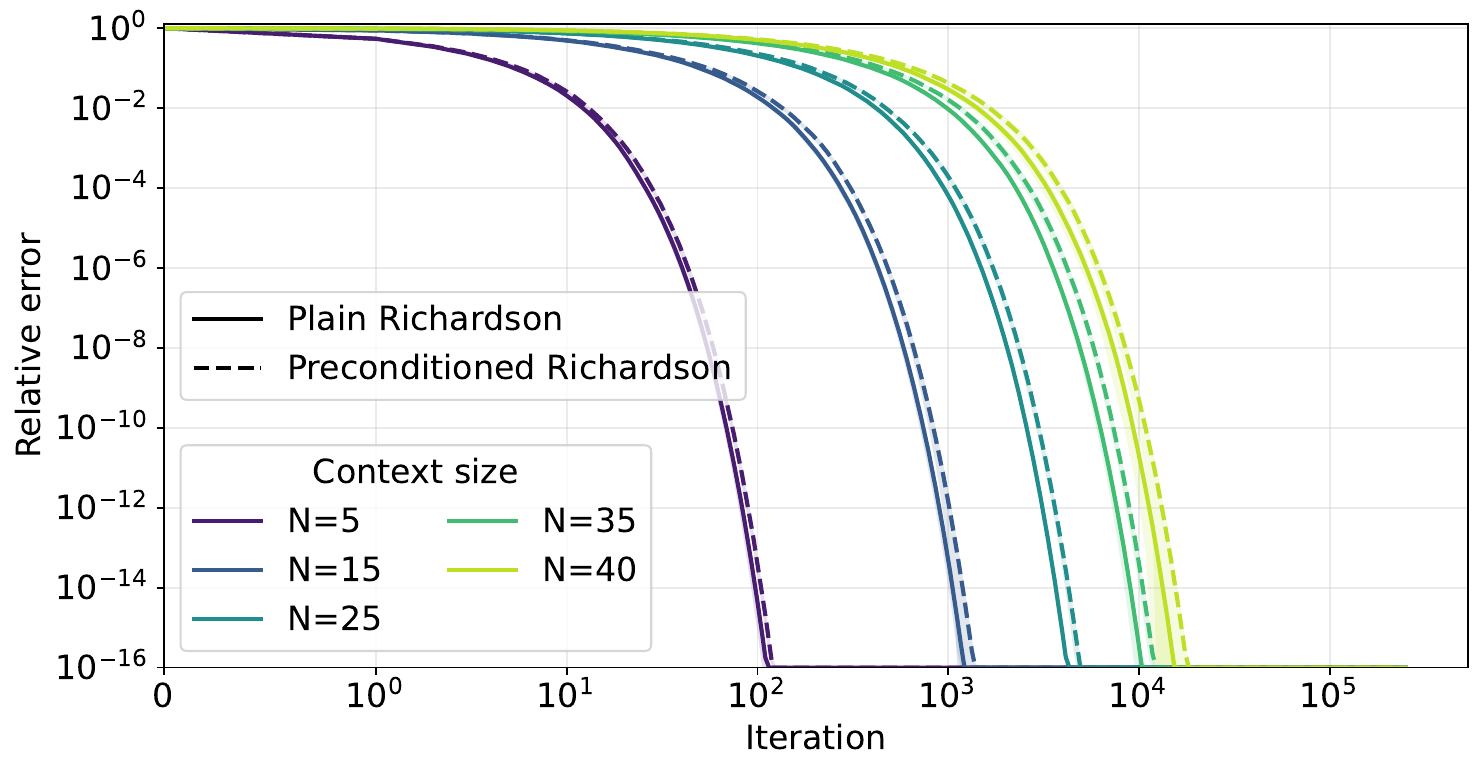}%
  }

  \caption{Convergence under each method's optimal step size $2/(\eig_{\max}(\mS)+\eig_{\min}(\mS))$. Plain and preconditioned Richardson require comparable numbers of iterations to reach the same relative coefficient error.}
  \label{fig:convergence-all-distributions}
\end{figure}

\FloatBarrier
\Needspace{6\baselineskip}
\subsection{Additional results on progressive alignment}
We supplement Section~\ref{subsec:convergence} with additional results on
prediction-error and induced-weight alignment. Figure~\ref{fig-app: all MSEs}
extends the prediction-error comparison in Figure~\ref{fig:convergence-a} to three input distributions, and Figure~\ref{fig-app:simw-depth-by-context} extends the induced-weight alignment
analysis in Figure~\ref{fig:convergence-b} across context lengths and input
distributions.

\begin{figure}[!htbp]
\includegraphics[width=\linewidth]{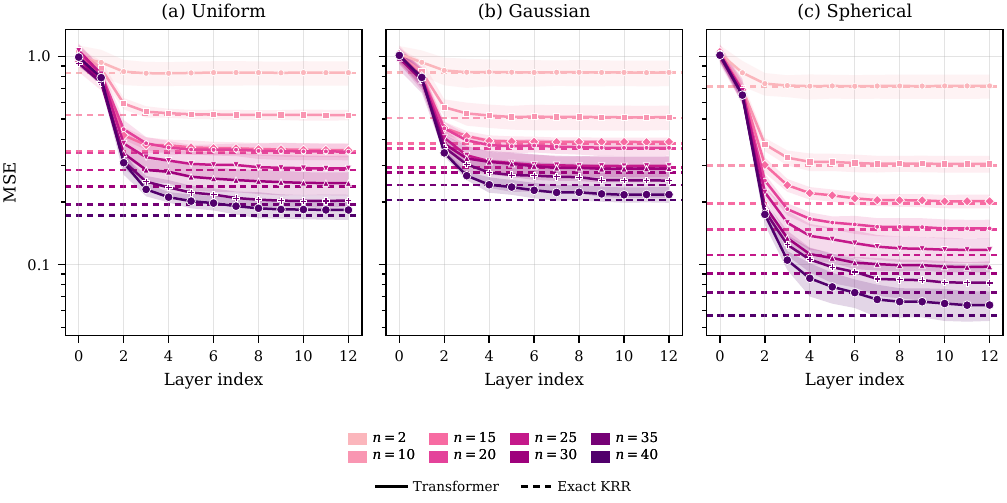}\caption{\textbf{Prediction-error alignment across input distributions.} Layer-wise prediction MSE across context lengths for Gaussian, spherical, and
Uniform inputs. Solid curves show the transformer MSE, and dashed horizontal
lines show the corresponding exact-KRR MSEs. Curves and shaded regions
indicate the mean and one standard deviation over five evaluation seeds,
respectively.}
\label{fig-app: all MSEs}
\end{figure}

\begin{figure}[!htbp]
    \centering

    \subfigure[Uniform
    \label{fig-app:simw-context-uniform}]{
        \includegraphics[width=0.98\linewidth]
        {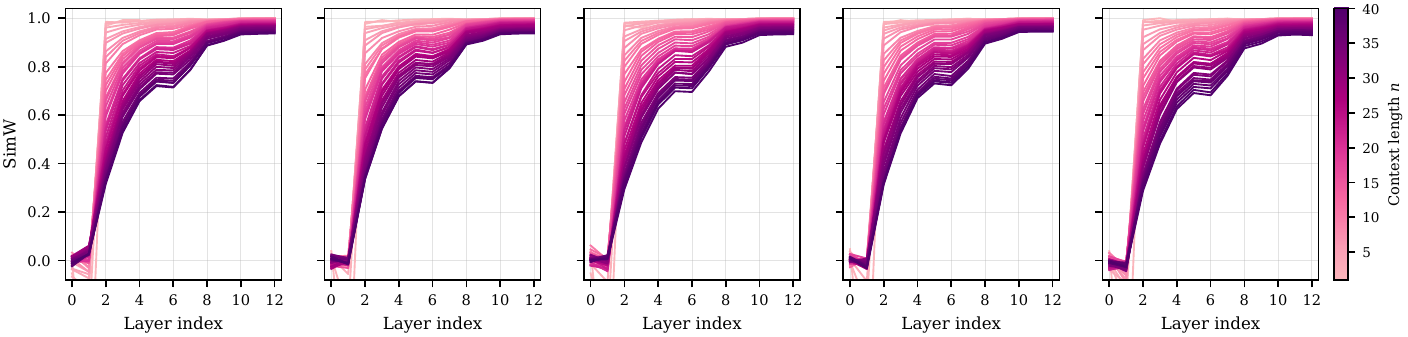}
    }

    \par\medskip

    \subfigure[Gaussian
    \label{fig-app:simw-context-gaussian}]{
        \includegraphics[width=0.98\linewidth]
        {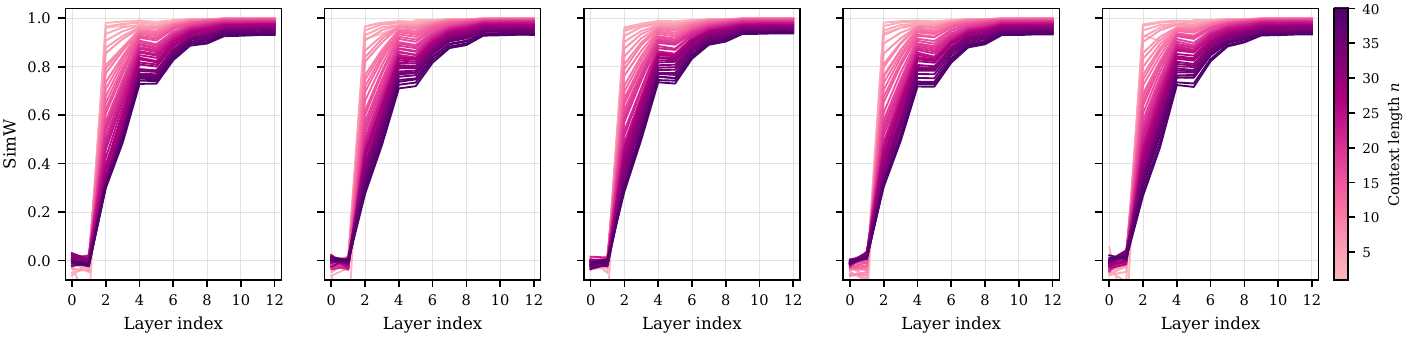}
    }

    \par\medskip

    \subfigure[Spherical
    \label{fig-app:simw-context-spherical}]{
        \includegraphics[width=0.98\linewidth]
        {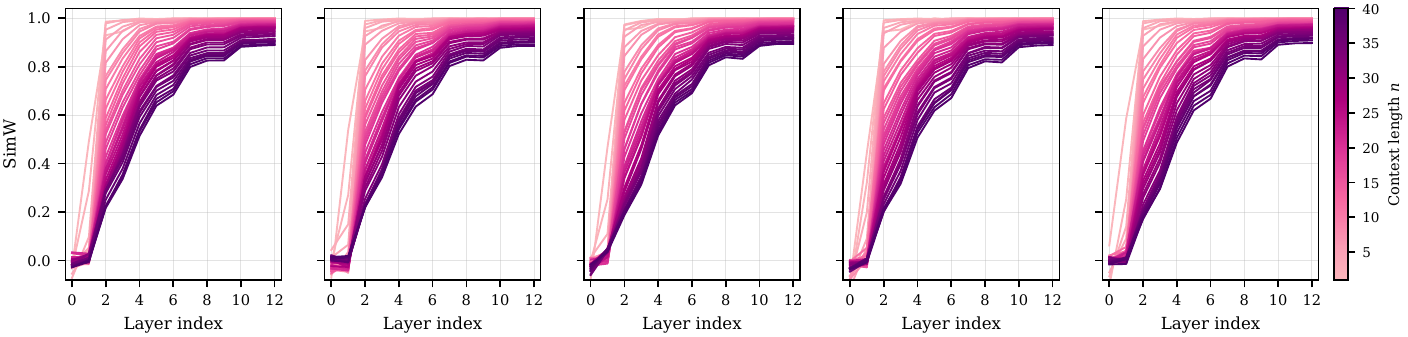}
    }

    \caption{\textbf{SimW across depth and context length.}
    Layer-wise SimW between the transformer's induced weights and the exact
    KRR weights under Uniform, Gaussian, and Spherical inputs.
    Within each subfigure, the five panels correspond to evaluation seeds
    1--5 from left to right, and each curve represents a different context
    length $n$. Across all three input distributions and context lengths,
    alignment with the exact KRR weights progressively increases with
    transformer depth.}
    \label{fig-app:simw-depth-by-context}
\end{figure}

\FloatBarrier
\Needspace{6\baselineskip}
\subsection{Ablation results}
\textbf{Robustness to linear probing.} Figures~\ref{fig-app: random backbones} and
\ref{fig-app: probing vs native readout} explore the role of linear probes, discussed in Section \ref{section: ablation studies}, through random-backbone controls and comparison with the transformer's
native readout, respectively.
\begin{figure}[!htbp]
\includegraphics[width=\linewidth]{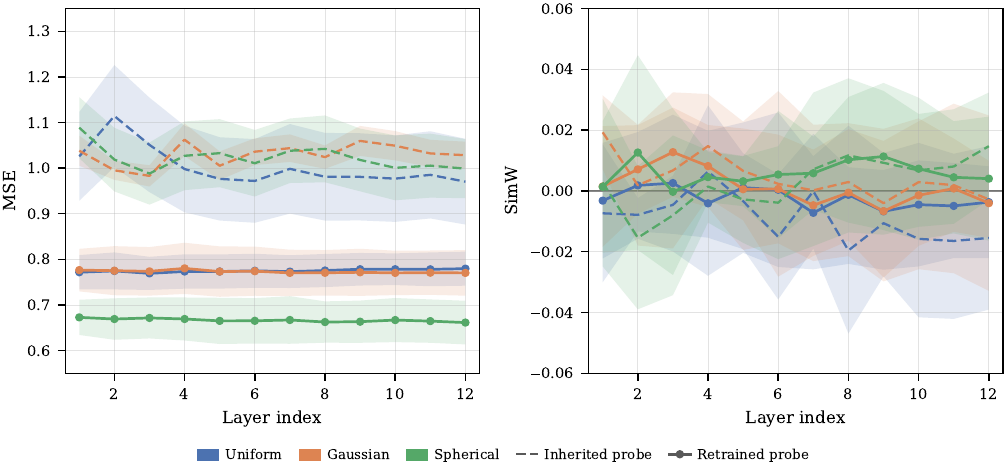}\caption{\textbf{Random-backbone ablations.} Layer-wise prediction MSE and SimW at $n=40$ for
Gaussian, spherical, and Uniform inputs. The random backbone is evaluated using
either probes transferred from the trained backbone or probes retrained on the
random backbone.}
\label{fig-app: random backbones}
\end{figure}

\begin{figure}[!htbp]
\includegraphics[width=\linewidth]{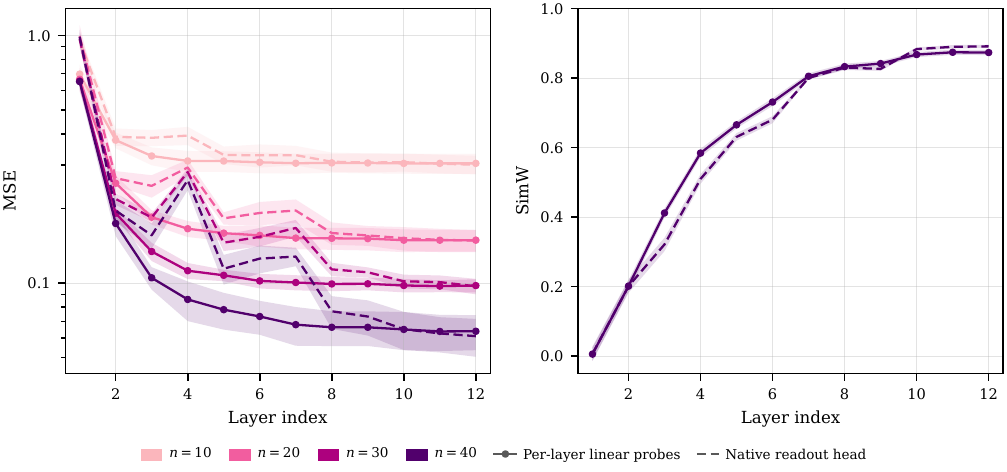}\caption{\textbf{Linear probing versus the native transformer readout on spherical inputs.} Layer-wise prediction MSE and SimW at $n=40$ obtained using separately trained linear probes or the
transformer's trained final LayerNorm and prediction head applied directly to
each intermediate hidden state.}
\label{fig-app: probing vs native readout}
\end{figure}

\FloatBarrier
\Needspace{6\baselineskip}
\textbf{Mismatched estimator kernel.}
Figure~\ref{fig1-app: kernel mismatched} presents the spherical-input
evaluation for the kernel-mismatch ablation discussed in
Section~\ref{section: ablation studies}.
\begin{figure}[!htbp]
\includegraphics[width=\linewidth]{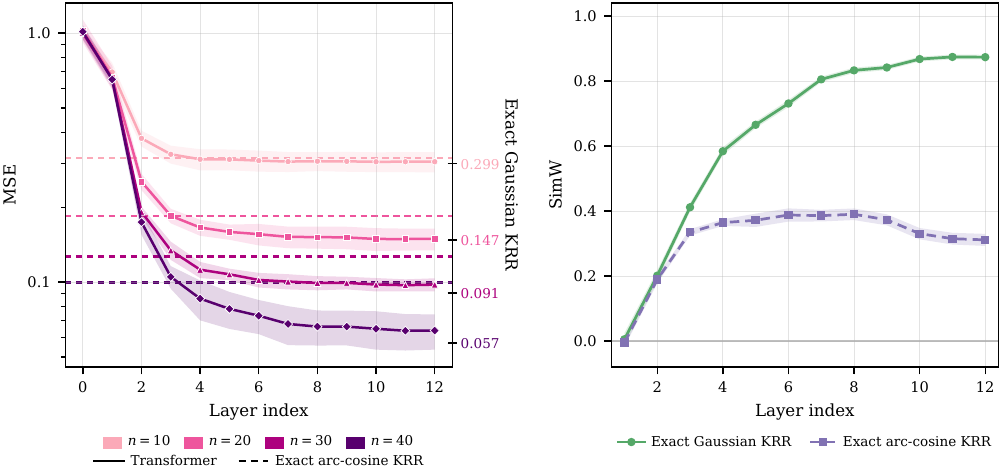}\caption{\textbf{Kernel-mismatch ablation on spherical inputs.}
\emph{Left:} Layer-wise transformer prediction MSE; dashed horizontal lines
show the exact arc-cosine-KRR MSEs, while the exact Gaussian-KRR MSEs are marked on the right axis.
\emph{Right:} Layer-wise SimW between the
transformer-induced weights and the exact KRR coefficient vectors obtained
using the Gaussian and arc-cosine estimator kernels.}
\label{fig1-app: kernel mismatched}
\end{figure}

\FloatBarrier
\Needspace{6\baselineskip}
\textbf{Noise-level ablation.}
Figure~\ref{fig1-app: noise level abalation} reports the noise-level sweep
across the three input distributions for the ablation discussed in
Section~\ref{section: ablation studies}.
\begin{figure}[!htbp]
    \centering
    \includegraphics[width=\linewidth]{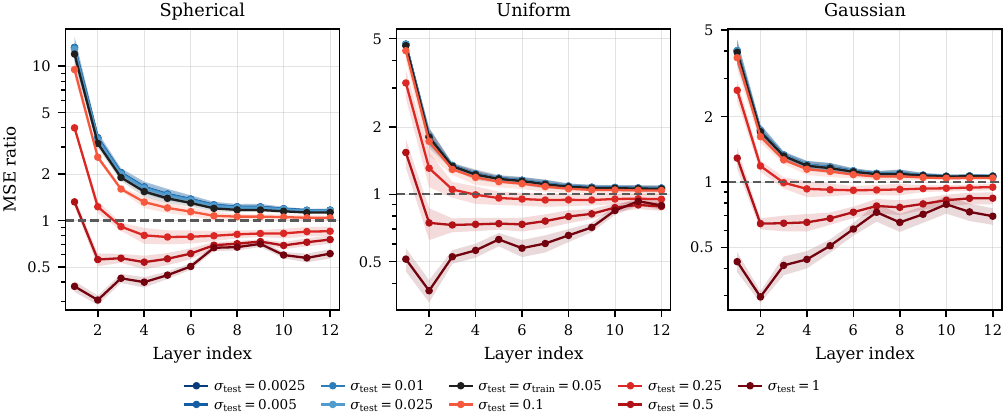}
    \caption{\textbf{Layer-wise $\mathrm{MSE}_{\mathrm{TF}}/\mathrm{MSE}_{\mathrm{KRR}}$ for different
    test noise levels and input distributions.} KRR uses $\lambda=\sigma_{\mathrm{train}}^2$. Curves and shaded regions show the
    mean and one standard deviation over five evaluation seeds, respectively.}
    \label{fig1-app: noise level abalation}
\end{figure}

\FloatBarrier
\Needspace{6\baselineskip}
\subsection{Solver comparisons for Gaussian and Uniform inputs}
Figures~\ref{fig2-app: gaussian nonidentification} and
\ref{fig2-app: uniform nonidentification} extend the solver comparison in
Figure~\ref{fig:nonidentification} to Gaussian and Uniform inputs using the
same evaluation protocol.
\begin{figure}[!htbp]
\centering
\subfiguretopcaptrue

\makeatletter
\renewcommand*{\@thesubfigure}{%
  \hspace*{4em}\thesubfigure\hskip\subfiglabelskip
}
\makeatother

\subfigure[All references]{%
  \includegraphics[width=0.32\textwidth]
    {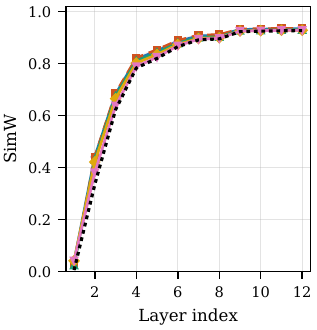}%
}\hfill
\subfigure[Late-layer zoom]{%
  \includegraphics[width=0.32\textwidth]
    {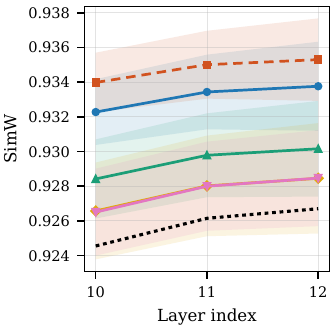}%
}\hfill
\subfigure[Best-matching step]{%
  \includegraphics[width=0.32\textwidth]
    {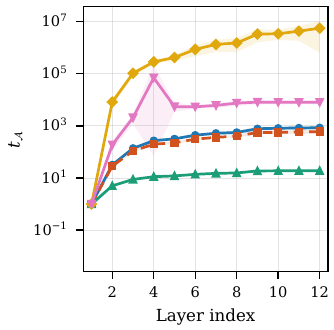}%
}

\par\vspace{1pt}

\includegraphics[width=0.8\textwidth]
  {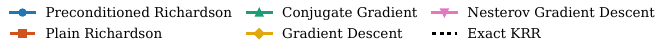}

\caption{\textbf{Best-match SimW and corresponding solver iterations for Gaussian inputs} at $n=40$. Panels follow the organization of Figure~\ref{fig:nonidentification}.}
\label{fig2-app: gaussian nonidentification}
\end{figure}

\begin{figure}[!htbp]
\centering
\subfiguretopcaptrue

\makeatletter
\renewcommand*{\@thesubfigure}{%
  \hspace*{4em}\thesubfigure\hskip\subfiglabelskip
}
\makeatother

\subfigure[All references]{%
  \includegraphics[width=0.32\textwidth]
    {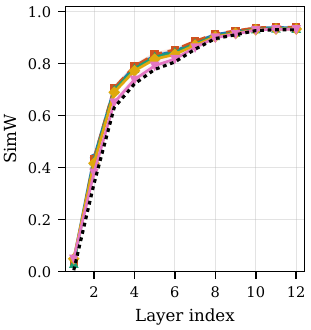}%
}\hfill
\subfigure[Late-layer zoom]{%
  \includegraphics[width=0.32\textwidth]
    {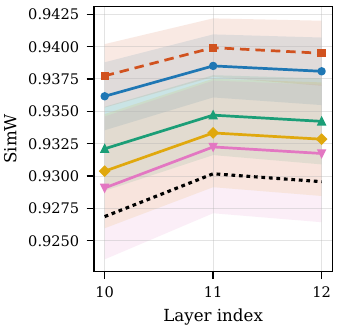}%
}\hfill
\subfigure[Best-matching step]{%
  \includegraphics[width=0.32\textwidth]
    {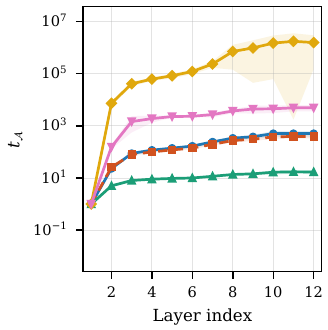}%
}

\par\vspace{1pt}

\includegraphics[width=0.8\textwidth]
  {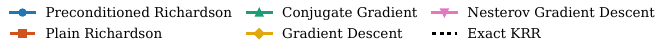}

\caption{\textbf{Best-match SimW and corresponding solver iterations for Uniform inputs} at $n=40$. Panels follow the organization of Figure~\ref{fig:nonidentification}.}
\label{fig2-app: uniform nonidentification}
\end{figure}

\FloatBarrier

\endgroup

\end{document}